\documentclass[Afour,sageh,times]{sagej}

\usepackage{moreverb,url}

\usepackage[colorlinks,bookmarksopen,bookmarksnumbered,citecolor=red,urlcolor=red]{hyperref}
\usepackage{times}
\usepackage{amsmath}
\usepackage{xspace}
\usepackage{float}
\usepackage{stfloats}
\usepackage{balance}
\usepackage{amssymb,latexsym,amsfonts}
\usepackage{xkeyval}
\usepackage{graphicx}
\usepackage{eso-pic}
\usepackage{epsfig}
\usepackage{url}
\usepackage{color}
\usepackage{tikz}
\usetikzlibrary{automata}
\usetikzlibrary{shapes,arrows}
\usepackage[latin1]{inputenc}
\usepackage{dcolumn}
\usepackage{subfigure}
\usepackage[algosection,linesnumbered,lined, ruled]{algorithm2e}
\usepackage{mathtools}
\usepackage[T1]{fontenc}
\usepackage{pgfplots}
\usepackage{pgfplotstable}
\usepackage{mathtools}
\usepackage{soul}
\usepackage{siunitx}
\usepackage{leftidx}
\usepackage{enumerate}
\usepackage[algosection,linesnumbered,lined, ruled]{algorithm2e}

\usetikzlibrary{patterns}

\DeclareMathOperator*{\argmax}{arg\,max}

\usepackage{multicol}
\usepackage{multirow}

\newtheorem{theorem}{Theorem}[section]
\newtheorem{lemma}[theorem]{Lemma}
\newtheorem{claim}[theorem]{Claim}

\newtheorem{corollary}[theorem]{Corollary}

\newtheorem{definition}[theorem]{Definition}

\newtheorem{remark}[theorem]{Remark}

\newcommand{\emrac}{\epsilon\textsc{-MRAC}}

\SetKwInOut{Function}{Functions}

\newcommand\BibTeX{{\rmfamily B\kern-.05em \textsc{i\kern-.025em b}\kern-.08em
T\kern-.1667em\lower.7ex\hbox{E}\kern-.125emX}}

\begin{document}

\runninghead{Kundu et~al.}

\title{Action-Consistent Decentralized Belief Space Planning with Inconsistent Beliefs and Limited Data Sharing: Framework and Simplification Algorithms with Formal Guarantees}

\author{Tanmoy Kundu\affilnum{1}, Moshe Rafaeli Shimron\affilnum{2}, Anton Gulyaev\affilnum{3}, Vadim Indelman\affilnum{4}}

\affiliation{\affilnum{1}Tanmoy Kundu is with the department of Computer Science and Engineering, IIIT-Delhi, India. This work was carried out when he was with the \revnew{Stephen B. Klein Faculty of Aerospace Engineering}, Technion - Israel Institute of Technology, Haifa 32000, Israel.  {\tt tanmoy.kundu@iiitd.ac.in}\\
\affilnum{2}Moshe Rafaeli Shimron is with the Technion Autonomous Systems Program (TASP), Technion - Israel Institute of Technology, Haifa 32000, Israel. {\tt mosh305@campus.technion.ac.il}\\
\affilnum{3}Anton Gulyaev is with the \revnew{Stephen B. Klein Faculty of Aerospace Engineering}, Technion - Israel Institute of Technology, Haifa 32000, Israel. {\tt ganton@technion.ac.il}\\
\affilnum{4}Vadim Indelman is with the \revnew{Stephen B. Klein Faculty of Aerospace Engineering and Faculty of Data and Decision Sciences}, Technion - Israel Institute of Technology, Haifa 32000, Israel. {\tt vadim.indelman@technion.ac.il}
}

\corrauth{\revnew{Vadim Indelman, Stephen B. Klein Faculty of Aerospace Engineering and Faculty of Data and Decision Sciences, Technion - Israel Institute of Technology, Haifa 32000, Israel.}}

\email{\revnew{vadim.indelman@technion.ac.il.}}

%






\newcommand{\indep}{\perp \!\!\! \perp}
\newcommand{\myputtext}[2]{
	\begin{tikzpicture}
	\node [fill=#2, rounded corners=2pt] {#1};
	\end{tikzpicture}
}
\newcommand{\VI}[1]{{\color{black} #1}}
\newcommand{\VIgray}[1]{{\color{black} #1}}
\newcommand{\tk}[1]{{\color{black} #1}}
\newcommand{\MR}[1]{{\color{black} #1}}
\newcommand{\AG}[1]{{\color{black} #1}}
\newcommand{\TODO}[1]{{\color{black} [TODO: #1]}}
\newcommand{\TEMP}[1]{{\color{black} #1}}
\newcommand{\rev}[1]{{#1}}
\newcommand{\revnew}[1]{{\color{black} #1}}

\newcommand{\Zeta}{\mathrm{Z}}
\newcommand{\bydef}{\ensuremath{\overset{def}{=}}}
\newcommand{\maxim}[2]{\ensuremath{\underset{#2}{#1}}}
\newcommand{\fnorm}[1]{\ensuremath{{\parallel}#1{\parallel_{\cal F}}}} 
\newcommand{\nrm}[2]{\ensuremath{{\parallel}{#2}{\parallel_{#1}}}}
\newcommand{\nrmsq}[2]{\ensuremath{{\parallel}{#2}{\parallel^2_{#1}}}}
\newcommand{\expt}[2]{\ensuremath{\underset{{#1}}{\mathbb E}{[{#2}]}}}
\newcommand{\alias}[1]{\ensuremath{\{{#1}\}_{\textbf{aliased}}}}

\newcommand{\prob}[1]{\ensuremath{\mathbb{P}({#1})}}
\newcommand{\probsup}[2]{\ensuremath{\mathbb{P}^{{#2}}({#1})}}
\newcommand{\blf}[1]{b({#1})}
\newcommand{\blfsup}[2]{b^{{#2}}({#1})}

\newcommand{\myvec}[1]{\underline{#1}}


\newcommand\la{\langle\xspace}  
\newcommand\ra{\rangle\xspace}

\newcommand*\Let[2]{\State #1 $\gets$ #2}

\newcommand{\priorB}{\ensuremath{\blf{X^-_{k+1}}}\xspace}
\newcommand{\condB}{\blf{X_{k+1} \mid z_{k+1}, \his}\xspace}
\newcommand{\condBi}[1]{\blf{X_{k+1} \mid #1, z_{k+1}, \his}\xspace}
\newcommand{\event}[1]{\ensuremath{A_{#1}}\xspace}

\newcommand{\poses}{\ensuremath{{\cal X}}\xspace}
\newcommand{\observations}{\ensuremath{{\cal Z}}\xspace}
\newcommand{\events}{\ensuremath{\{\event{\mathbb{N}}\}}\xspace}
\newcommand{\controls}{\ensuremath{{\cal U}}\xspace}

\newcommand{\his}{\ensuremath{{\cal H}}\xspace}

\newcommand{\inv}{\ensuremath{in}\xspace}

\newcommand{\Xinv}{\ensuremath{X^{\inv}}\xspace}
\newcommand{\Xnotinv}{\ensuremath{X^{\lnot \inv}}\xspace}
\newcommand{\Xinvsup}[1]{\ensuremath{X^{\inv{#1}}}\xspace}
\newcommand{\Xnotinvsup}[1]{\ensuremath{X^{\lnot \inv{#1}}}\xspace}
\newcommand{\entropy}[1]{\ensuremath{\mathcal{H}({#1})}\xspace}
\newcommand{\XF}{\ensuremath{X^{F}}\xspace}
\newcommand{\XnotF}{\ensuremath{X^{\lnot F}}\xspace}
\newcommand{\XinvF}{\ensuremath{X^{\inv,F}}\xspace}
\newcommand{\XnotinvF}{\ensuremath{X^{\lnot \inv, F}}\xspace}
\newcommand{\XnotinvnotF}{\ensuremath{X^{\lnot \inv, \lnot F}}\xspace} 
\newcommand{\XinvnotF}{\ensuremath{X^{\inv, \lnot F}}\xspace} 
\newcommand{\Finvv}{\ensuremath{\mathcal{F}^{\inv}}\xspace}
\newcommand{\Fnotinv}{\ensuremath{\mathcal{F}^{\lnot \inv}}\xspace}
\newcommand{\Finvnotinv}{\ensuremath{\mathcal{F}^{\inv, \lnot \inv}}\xspace}

\newcommand{\belmarginalpartial}[2]{\ensuremath{b^{{#1},M_p}_{#2}}\xspace}
\newcommand{\belmarginal}[2]{\ensuremath{b^{{#1},M}_{#2}}\xspace}
\newcommand{\bellocal}[2]{\ensuremath{b^{{#1},L}_{#2}}\xspace}

\newcommand{\belief}[2]{\ensuremath{b^{#1}_{#2}}\xspace}
\newcommand{\action}[1]{\ensuremath{a_{#1}}\xspace}
\newcommand{\obs}[2]{\ensuremath{z^{#1}_{#2}}\xspace}
\newcommand{\objec}[1]{\ensuremath{J^{#1}}\xspace}

\newcommand{\ac}{AC\xspace}
\newcommand{\mrac}{MR-AC\xspace}
\newcommand{\mrbsp}{MR-BSP\xspace}


\begin{abstract}
	
	In multi-robot systems, ensuring safe and reliable decision making under uncertain conditions demands robust multi-robot belief space planning (MR-BSP) algorithms. While planning with multiple robots, each robot maintains a belief over the state of the environment and reasons how the belief would evolve in the future for different possible actions. However, existing MR-BSP works have a common assumption that the beliefs of different robots are the same at planning time. Such an assumption is often unrealistic as it requires prohibitively extensive and frequent data sharing capabilities. In practice, robots may have limited communication capabilities, and consequently beliefs of the robots can be different. Crucially, when the robots have inconsistent beliefs,  existing approaches could result in lack of coordination between the robots and may lead to unsafe decisions. 
	In this paper, we address this gap by formulating the problem for a general multi-robot team and developing decentralized MR-BSP algorithms, with formal guarantees, for the two-robot setting.
	 Our algorithms leverage the notion of action preferences. The verification procedure \textsc{VerifyAC} certifies consistent joint action selection through a three-step test. When verification succeeds, no communication is required; otherwise, \textsc{EnforceAC} triggers data sharing and repeats the verification. Upon termination, \textsc{EnforceAC} guarantees that both robots select the same joint action sequence and that this sequence is optimal with respect to each robot's local belief conditioned on all data available to that robot at termination, including data received from the other robot, even though their beliefs may remain inconsistent. 
	We further develop the relaxed procedures \textsc{R-VerifyAC} and \textsc{R-EnforceAC}, which can reduce communication by relaxing the action-consistency criterion. Another extension, \textsc{R-VerifyAC-simp}, verifies only a subset of the possible observations and thereby significantly reduces computation time. We corroborate the theoretical guarantees through simulations in a discrete setting. Furthermore, we formulate our approaches for continuous, high-dimensional state and observation spaces and evaluate them in active multi-robot visual SLAM experiments with real robots.

\end{abstract}


\maketitle

\section[introduction]{Introduction}\label{sec:intro}
Decision making for multi-robot systems under various uncertainties and with partial observabilities is a significant problem in robotics, with its applications ranging vastly among search and rescue, surveillance, autonomous driving etc. to name a few.  
Planning in partially observable domains is challenging even with a single robot, and the challenge increases multi-fold with multiple robots in the system. 
Planning under uncertainty with multiple agents is often formulated with Decentralized Partially Observed Markov Decision Process (Dec-POMDP) or multi-robot Belief Space Planning (MR-BSP). In this work, we consider multiple robots working in a cooperative and collaborative way.

Multi-robot planning under uncertainties has been investigated in  recent years.
Yet, a prevailing assumption in existing approaches, being explicit or implicit, is that the beliefs of different robots at planning time are \emph{consistent}, i.e.~conditioned on the \emph{same} information.
This assumption is only valid when all the data (observations) captured by each robot are fully shared with all other robots and results in the beliefs of the robots be conditioned on the same data, i.e. consistent beliefs of the robots.
This rigorous assumption of consistent beliefs requires prohibitively extensive and frequent communication capabilities.
However, in numerous problems and scenarios, extensive data sharing between the robots cannot be made on a regular basis.
Moreover, it is often the case that only a partial or compressed version of the data can be communicated in practice.
As a result, each robot may have access to different data, which would correspond to different beliefs about the state of the environment, i.e.~to inconsistent beliefs of the robots. 

Crucially, when the beliefs of different robots are inconsistent, the state-of-the-art MR-BSP approaches could result in a lack of coordination between the robots, and in general, could yield dangerous, unsafe, and sub-optimal decisions.
For instance, consider the toy example shown in Figure \ref{fig:concept}, where  two robots aim to reach a common goal point without colliding with each other.
Due to lack of communication, the robots' beliefs become inconsistent; MR-BSP in such a setting may result in each robot calculating a different joint action, which can lead to a collision (Figure \ref{fig:motivational-example-c}).

\begin{figure*}[t]
	\subfigure[]{\includegraphics[height=2.7cm, width=0.32\textwidth]{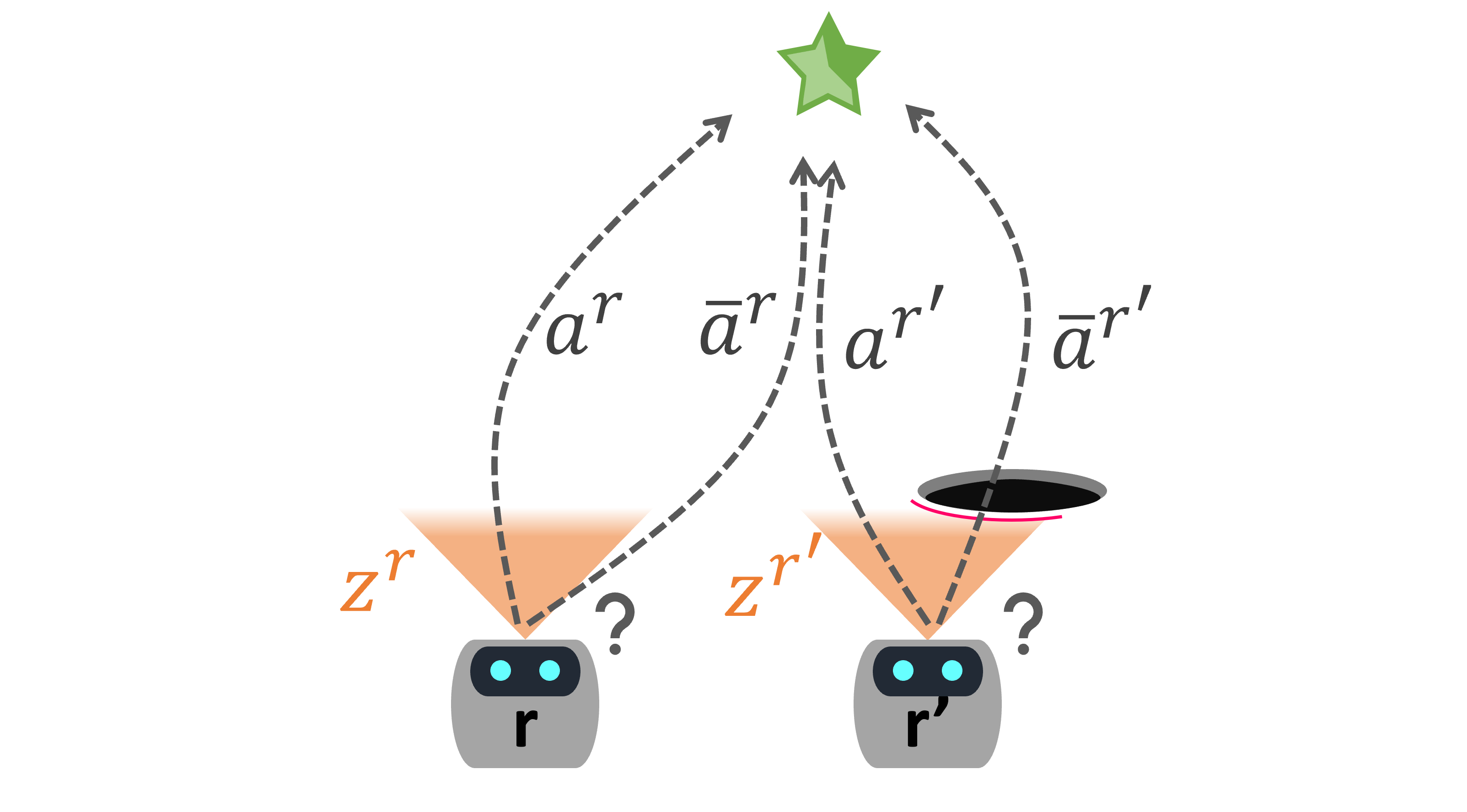}\label{fig:motivational-example-a}}
	\subfigure[]{\includegraphics[height=2.7cm, width=0.32\textwidth]{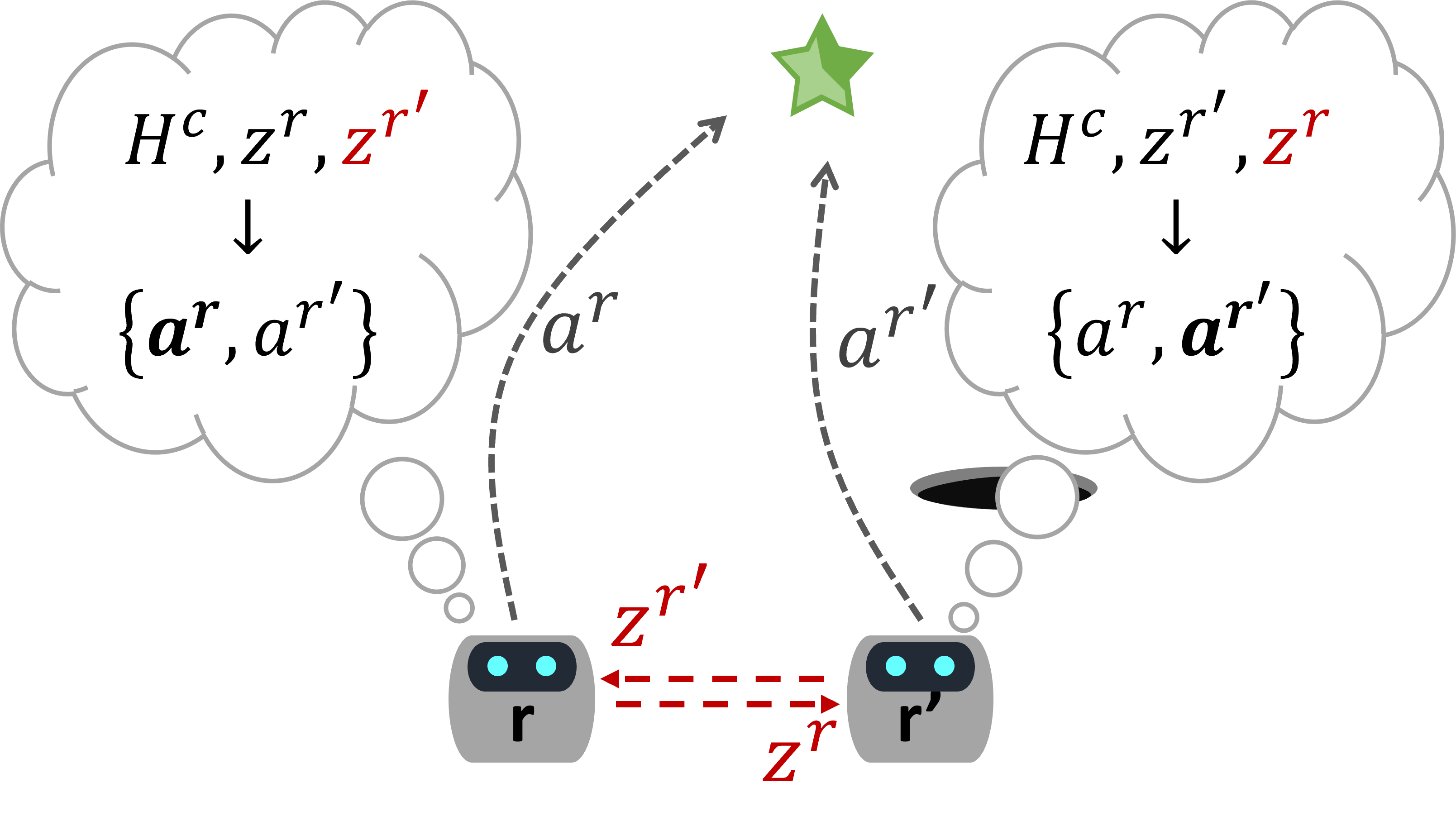}\label{fig:motivational-example-b}}
	\subfigure[]{\includegraphics[height=2.7cm, width=0.32\textwidth]{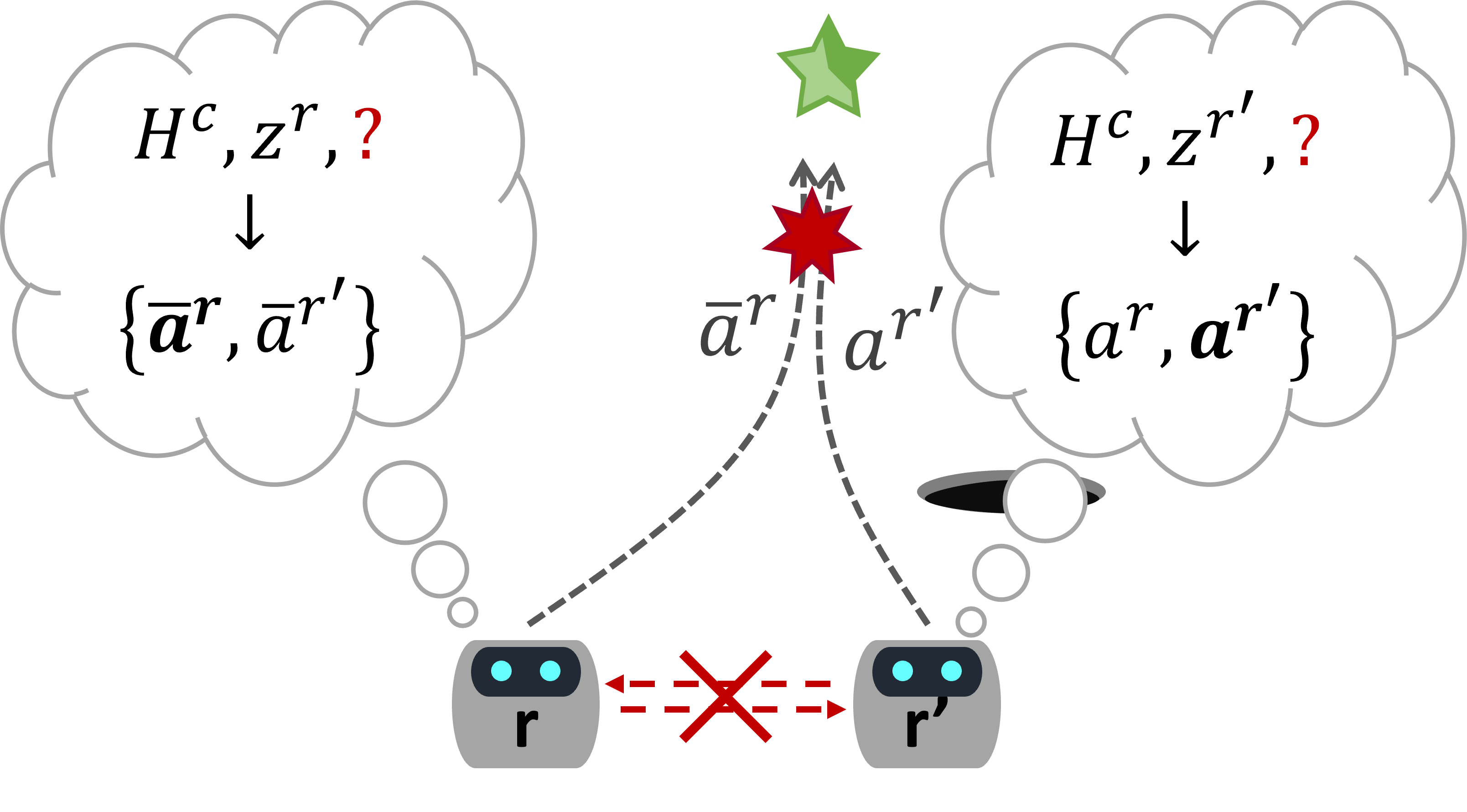}\label{fig:motivational-example-c}}
	\caption{\label{fig:concept}\scriptsize 
		\textbf{(a)} Two robots, $r$ and $r'$, acquire separate observations ($z^r, z^{r'}$) and begin a planning session. The robots aim to cooperatively reach the green star while satisfying a safety property of avoiding obstacles and collisions; each agent has two candidate actions $\bigl(\{a^r, \bar{a}^r\}\text{ for robot $r$, and }\{a^{r'}, \bar{a}^{r'}\}\text{ for robot $r'$ }\bigr)$.
		\textbf{(b)} The robots communicate their observations to each other (red color), after which their histories, and beliefs, become consistent. Decentralized MR-BSP in such a case yields the same best joint action for both robots. Here, $H^c$ represents the common history between the robots prior to acquiring the observations $z^r$ and $z^{r'}$.
		\textbf{(c)} The robots do not communicate their observations, and as a result, the robots' beliefs become inconsistent. This causes the robots to conclude inconsistent joint-actions, which leads to a collision. \vspace{-10pt}
	}
\end{figure*}

Despite this, multi-robot planning with \emph{inconsistent} beliefs has not been explicitly addressed thus far.
In this paper we address this crucial gap of MR-BSP with inconsistent beliefs of the robots.
To our knowledge, this work is the first to address this gap. We develop a novel decentralized algorithm named \textsc{VerifyAC} that implements an MR-BSP framework that explicitly accounts for inconsistent beliefs, and self-triggers communication only when the same joint action selection by different robots cannot be guaranteed.
Considering \textsc{VerifyAC} as the base algorithm, we develop two extended versions \textsc{R-VerifyAC} and \textsc{R-VerifyAC-simp} to further reduce the number of communications and reduce the computation time while ensuring action consistency.

At the core of our proposed approach is the notion of \emph{action consistency} \citep{Indelman16ral, Elimelech22ijrr,  Kitanov24ijrr}, which captures the observation that decision making involves identifying which action is preferable over other actions: 
If in two decision making problems a certain action is preferred over other candidate actions, then that particular action would be identified as the best action in both problems, regardless of the actual objective function values.
In such  case, the two problems are action-consistent.

Our key observation is that when the robots perform MR-BSP with inconsistent beliefs, each robot solves a \emph{different} planning problem; nevertheless, these problems can still yield the \emph{same} best joint action if they are action-consistent.
Leveraging this key observation, in this paper we develop an approach that detects if a consistent decision making among robots can be guaranteed albeit the inconsistent beliefs.
This involves reasoning about the missing information (observations) and the corresponding beliefs of the other robot(s).
In the case where an action-consistent decision-making cannot be guaranteed, our approach self-triggers communication of the information that will eventually lead to an action-consistent decision-making. 
The communications are self-triggered because each robot reasons about when it has to initiate a communication without being triggered by any other robot.


\revnew{While the addressed problem and the underlying notion of action consistency apply to a general team of robots, the algorithms and formal guarantees developed in this work focus on the two-robot setting; extending them to larger teams is left for future work.}

Specifically, Algorithm \textsc{VerifyAC} deterministically guarantees \mrac according to the following three steps. For each given robot, \textsc{VerifyAC} reasons for action preferences about 1) its local information, 2) what it perceives about the reasoning of the other robot, and 3) what it perceives about the reasoning of itself perceived by the other robot. However, it requires all the observations in steps 2 and 3 to be in favor of a particular joint action selected in step 1.  
This is quite challenging to satisfy consistency requirements for \emph{all} the observations in steps 2 and 3 and less likely to happen without frequent communications. 

We design an extended version of \textsc{VerifyAC} named \textsc{R-VerifyAC} which relaxes the satisfaction criteria for \mrac. Algorithm \textsc{R-VerifyAC} requires some, instead of all, the observations to be in favor of the step 1 action.
\mbox{\textsc{R-VerifyAC}} contributes to further reduction in the number of communications. In terms of performance guarantees, \textsc{R-VerifyAC} provides deterministic and probabilistic guarantees on \mrac based on some conditions.

Although \textsc{R-VerifyAc} reduces the number of communications with performance guarantees, it still requires to compute over all the observations in steps 2 and 3. We address this issue by proposing a simplified variant of \mbox{\textsc{R-VerifyAC}} named \textsc{R-VerifyAC-simp} that aims to reduce the computation time by computing over a smaller subset of the possible observations.


To summarize, the \textbf{main contributions} of this paper are as follows: 
\begin{enumerate}[(a)]
	\item \label{contribution:a} We introduce and formulate a new problem, i.e.~MR-BSP with inconsistent beliefs \revnew{across a general team of cooperative robots}.
	%
	\revnew{\item For the two-robot setting, we develop \textsc{VerifyAC}, a three-step verification procedure that uses action preferences to certify consistent joint action selection despite inconsistent beliefs. When verification succeeds, no communication is required; otherwise, \textsc{EnforceAC} triggers data sharing and repeats the verification. 	We prove that \textsc{EnforceAC} terminates after a finite number of communication rounds, upon which both robots select the same joint action sequence. This sequence is optimal with respect to each robot's local belief conditioned on all data available to that robot at termination, including data received from the other robot during communication, even if the resulting local beliefs remain inconsistent.}
	
	\item \label{contribution:c} We design two extended versions \textsc{R-VerifyAC} and \textsc{R-VerifyAC-simp} of the base \textsc{VerifyAC} algorithm. These extended algorithms further reduce the number of communications and improve the computation time compared to the base algorithm \textsc{VerifyAC}, with performance guarantees.
	\item \label{contribution:d} We formulate an estimator for applying our algorithms in high-dimensional and continuous state and observation spaces. 
	\item \label{contribution:e}  We evaluate our methods and corroborate the provided MR-AC guarantees considering two problems, both with limited data sharing capabilities. First, we consider in simulation a multi-robot search and rescue  problem, formulating it in  discrete spaces. Second, we provide experimental results with real robots considering an active multi-robot visual SLAM problem that is formulated in continuous state and observation spaces.  
\end{enumerate}
A  conference version of this paper has appeared in \citep{Kundu24iros}. The current paper extends upon that conference paper with contributions \eqref{contribution:c}-\eqref{contribution:e}. 

This paper is organized as follows. Section \ref{sec:related-work} provides an overview of related work. In Section  \ref{section:prelim} we provide preliminaries and formulate the addressed problem. In \mbox{Section \ref{sec:approach}} we develop our \textsc{VerifyAC} approach, while in Section \ref{sec:simplification} we design two extended versions,  \textsc{R-VerifyAC} and \mbox{\textsc{R-VerifyAC-simp}}, and develop corresponding performance guarantees. In Section \ref{sec:continuous-and-high-dim-cases} we develop estimators to support  continuous and high-dimensional state and observation spaces. In Section \ref{sec:results} we evaluate our approaches in simulation and  real-robot experiments. Finally, in Section \ref{sec:conclusion} we conclude the discussion.


\section[related-work]{Related Work}\label{sec:related-work}
Multi-robot decision making with partial observability (POMDP) and belief space planning have received considerable attention in recent times. In multi-robot settings, significant research efforts have been dedicated broadly towards centralized and decentralized control of robots, cooperative and non-cooperative type of robots.
While, historically, solving a POMDP for a single robot is computationally hard~\citep{Pineau06jair} due to curse of dimensionality and curse of history, solving a POMDP for multiple robots is even harder~\citep{Bernstein02math}. The work in~\cite{Bernstein02math} proves that optimal decentralized planning is computationally intractable even with two agents in the system.
Nevertheless, progress has been made considering a decentralized general-purpose paradigm, Dec-POMDP, and macro-actions 
\citep{Amato16ijrr, Amato09aamas2, Capitan13ijrr, Oliehoek12chapter, Omidshafiei15icra}.


The concept of action consistency has been previously introduced and investigated in the context of simplification of single-robot belief space planning problems \citep{Indelman16ral, Elimelech22ijrr, Kitanov24ijrr}. The fundamental principle states that if the original and simplified problems have the same action preferences, then they also share the same optimal action (or policy). Since it is often not possible to guarantee the same action preference between the original and simplified decision making problems, in further research planning performance guarantees, in terms of adaptive deterministic and probabilistic bounds, were developed considering different simplification techniques and under varying problem setups (see e.g.~\citep{Barenboim23ral2, LevYehudi24aaai, Shienman22isrr, Zhitnikov22ai, Zhitnikov24ijrr, Yotam24tro, Kong24isrr, Kong26wafr,  Barenboim26aij, Pariente26jair_accepted}). In our work, we focus on ensuring action consistency for multiple robots under uncertainties. 


Approaches that consider a non-cooperative multi-robot system, where each robot has its own reward function (representing a different task), typically tackle the problem within the framework of dynamic games by reasoning about the Nash equilibrium under various settings \citep{Mehr23tro, Schwarting21tro, So23icra}.
Multi-Robot BSP (MR-BSP) approaches have also been investigated considering a cooperative setting, i.e.~all robots have the same reward function (same task).
For instance, the works in \citep{Atanasov15icra, Indelman17arj, Regev17arj} consider MR-BSP with Gaussian high-dimensional distributions in the context of cooperative active SLAM and active inference. 

\rev{However, these works assume the robots have access to the same belief at planning time. In contrast, in this paper, we consider the setting where the robots have inconsistent beliefs at planning time due to limited data sharing capabilities. We additionally provide simplification algorithms that considerably reduce computational complexity while providing formal guarantees on achieving multi-robot action consistency, i.e.~the probability that robots will choose the same joint action despite having inconsistent beliefs. 
}



 In the past, there has been extensive research on multi-robot coordination with connectivity constraints. In a collaborative setting, approaches involving periodic connectivity at fixed intervals \citep{Hollinger12tro} as opposed to continual connectivity, robots with communication capabilities within a common radius \citep{Dutta19icra}, interaction constraints due to limited computational resources \citep{Heintzman21ral}, topology correction controller to accommodate a surrogate robot to replace a robot failing due to limited connectivity \citep{Yi21icra} have been proposed. \rev{The work in \citep{Khodayi-mehr19tro} studies the problem of distributed state estimation and informative path planning using multirobot systems, with limited communication capabilities between the robots. 
 They propose an intermittently connected robot network which divides the robots into subnetworks based on physical proximity of the robots, and have connectivity between the subnetworks. So, the robots in a subnetwork propagate all local measurements with each other within the subnetwork, thus having consistent beliefs, and then inter-subnetwork communication enables to propagate the measurements throughout the network with some delay. 
 In the contrary, our work enables a cooperative team of robots to achieve consistent decision making even though the beliefs of the member robots are inconsistent. Our approach decides to trigger  communication only when consistent decision making cannot be guaranteed with sufficient confidence. }

To the best of our knowledge, the closest work to our paper is \citep{Wu11ai}, where the authors explicitly consider the beliefs of different robots to be inconsistent.
However, that work requires communication whenever the beliefs of different robots are detected to be inconsistent, and hence, eventually, planning is done with consistent beliefs. In this paper, we deal with the setting wherein the robots are barred from doing frequent communication and have inconsistent beliefs. Despite inconsistent beliefs, our work ensures consistent decision making by the robots. Our approaches decide when to communicate with a goal to reduce the number of communications. Thus, we relax the criteria of consistent beliefs of the robots and handle the problem of multi-robot coordination with inconsistent beliefs of the robots.


\section{Preliminaries \& Problem Formulation}\label{section:prelim}
\subsection{Preliminaries}
Consider a team of $u$ robots $\Gamma= \{r_1,r_2,\dots, r_u\}$ performing some task(s).
We pick a robot $r$ arbitrarily from $\Gamma$, and denote by $-r$ the rest of the robots in the group. From the perspective of robot $r$, we define a \emph{decentralized} multi-robot POMDP as a $7$-tuple:  $\langle \mathcal{X}, \mathcal{Z}, \mathcal{A}, T, O, \rho^r, b_k^r \rangle$. Here, $\mathcal{X}$ is the application-dependent joint state space, and $\mathcal{Z}$ and $\mathcal{A}$ are the joint observation and joint action spaces.
\rev{The joint action space consists of the individual action spaces of the robots $\mathcal{A} \equiv \times_{r \in \Gamma} \mathcal{A}^{r}$, where each $\mathcal{A}^{r}$ is the individual action space of robot $r \in \Gamma$. Similarly, the joint observation space is $\mathcal{Z} \equiv \times_{r \in \Gamma} \mathcal{Z}^{r}$, where each $\mathcal{Z}^{r}$ is the individual observation space of robot $r \in \Gamma$.}
$T(x' \mid x, a)$ and $O(z \mid x)$ are the joint transition and observation models, where $x\in \mathcal{X}$, $a\in \mathcal{A}$ and $z \in \mathcal{Z}$ are, respectively, the joint state,  action and observation. Furthermore, we assume that the observations of different robots are independent conditioned on the state, i.e.~$O(z \mid x)=\prod_{r \in \Gamma} O^r(z^r \mid x)$  where each $z^r \in \mathcal{Z}^r$ is the local observation of robot $r$, and $O^r(.)$ is the corresponding observation model of robot $r$. $\rho^r$ is a general belief-dependent reward function $\rho^r: \mathcal{B}\times \mathcal{A} \mapsto \mathbb{R}$, where $\mathcal{B}$ is the belief space.  In this work we consider a \emph{cooperative} setting, i.e.~each robot has the same reward function $\rho$ that  describes a joint task allocated to the group (e.g.~information gathering).
Therefore, $\rho^r(b,a)=\rho^{r'}(b,a)$ for any $r, r' \in \Gamma$. 
\rev{The action space $\mathcal{A}^{r}$, observation space $\mathcal{Z}^{r}$ and observation model $O^r(.)$ of different robots can be different, but we assume that the spaces and models of all robots are available to every robot. Additionally, we assume that each robot can \emph{potentially} communicate with any other robot in the group, i.e.~a fully connected communication graph, and that when communication occurs it is noiseless and instantaneous.}




We denote by $b_k^r$  the belief of robot $r$ at time $k$ over the state $x_k \in \mathcal{X}$,
\begin{equation}
	b^r_k[x_k] \triangleq  \prob{x_k \mid H_k^{r}},
\end{equation}
where $H_k^{r}$ is the history available to robot $r$ at time $k$, which includes its own actions and observations, \rev{and might also include} those of other robots in the group. In the specific case where the actions and observations of all robots in the group are available to robot $r$, $H_k^{r}$ is \rev{ $\{a_{0:k-1}^r, z_{1:k}^r, a_{0:k-1}^{-r}, z_{1:k}^{-r}\}$ }.


In a \emph{collaborative} setting robot $r$ reasons over the joint actions of the robots, instead of its individual actions.
For ease of exposition we shall consider an open loop setting, although this is not a limitation of our proposed concept. 

\rev{Further, we consider a finite planning horizon of $L$ steps, and  define $k+\triangleq [k:k+L-1]$, where $k$ is the current time instant.} 
Let  $\mathcal{A}_{k+} =\mathcal{A}_{k:k+L-1}=\{a_{k:k+L-1}\}$ denote a set of $L$-step  joint action sequences  formed from the joint action space $\mathcal{A}_{k+}$,
\rev{where each joint action $a_{t}$ in the sequence is defined as \mbox{$a_{t} \triangleq (a_{t}^{r_1},\dots , a_{t}^{r_u}) = (a_{t}^r,a_{t}^{-r}) \in \mathcal{A}$}}.
Under these assumptions, the objective function of robot $r$ for 
a candidate joint action sequence $a_{k+} \triangleq a_{k:k+L-1} \in \mathcal{A}_{k+}$ is defined as
\rev{
\begin{equation}\label{eq-objfun}
	J(b_k^r, a_{k+}) = \! \!
	\underset{z_{k+1:k+L}}{\mathbb{E}} \! \!
	\left[ \sum_{l=0}^{L-1} \rho_l\bigl(b_{k+l}^r, a_{k+l}\bigr) \ \, +\ \, \rho_L\bigl(b_{k+L}^r\bigr) \right],
\end{equation}
}
where the expectation is over future observations $z_{k+1:k+L}$ of all  robots in the group with respect to the distribution $\prob{z_{k+1:k+L}\mid b^r_k, a_{k+}}$.  The optimal joint action sequence is:
\begin{equation}\label{eq-objargmax}
	a_{k+}^*  = \argmax\limits_{a_{k+} \in \mathcal{A}_{k+} }\ J(b_k^r, a_{k+}).
\end{equation}
\rev{When two actions have the exact same objective values, we assume the existence of a tie-breaker mechanism that decides which action is preferred over the other. We assume this mechanism is known to all robots. For instance, all possible joint actions $a_{k+} \in \mathcal{A}_{k+}$ can be priorly assigned unique IDs and in case two actions $a_{k+}$ and $a'_{k+}$ have the same objective value $J(b,a_{k+}) = J(b,a'_{k+})$ for some belief $b$, the action with the higher ID is preferred. }

In this paper we use the terms ``action sequence'' and ``action'' interchangeably.

\subsection{Problem Formulation}\label{sec:ProblemFormulation}

A typical assumption in existing multi-robot belief space planning  approaches is that of  \emph{consistent histories} across all the robots in $\Gamma$ at any planning time instant $k$, i.e.
\begin{equation}
	\forall r, r'\in \Gamma, \quad H_k^{r} \equiv H_k^{r'},
	\label{eq:consistenthis}
\end{equation}
which corresponds to the 
assumption that each robot has access to the \rev{actions and} observations of all other robots. Yet, in numerous real world problems and scenarios, such an assumption is clearly unrealistic (see Section \ref{sec:intro}).
We shall use the term \emph{inconsistent beliefs} and \rev{the notation  $b^r_k \neq b^{r'}_k$} whenever \eqref{eq:consistenthis} is not satisfied\footnote{Note that in nonparametric inference methods beliefs are generally inconsistent also given consistent histories (as given the same history the belief could be represented by different sets of particles). We leave the extension of our approach to such a setting to future work, and consider herein deterministic inference methods.}.

\rev{Generally, each robot has access only to its own local actions and observations, unless other robots communicated their actions and observations. 
Throughout this paper, our formulations consider the assumption that \emph{the performed actions of other robots} are available to robot $r$, but their observations may not be fully shared. 
Specifically, this assumption corresponds in the sequential decision-making framework to robots communicating their actions \emph{after} determining these actions as part of the planning phase, in parallel to their execution. Such an assumption is reasonable in many problems, especially when actions are relatively low-dimensional.
In contrast, the observations may be of higher  dimensionality (e.g.~landmarks' observations, images), and fully sharing them is often unrealistic.  
With this assumption, histories available to different agents may vary only due to not fully sharing observations with each other. 
While the proposed conceptual approach is valid also without this assumption, relaxing the latter is an ongoing research 
and is outside the scope of this paper. 
}

If any two robots $r$ and $r'$ have inconsistent beliefs, \rev{i.e. $b^r_k \neq b^{r'}_k$}, their theoretical objective function values \eqref{eq-objfun} for the same joint action $a_{k+}$ are not necessarily the same.
\rev{The inconsistent beliefs affect the difference between the objective function values in two ways.}
First, the expectation in \eqref{eq-objfun} is taken with respect to two different distributions, i.e.~$\prob{z_{k+1:k+L}\mid b^r_k, a_{k+}}$ and $\prob{z_{k+1:k+L}\mid b^{r'}_k, a_{k+}}$ \rev{are different}. Second, even when conditioned on the same realization of a future observation sequence $z_{k+1:k+l}$ for any time step $l\in[1,L]$, the theoretical posterior future beliefs $b^r_{k+l}$ and $b^{r'}_{k+l}$ are still inconsistent, \rev{since the former is conditioned on $b^r_k$ and the latter on $b^{r'}_k$. Hence, the}  corresponding rewards, \rev{$\rho_l(b_{k+l}^r, a_{k+l})$ and $\rho_l(b_{k+l}^{r'}, a_{k+l})$,} are different.

\revnew{
As a result, robots $r$ and $r'$ may identify different locally optimal joint actions:
\begin{align}
	\arg\max_{\bar a\in\mathcal A}J(b_k^r,\bar a)
	\neq
	\arg\max_{\bar a\in\mathcal A}J(b_k^{r'},\bar a).
\end{align}
Figure~\ref{fig:objective-surface} provides a conceptual illustration of this by viewing the objective $J(b,\bar a)$ as a belief- and action-dependent surface. Fixing $b$ selects a \emph{slice} (profile) $J(b,\cdot)$, whose maximum determines the corresponding optimal joint action,
\begin{align}
	\bar a^\star(b)\in
	\arg\max_{\bar a\in\mathcal A}J(b,\bar a).
\end{align}
%
%
Although each robot's selected joint action is optimal with respect to its own local belief, the selections of different robots may not agree. If each robot independently executes its component of its locally optimal joint action, the resulting joint action may combine components from two incompatible plans.  
The resulting joint action, upon execution, need not maximize either robot's local objective with respect to their corresponding beliefs, and may therefore degrade planning performance or safety.


}

\revnew{

\begin{figure}[t]
	\centering
	\includegraphics[ width=\columnwidth]{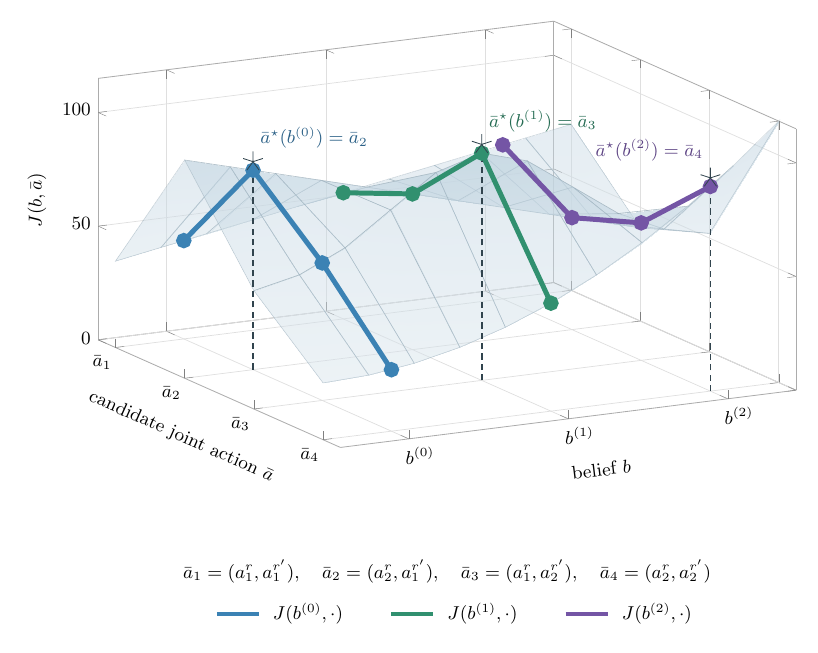}
	\caption{\label{fig:objective-surface}\scriptsize 
		\revnew{Illustration of the objective function $J(b_,\bar{a})$ as a belief- and action-dependent surface. Fixing a belief $b$ selects a slice, or profile, $J(b,\cdot)$. The figure shows several such slices for different  beliefs $b_k^{(i)}$. Each slice has a corresponding optimal joint action 
			$\bar{a}^\star(b_k^{(i)})\in\arg\max_{\bar{a}\in\mathcal A}J(b_k^{(i)},\bar{a})$. 
		}
		\vspace{-10pt}
	}
\end{figure}

}

Specifically, consider any two robots $r,r'\in \Gamma$. In a limited communication setting, at time $k$, 
consider that the last time instant when the beliefs  of these two robots were consistent is $p\in [1,k]$ time steps behind $k$. In other words, at time instant $k-p$, robots $r$ and $r'$ communicated with each other, resulting in $b^r_{k-p} = \prob{x_{k-p} \mid H^{r}_{k-p}}$ and $b^{r'}_{k-p} = \prob{x_{k-p} \mid H^{r'}_{k-p}}$ with $H^{r}_{k-p}=\{a_{0:k-p-1}, z_{1:k-p}^r, z_{1:k-p}^{-r}\}\equiv \{a_{0:k-p-1}, z_{1:k-p}^{r'}, z_{1:k-p}^{-r'}\}=H^{r'}_{k-p}$.  $H_{k-p}\triangleq H_{k-p}^r = H_{k-p}^{r'}$ for any $r,r'\in \Gamma$. 
In particular, if there are only two robots in the group, then $H_{k-p}= \{a_{0:k-p-1}, z_{1:k-p}^r, z_{1:k-p}^{r'}\}$.

During time period $[k-p+1,k]$, there was no communication and any robots $r, r' \in \Gamma$ do not have access to the non-local observations from these time instances. 
In other words, their beliefs 
\begin{equation}
	b^r_k = \prob{x_k \mid H^{r}_{k}},  \quad 
	b^{r'}_k = \prob{x_k \mid H^{r'}_{k}}, \label{eq:br_brtag}
\end{equation}
are inconsistent since robot $r$ does not have access to $z^{-r}_{k-p+1:k}$, and robot $r'$ does not have access to $z^{-r'}_{k-p+1:k}$. Currently we assume the actions performed by each robot by time instant $k$ are known. Thus,  $H_{k}^r \neq H_{k}^{r'}$ for any two robots $r, r'\in \Gamma$, where
\begin{align}
	H^{r}_{k} &= H^{r}_{k-p} \cup \{a_{k-p:k-1}, z^r_{k-p+1:k}\}, \label{eq:inconsistenthis}
	\\
	H^{r'}_{k} &= H^{r'}_{k-p}\cup \{a_{k-p:k-1}, z^{r'}_{k-p+1:k}\}.
	\label{eq:inconsistenthistag}
\end{align}

%

\revnew{
	\noindent The challenge addressed in this paper is therefore to ensure that all robots select the same joint action sequence $\bar{a}^{\star}_{k+}$, and that this sequence is optimal with respect to each robot's local beliefs, despite those beliefs being inconsistent.
	}


\section{Approach}\label{sec:approach}


Our key objective is to guarantee action consistency for multiple robots with inconsistent beliefs. 
\rev{For a given joint action, the $J$-values evaluated by different robots will generally be different because their beliefs are inconsistent. Consequently, the preference ordering of the actions can also be different. However, if we can guarantee that the best (most preferred) joint action selected by the robots are same, then we yield a \emph{consistent} best joint action regardless of the magnitude of the $J$-values. An illustration is shown in Figure~\ref{fig:action-consistency}.}

This is in striking contrast to existing approaches that implicitly ensure multi-robot action consistency (\mrac) by requiring the robots to have consistent beliefs, i.e.~assuming all the data between robots are communicated.

Specifically, we propose to utilize the concept of action consistency to address multi-robot decision making problems with inconsistent beliefs. We extend the definition of action consistency considering a multi-robot setting.
\begin{definition}[Multi-Robot Action Consistency (\mrac)]\label{def:mrac} Consider two robots $r,r'\in \Gamma$ where $r\neq r'$. At time $k$, the joint actions selected by $r$ and $r'$ are $a\in \mathcal{A}_{k+}$ and $a'\in \mathcal{A}_{k+}$ respectively. Robots $r$ and $r'$ are \emph{action consistent} at time $k$ if and only if $a = a'$. If, at time $k$, action consistency is satisfied for any two robots $r,r'\in \Gamma$, then the system of robots $\Gamma$ is action consistent at that time.
\end{definition}

\subsection{Action preferences with different beliefs}\label{subsec:ACwithDiffBeliefs}
We use the notion of comparing $J$-values where the \textit{order} of the values matters, and not the magnitude.
We define \emph{action preference} as a binary relation \rev{$\succcurlyeq$}
Consider two joint actions $a, a' \in \mathcal{A}_{k+}$.
The joint action $a$ is preferred over $a'$ w.r.t.~a given set of beliefs $\mathcal{B}_{Z}$  when action $a$ dominates $a'$ for all the beliefs in $\mathcal{B}_{Z}$: 
\begin{equation}\label{consistent-obs}
	\begin{split}
		\forall b \in \mathcal{B}_{Z} \ \ \ 
		a \succcurlyeq a' \iff J(b, a) \ge J(b, a').
	\end{split}
\end{equation}
While this is valid for any set of beliefs $\mathcal{B}_{Z}$, in this paper we shall consider \emph{a given set of observations $Z$}, and each belief  $b \in \mathcal{B}_{Z}$ results from Bayesian inference considering a particular observation  $z\in Z$.  
The joint action $a^* \in \mathcal{A}$ is most preferred if 
$a^* \succcurlyeq a$ holds for all other $a\in \mathcal{A}$.

\begin{definition}[Consistent observations]\label{def:consis-obs}
	Consider a set of observations $Z$, a set of joint actions $\mathcal{A}_{k+}$ and an objective function $J(.)$. If there exists an $a^*\in \mathcal{A}_{k+}$ satisfying $a^* \succcurlyeq a$ for all $a\in \mathcal{A}_{k+}$, we call $Z$ to be consistent in favor of $a^*$. We denote the consistency of $Z$ favoring action $a^*$ as $\mathtt{con}_{a^*}(Z) = \mathtt{true}$.
\end{definition}

\rev{Define $\mathtt{cobs}_{a^*}(Z)$ as the set $Z'\subseteq Z$ of all consistent observations in favor of action $a^*$:
%
\begin{equation}\label{eq:conobs}
	\mathtt{cobs}_{a^*}(Z) = Z'\subseteq Z \ \ \text{s.t.} \ \ \mathtt{con}_{a^*}(Z')\!=\!\mathtt{true}.
\end{equation}
}
%
When $\mathtt{con}_{a^*}(Z) = \mathtt{true}$, $\mathtt{cobs}_{a^*}(Z)$ contains the entire $Z$ because all the observations in $Z$ are consistent in favor of $a^*$. 
When $\mathtt{con}_{a^*}(Z) = \mathtt{false}$, 
$\mathtt{cobs}_{a^*}(Z)$ contains a proper subset of observations  $Z'\subset Z$ (instead of the entire $Z$) consistent in favor of $a^*$.

Next, we present our approach to check if \mrac exists despite the robots having inconsistent beliefs, and then describe a mechanism for self-triggered communications until \mrac is achieved.


\begin{figure}[t]
	\begin{center}
	\resizebox{0.7\linewidth}{!}{
	\begin{tikzpicture}
		\pgfplotsset{every tick label/.append style={font=\Large}}
		\begin{axis}[
			symbolic x coords={1,2,3,4,5,6,7,8,9,10,11,12,13,14,15},
			enlargelimits=0.05,
			legend style={at={(0.5,1.22)}, anchor=north,legend columns=-1},
			ylabel=$J(.)$,
			ymax=90, ymin=40,
			x tick label style={rotate=0, anchor=east, align=center, yshift=-0.5cm, font=\LARGE},
			y tick label style={color=white, font=\normalsize},
			xtick={1,2,3,4,5,6,7,8,9,10,11,12,13,14,15},
			xticklabels={,,$a_1$,,,$a_2$,,,$a_3$,,,$a_4$,,,$a_5$},
			xmin=1, xmax=15,
			ymin=0, ymax=180,
			height=5.5cm, width=10cm,
			]
			\addplot[red, mark=*, mark options={scale=1.5}]
			coordinates { (2, 145) (5,175) (8,120) (11, 25) (14,80)};
			\addlegendentry{\normalsize$J(b^{r},a))$ (for robot $r$)}
			
			\addplot[blue, mark=square*, mark options={scale=1.3}]
			coordinates { (2, 40) (5,155) (8,75) (11, 110) (14,130)};
			\addlegendentry{\normalsize$J(b^{r'},a))$ (for robot $r'$)}

			\draw (2,141) node [anchor=north west][inner sep=0.75pt]   [align=left] {rank-2};
			\draw (45,180) node [anchor=north west][inner sep=0.75pt]   [align=left] {rank-1};
			\draw (75,130) node [anchor=north west][inner sep=0.75pt]   [align=left] {rank-3};						
			\draw (105,30) node [anchor=north west][inner sep=0.75pt]   [align=left] {rank-5};
			\draw (120,74) node [anchor=north west][inner sep=0.75pt]   [align=left] {rank-4};
			
			\draw (2,33) node [anchor=north west][inner sep=0.75pt]   [align=left] {\color{blue}rank-5};
			\draw (45,160) node [anchor=north west][inner sep=0.75pt]   [align=left] {\color{blue}rank-1};
			\draw (75,85) node [anchor=north west][inner sep=0.75pt]   [align=left] {\color{blue}rank-4};						
			\draw (105,115) node [anchor=north west][inner sep=0.75pt]   [align=left] {\color{blue}rank-3};
			\draw (120,149) node [anchor=north west][inner sep=0.75pt]   [align=left] {\color{blue}rank-2};
		\end{axis}
	\end{tikzpicture}
	}
	\caption{\rev{Consider two robots $r$ and $r'$. The robots have different beliefs $b^r$ and $b^{r'}$, respectively. Plots of $J(b,a)$-values for different joint actions ($a_1$-$a_5$) is shown for robot $r$ (red) and $r'$ (blue). Though the robots compute different rankings of the actions $a_1$, $a_3$, $a_4$ and $a_5$, the rank of action $a_2$ is consistent (rank-1) for both the robots. The rank refers to the order of preference of the actions irrespective of the $J$-values of the actions. So, in this case action consistency is satisfied, and both the robots select the joint action $a_2$.}
	}
	\label{fig:action-consistency}
	\end{center}
\end{figure}

\subsection{MR-AC for robots with inconsistent beliefs}\label{mrac}
\label{approach-mrac}


\revnew{The problem formulation and definition of MR-AC presented above apply to a general team of robots. From this point onward, we develop the proposed algorithms and their guarantees for the two-robot setting, \(\Gamma=\{r,r'\}\). Extending these algorithms and guarantees to teams of more than two robots is outside the scope of this work.}

In a decentralized setting, we propose a mechanism to ensure \mrac for $\Gamma$ from the perspective of an arbitrarily chosen robot $r\in \Gamma$.
The aim of $r$ is to select, at planning time $k$, a joint action which is necessarily the same with the one chosen by robot $r'$, though $r$ and $r'$ have inconsistent beliefs $b^r_k$ and $b^{r'}_k$ from \eqref{eq:br_brtag}.

\begin{figure}[h]
\centering
\begin{tikzpicture}
	\begin{scope} [fill opacity = .5]
		\draw[fill=white, draw = white](-3.7,1.0) rectangle (3.7,0.0);
		\draw[fill=yellow, draw = black] (-1,1) ellipse (2cm and 0.4cm);
		\draw[fill=cyan, draw = black] (1,1) ellipse (2cm and 0.4cm);
		
		\node at (-1.7,1.0) {\textbf{$\Delta H^{r',r}$}};
		\node at (1.7,1.0) {\textbf{$\Delta H^{r,r'}$}};
		\node at (0,1.0) {\textbf{$\leftidx{^c}H^{r,r'}$}};
		\node at (-1.1,0.3) {\textbf{$H^{r}$}};
		\node at (1.1,0.3) {\textbf{$H^{r'}$}};
	\end{scope}
\end{tikzpicture}
\caption{\scriptsize Illustration  of $H^r, H^{r'}, \leftidx{^c}H^{r,r'}, \Delta  H^{r,r'}$, and $\Delta  H^{r',r}$. See text for details. 
}
\label{fig:history}
\end{figure}


Though robots $r$ and $r'$ have inconsistent histories at planning time $k$, they have a common part of history that we shall denote as $\leftidx{^c}H_k^{r,r'} \triangleq H^r_k \cap H^{r'}_k$.
Accordingly, we define by $\Delta  H^{r,r'}_k \triangleq H^{r'}_k \setminus \leftidx{^c} H_k^{r',r}$ the part in history of robot $r'$, i.e.~an observation sequence, that is \emph{unavailable} to robot $r$.
As discussed below, robot $r$ will have to reason about these missing observations of robot $r'$.
Similarly we define  $\Delta  H^{r',r}_k$.
Therefore, $H^{r}_k = \{\leftidx{^c} H_k^{r,r'} , \Delta  H^{r',r}_k\}$ and $H^{r'}_k = \{\leftidx{^c} H_k^{r,r'}, \Delta  H^{r,r'}_k\}$ as illustrated in Figure~\ref{fig:history}.

The beliefs \eqref{eq:br_brtag} can then be expressed as:
\begin{equation}\label{eq:br2}
	\begin{split}
		b^r_k &= \prob{x_k \mid \leftidx{^c} H_k^{r,r'}, \Delta  H^{r',r}_k}\\
		b^{r'}_k &= \prob{x_k \mid \leftidx{^c} H_k^{r',r}, \Delta  H^{r,r'}_k}.
	\end{split}
\end{equation}
Recall we assumed that the robots have consistent histories until time $k-p$.

	In the general case, some of the observations from time steps $[k-p+1, k]$ could have been communicated between the agents up to planning time $k$.
	In this case, we denote the missing observations sets as
	\begin{align}
		\Delta  H^{r',r}_k &= \{ z^{r}_{i_1}, z^{r'}_{i_2}, \ldots, z^{r}_{i_C} \}\label{eq:dHrtagr}
		\\
		\Delta  H^{r,r'}_k &= \{ z^{r'}_{i'_1}, z^{r'}_{i'_2}, \ldots, z^{r'}_{i'_{C'}} \}, \label{eq:dHrrtag}
	\end{align}
	with 
	\rev{$k-p+1 \leq i_1 < i_2 < \ldots < i_C \leq k$ and $k-p+1 \leq i'_1 < i'_2 < \ldots < i'_{C'} \leq k$.}
%
	A specific case can be when no communication was triggered in the time steps $[k-p+1, k]$, as described in \eqref{eq:inconsistenthis} and \eqref{eq:inconsistenthistag},
	\mbox{$\leftidx{^c} H_k^{r,r'}= H_{k-p} \cup \{a_{k-p:k-1}\}$},
	$\Delta  H^{r',r}_k= \{z^{r}_{k-p+1:k}\}$ and
	$\Delta  H^{r,r'}_k = \{z^{r'}_{k-p+1:k}\}$.

We propose the following steps to identify \mrac by robot $r$ despite inconsistent beliefs $b^r_k$ and $b^{r'}_k$ \eqref{eq:br2} of the robots.
Conceptually, robot $r$ needs to analyze the joint action preferences from different perspectives: i) its own perspective, ii) perspective of the other robot $r'$, and iii) its own perspective reasoned by the other robot $r'$.
If $r$ finds the same best joint action from all the above perspectives, then it can be assured that the other robot has also calculated the same best joint action;  hence, in such case,  both robots are action-consistent, i.e.~choose the same joint action.


Since robot $r$ does not have access to $\Delta  H^{r,r'}_k$, and it is aware that robot $r'$ does not have access to $\Delta  H^{r',r}_k$, these steps involve reasoning about all possible values of these missing observations.  
We denote the corresponding joint observation spaces, that represent these possible values, by $\Delta \mathcal{Z}^{r,r'}_k$ and $\Delta \mathcal{Z}^{r',r}_k$. Thus, $\Delta  H^{r,r'}_k \in \Delta \mathcal{Z}^{r,r'}_k$ and $\Delta  H^{r',r}_k \in \Delta \mathcal{Z}^{r',r}_k$. 
Also, define $\Delta \mathcal{Z}_k$ as
\begin{align}
	\Delta \mathcal{Z}_k &=\Delta \mathcal{Z}^{r,r'}_k \cup \Delta \mathcal{Z}^{r',r}_k.
	\label{eq:dZ}
\end{align}
%
	According to the general definitions of the missing observations sets,
	\eqref{eq:dHrtagr} and \eqref{eq:dHrrtag},
	the missing observations spaces are defined as
	\begin{align}
		\Delta \mathcal{Z}^{r,r'}_k &= \mathcal{Z}^{r'}_{i'_1} \times \mathcal{Z}^{r'}_{i'_2} \times \ldots \times \mathcal{Z}^{r'}_{i'_{C'}} \label{eq:dZrrtag}
		\\
		\Delta \mathcal{Z}^{r',r}_k &=\mathcal{Z}^{r}_{i_1}\times \mathcal{Z}^{r}_{i_2} \times \ldots \times \mathcal{Z}^{r}_{i_C}.
		\label{eq:dZrtagr}
	\end{align}
%
%
In Sections \ref{sec:approach} and \ref{sec:simplification} we assume these observation spaces to be discrete. In Section \ref{sec:continuous-and-high-dim-cases} we extend the approach to continuous and high-dimensional spaces.

\subsection{Algorithm \textsc{VerifyAC}}\label{subsec:verifyac}
We present an algorithm named \textsc{VerifyAC} that verifies \mrac from the perspective of robot $r$.
The algorithm captures the above concept and we now present  it in detail. 
Figure~\ref{ac-concept} illustrates \emph{conceptually} the mentioned steps in a toy example that has only two possible joint actions and two possible observations for each robot.
The steps of \textsc{VerifyAC} are described below.


\begin{figure}[t]
	\begin{center}
    \subfigure[Step 1 of robot $r$]{
    \resizebox{0.47\linewidth}{!}{
    \begin{tikzpicture}
        \pgfplotsset{every tick label/.append style={font=\Large}}
        \begin{axis}[
            symbolic x coords={2,3,4,5,a,b,6,7,8,c,9,10,11,12},
            enlargelimits=0.05,
            legend style={at={(0.5,1.20)}, anchor=north,legend columns=-1},
            ylabel=$J(.)$,
            ymax=90, ymin=40,
            x tick label style={rotate=0, font=\LARGE},
            y tick label style={color=white, font=\normalsize},
            xtick={4,10},
            xticklabels={$\bar{a}$,$\bar{a}'$},
            xmin=2, xmax=12,
            ymin=35, ymax=90,
            height=6cm, width=8cm,
            ]
            \addplot[black, mark=triangle*, mark options={scale=3}, dotted]
            coordinates { (3, 82.5) (9, 57.0)};
            \addlegendentry{\large $J(b_k^r,a)$}

        \end{axis}
        \label{ac-1-obs}
    \end{tikzpicture}
    }} 
    \subfigure[Steps 1-2 of robot $r$]{
    \resizebox{0.47\linewidth}{!}{
    \begin{tikzpicture}
        \pgfplotsset{every tick label/.append style={font=\Large}}
        \begin{axis}[
            symbolic x coords={2,3,4,5,a,b,6,7,8,c,9,10,11,12},
            enlargelimits=0.05,
            legend style={at={(0.5,1.22)}, anchor=north,legend columns=-1},
            ylabel=$J(.)$,
            ymax=90, ymin=40,
            x tick label style={rotate=0, font=\LARGE},
            y tick label style={color=white, font=\normalsize},
            xtick={4,10},
            xticklabels={$\bar{a}$,$\bar{a}'$},
            xmin=2, xmax=12,
            ymin=35, ymax=90,
            height=6cm, width=8cm,
            ]
            \addplot[black, mark=triangle*, mark options={scale=3}, dotted]
            coordinates { (3, 82.5) (9, 57.0)};
            \addlegendentry{\large $J(b_k^r,a)$}

            \addplot[cyan, mark=square*, mark options={scale=2}, dotted]
            coordinates { (5, 79.5) (11, 56.0)};
            \addlegendentry{\large$J(b_k^{r'|r},a)$}
        
            \addplot [cyan, mark=square*, mark options={scale=2}, dotted]
            coordinates { (5, 60.5) (11, 44.0)};
            
        \end{axis}
        \label{ac-2-obs}
    \end{tikzpicture}
    }}
    \subfigure[Steps 1-2 of robot $r'$]{
    \resizebox{0.46\linewidth}{!}{
    \begin{tikzpicture}
        \pgfplotsset{every tick label/.append style={font=\normalsize}}
        \begin{axis}[
            symbolic x coords={2,3,4,5,a,b,6,7,8,c,9,10,11,12},
            enlargelimits=0.05,
            legend style={at={(0.5,1.20)}, anchor=north,legend columns=-1},
            ylabel=$J(.)$,
            ymax=90, ymin=40,
            x tick label style={rotate=0, font=\LARGE},
            y tick label style={color=white, font=\normalsize},
            xtick={4,10},
            xticklabels={$\bar{a}$,$\bar{a}'$},
            xmin=2, xmax=12,
            ymin=35, ymax=90,
            height=6cm, width=8cm,
            ]
            \addplot[black, mark=square*, mark options={scale=2}, dotted]
            coordinates { (5, 79.5) (11, 56.0)};
            \addlegendentry{\large$J(b_k^{r'},a)$}

            \addplot [cyan, mark=triangle*, mark options={scale=3}, dotted]
            coordinates { (3, 62.5) (9, 42.0)};
            \addlegendentry{\large $J(b_k^{r|r'},a)$}

            \addplot[cyan, mark=triangle*, mark options={scale=3}, dotted]
			coordinates { (3, 77) (9, 65.0)};

        \end{axis}
    \end{tikzpicture}
    \label{ac-2-obs-other}
    }}
    \subfigure[Steps 1-3 of robot $r$]{
    \resizebox{0.48\linewidth}{!}{
    \begin{tikzpicture}
        \pgfplotsset{every tick label/.append style={font=\normalsize}}
        \begin{axis}[
            symbolic x coords={2,3,4,5,a,b,6,7,8,c,9,10,11,12},
            enlargelimits=0.05,
            legend style={at={(0.5,1.20)}, anchor=north,legend columns=-1},
            ylabel=$J(.)$,
            ymax=90, ymin=40,
            x tick label style={rotate=0, font=\LARGE},
            y tick label style={color=white, font=\normalsize},
            xtick={4,10},
            xticklabels={$\bar{a}$,$\bar{a}'$},
            xmin=2, xmax=12,
            ymin=35, ymax=90,
            height=6cm, width=8cm,
            ]
            \addplot[black, mark=triangle*, mark options={scale=3}, dotted]
			coordinates { (3, 82.5) (9, 57.0)};
			\addlegendentry{\large$J(b_k^r,a)$}

			\addplot[cyan, mark=square*, mark options={scale=2}, dotted]
			coordinates { (5, 79.5) (11, 56.0)};
			\addlegendentry{\large$J(b_k^{r'|r},a)$}

            \addplot [cyan, mark=triangle*, mark options={scale=3}, dotted]
			coordinates { (3, 62.5) (9, 42.0)};
			\addlegendentry{\large $J(b_k^{r|r'|r},a)$}

			\addplot [cyan, mark=square*, mark options={scale=2}, dotted]
			coordinates { (5, 60.5) (11, 44.0)};  
			   
        \end{axis}
    \end{tikzpicture}
    \label{ac-3-obs}
    }}
    \caption{Illustration of \textsc{VerifyAC} from the perspective of robot $r$. Robots $r$ and $r'$ have inconsistent beliefs $b^r_k$ and $b^{r'}_k$ at time $k$.
    Candidate joint actions are $\bar{a}$ and $\bar{a}'$. 
    Triangles and squares denote objective function ($J(.)$) evaluations for $r$ and $r'$ respectively.
    \textbf{(a)} Step 1 of $r$: Robot $r$ computes its belief for its actual observation. Chooses $\bar{a}$ as the best action. 
    \textbf{(b)} Step 1-2 of $r$: In Step 2, $r$ computes $J(.)$ for each possible observation of $r'$. All the observations are consistent in favor of $\bar{a}$.
    \textbf{(c)} Step 1-2 of $r'$: Similarly, robot $r'$ computes Step 1 for its actual observation, and Step 2 for all possible observations of $r$.
    \textbf{(d)} Step 3 of $r$: Combines (a)-(c) and verifies that the observations at each step are consistent in favor of action $\bar{a}$.
    Hence, $r$ can be assured that $r'$ also has chosen $\bar{a}$. Thus $r$ chooses action $\bar{a}$ at time $k$.
    }
    \label{ac-concept}
    \end{center}
\end{figure}

\subsubsection*{\textbf{Step 1: }\textit{Robot $r$ calculates the best joint action given its own belief $b^r_k$ via \eqref{eq-objargmax}}}\label{step-1}  
This involves evaluation of the objective function $J(b^r_k, \bar{a})$ for different candidate joint actions in $\mathcal{A}_{k+}$.
In other words, this involves evaluation of the objective function considering the belief $b^r_k$ from \eqref{eq:br2} which is conditioned on the consistent history $\leftidx{^c} H_k^{r,r'}$ and on the actual local observation(s) $\Delta  H^{r',r}_k$ of robot $r$, \rev{that are unavailable to robot $r'$}.
Finally, robot $r$ selects the best action $\bar{a}\in \mathcal{A}_{k+}$ such that $J(b^r_k, \bar{a}) > J(b^r_k, \bar{a}')\ \ \forall \bar{a}'\in \mathcal{A}_{k+} \text{ and } \bar{a}\neq \bar{a}'$.
The concept is illustrated in Figure \ref{ac-1-obs} by black triangles. 

However, since robots $r$ and $r'$ have inconsistent beliefs, it is not guaranteed that the joint action chosen by robot $r$ will be the same as chosen by $r'$. So, we move to Step 2.

\noindent
\subsubsection*{\textbf{Step 2: }Robot $r$ mimics the reasoning done by robot $r'$}\label{step-2}
The belief $b^{r'}_k$  \eqref{eq:br2} of the other robot $r'$ is conditioned on $\Delta  H^{r,r'}_k$ 
which is unavailable to robot $r$. Moreover, $b^{r'}_k$ is not conditioned on $\Delta  H^{r',r}_k$, 
which is unavailable to robot $r'$. Explicitly, the two beliefs  are given by \eqref{eq:br2}.

As $\Delta  H^{r,r'}_k$
is unavailable to robot $r$, it has now to reason over all the possible observation  realizations in $\Delta \mathcal{Z}^{r,r'}_k$  of robot $r'$. For instance, 
in the case of \eqref{eq:inconsistenthis}-\eqref{eq:inconsistenthistag}, i.e.~prior to any communication since the time instant $k-p$, 
this corresponds to all the possible realizations of  observation sequences of robot $r'$ between time instances $k-p+1$ and $k$ (one of which is the actually captured sequence of observations, $\Delta  H^{r,r'}_k = \{ z^{r'}_{k-p+1:k} \}\in \Delta \mathcal{Z}^{r,r'}_k$). 

Robot $r$ verifies the consistency of observations $\Delta \mathcal{Z}^{r,r'}_k$, i.e. $\mathtt{con}_{\bar{a}}(\Delta \mathcal{Z}^{r,r'}_k)=\mathtt{true}$ (Definition~\ref{def:consis-obs}), in favor of the action $\bar{a}$ derived in Step 1. For each such possible realization, denoted abstractly by $\tilde{z}^{r'}\in \Delta \mathcal{Z}^{r,r'}_k$ 
, robot $r$ first constructs  a plausible corresponding belief of robot $r'$, denoted as 
\begin{equation}
	b_k^{r'\mid r}(\tilde{z}^{r'}) \triangleq \prob{x_k \mid \leftidx{^c} H_k^{r,r'}, \tilde{z}^{r'}}.
	\label{eq:brtagr}
\end{equation}		
Note the belief is still over the state $x_k$, and it varies for different values of $\tilde{z}^{r'}$ it is conditioned upon.

In practice, the belief \eqref{eq:brtagr} can be calculated 
in a Bayesian manner either by down-dating the observations $\Delta  H^{r',r}_k$
 from $b^r_k$ \eqref{eq:br2}  and updating with $\tilde{z}^{r'}$, or equivalently, directly from  $\prob{x_k \mid \leftidx{^c} H_k^{r,r'}}$, which would have to be maintained. For instance,  for $p=1$, these two alternatives correspond to
\begin{align}
	\label{eq:brtagr-bayesian-update}
	\begin{split}
		b_k^{r'\mid r}(\tilde{z}^{r'}) &= b^r_k \cdot \frac{ \prob{z^{r}_k\mid \leftidx{^c} H_k^{r,r'}} \prob{ \tilde{z}^{r'}_k \mid x_k} }{\prob{z^{r}_k \mid x_k}\prob{\tilde{z}^{r'}_k\mid \leftidx{^c} H_k^{r,r'}}} \\
		&= \prob{x_k \mid \leftidx{^c} H_k^{r,r'}} \cdot \frac{\prob{ \tilde{z}^{r'}_k \mid x_k}}{\prob{\tilde{z}^{r'}_k\mid \leftidx{^c} H_k^{r,r'}}}.
	\end{split}
\end{align}
Then, for each  $\tilde{z}^{r'} \in \Delta \mathcal{Z}^{r,r'}_k$ of $r'$, robot $r$ evaluates the objective function $J(b_k^{r'\mid r}(\tilde{z}^{r'}), a)$ for different candidate joint actions $a\in \mathcal{A}_{k+}$. 
This is illustrated in Figure \ref{ac-2-obs} using blue squares.  Generally, each $\tilde{z}^{r'} \in \Delta \mathcal{Z}^{r,r'}_k$  yields its own $J$-values.  
Moreover,  we do not necessarily  expect either of these values to match the objective values $J(b^r_k, \bar{a})$ calculated by robot $r$ in Step 1 (black triangles), since generally the observation models and spaces of different robots could vary, \rev{although the robots are assumed to know the models of the other robots}. Importantly, with this formulation, the actual observation that robot $r'$ captured will be considered, since $\Delta  H^{r,r'}_k \in \Delta \mathcal{Z}^{r,r'}_k$.

Regardless of the magnitude of $J$-values, it may happen that for all $\tilde{z}^{r'}$ the same joint action is chosen, and that action is identical to the one chosen in Step 1. Such a situation is depicted in Figure \ref{ac-2-obs} where $\bar{a}$ is the best joint action in both steps 1 and 2. In other words, in this scenario, regardless of what the actual observation of $r'$ is, robot $r$ can be assured that when $r'$ performs its own decision making, i.e. step 1, it will necessarily choose the same joint action as the one chosen by $r$. Therefore, $r$ checks if the best action selected for each $\tilde{z}^{r'}\in \Delta \mathcal{Z}_{k}^{r,r'}$ is the action $\bar{a}$ derived in Step 1, i.e.~if $\mathtt{con}_{\bar{a}}(\Delta \mathcal{Z}^{r,r'}_k)=\mathtt{true}$ holds.
Thus, $r$ captures the reasoning about the action selection by $r'$. 

Yet, at this point, robot $r$ cannot guarantee that robot $r'$ will also reach the same conclusion, i.e. regardless of the actual observation of robot $r$ (that is unavailable to robot $r'$), the same joint action will be selected by robots $r'$ and $r$. Therefore, Steps 1 and 2 are insufficient to guarantee \mrac between the two robots, which brings us to Step 3.

\subsubsection*{\textbf{Step 3: }\textit{Robot $r$ mimics the reasoning done by robot $r'$ that mimics the reasoning done by robot $r$}}\label{step-3}
Robot $r'$, on its side, similarly performs Steps 1 and 2. In Step 1, $r'$ evaluates the objective function for different candidate joint actions based on its own belief $b^{r'}_k$ conditioned on $\leftidx{^c} H_k^{r',r}$ and $\Delta  H^{r,r'}_k$ (see \eqref{eq:br2}). 
Refer to Figure~\ref{ac-2-obs-other}. Since 
$\Delta  H^{r,r'}_k \in \Delta \mathcal{Z}^{r,r'}_k$, this calculation will be considered by robot $r$ as a part of its Step 2. 

Robot $r'$ also performs Step 2, on its side, in which $r'$ reasons about all possible observations of $r$, i.e. $\Delta \mathcal{Z}_{k}^{r',r}$.
Robot $r'$ thus calculates $b_k^{r\mid r'}(\tilde{z}^{r}) \ \forall \tilde{z}^{r}\in\Delta \mathcal{Z}_{k}^{r',r}$. 
This is illustrated in Figure~\ref{ac-2-obs-other} by blue triangles.  Now, if robot $r$, on its side, mimics this reasoning done by $r'$, robot $r$ can perceive what $r'$ thinks about the reasoning done by $r$. 

Put formally, robot $r$ verifies if all observations in $\Delta \mathcal{Z}^{r',r}_k$ are in favor of the action $\bar{a}$ derived in Step 1 of $r$, i.e.~$\mathtt{con}_{\bar{a}}(\Delta \mathcal{Z}^{r',r}_k)=\mathtt{true}$.
So, $r$ computes $b_k^{r\mid r'\mid r}(\tilde{z}^{r})$,
\begin{equation}
	\label{eq:br-rtag-r}
	b_k^{r\mid r'\mid r}(\tilde{z}^{r})\triangleq \prob{x_k \mid \leftidx{^c} H_k^{r',r}, \tilde{z}^{r}},
\end{equation}
and evaluates $J(b_k^{r\mid r'\mid r}(\tilde{z}^{r}), a)$ for each $\tilde{z}^{r}\in \Delta \mathcal{Z}_{k}^{r',r}$ and for all candidate joint actions $a\in \mathcal{A}_{k+}$. 
Thus, $r$ captures the reasoning about the action selection by itself reasoned by $r'$.

Combining Steps 1-3, robot $r$ checks for \mrac by reasoning about selecting a consistent joint action  by $r$ and $r'$, which involves considering the observations in $\Delta\mathcal{Z}_k$ (defined in~\eqref{eq:dZ}).  When the same joint action $\bar{a}$ is chosen in Steps 1-3, as illustrated in Figure~\ref{ac-3-obs}, \mrac is identified by robot $r$ using \textsc{VerifyAC}. Thus, despite having inconsistent beliefs, robots $r$ and $r'$ are action consistent, i.e. identify the same joint action chosen by both the robots at time $k$.

\begin{theorem}\label{Thm:Verifymrac}
	Steps 1-3 of \textsc{VerifyAC} are  necessary and sufficient for any robot $r$ to find \mrac, if \mrac exists, with the robots in $\Gamma = \{r,r'\}$ having  inconsistent beliefs.
\end{theorem}
\begin{proof}
	\textit{Steps 1-3 are sufficient:}
	In Steps 1-3 of \textsc{VerifyAC}, robot $r$ analyzes the observation spaces of the robots $\Delta \mathcal{Z}_k$ (defined in \eqref{eq:dZ}) as per the algorithm. 
	Among the observations in $\Delta \mathcal{Z}_k$ we have the actual local observations $\Delta  H^{r,r'}_k$ and $\Delta  H^{r',r}_k$.
	When \mrac exists, 
	$\mathtt{con}_{\bar{a}}(\Delta \mathcal{Z}_k)$ becomes $\mathtt{true}$ in favor of some joint action $\bar{a}$. 
	This implies that the $J$-values corresponding to the actual observation values are also consistent in favor of $\bar{a}$.
	We know that the actual observations give the joint action preferences with zero uncertainty.
	Therefore, $\mathtt{con}_{\bar{a}}(\Delta \mathcal{Z}_k)=\mathtt{true}$ in Steps 1-3 implies \mrac for robots $r$ and $r'$.
	
	\textit{Steps 1-3 are necessary:}
	Without communication, robot $r$ is not aware of the actual observation $\Delta  H^{r,r'}_k$ of the other robot $r'$.
	In Steps 1-3 of our algorithm, robot $r$ exhaustively considers all observations in $\Delta \mathcal{Z}_k$ that include the actual observations of both the robots in $\Gamma$.
	The actual observations give the joint action preferences of the robots with zero uncertainty.
	If we remove a randomly selected observation $z\in \Delta \mathcal{Z}_k$ in any of the steps in \textsc{VerifyAC}, there is a possibility of $z$ being the actual observation of a robot. This does not guarantee selecting the joint action preference correctly.
	Hence, Steps 1-3 are \emph{necessary} to verify \mrac.\qed
\end{proof}

However, it may often happen that for given beliefs $b^r_k$ and $b^{r'}_k$,  an \mrac does not exist, i.e.~after performing Steps 1-3 of \textsc{VerifyAC} we cannot find a joint action $\bar{a}$ which is chosen by all the steps of \textsc{VerifyAC}. Mathematically,
\begin{equation}\label{no-mrac}
	\nexists \bar{a}\in \mathcal{A}_{k+}\ \ \mathtt{con}_{\bar{a}}(\Delta \mathcal{Z}_k) = \mathtt{true}.
\end{equation}
In the next section, we discuss such scenarios when \mrac is not satisfied, and describe our approach to initiate different communications (\textsc{comm}s) until \mrac is enforced.

\subsection{Self-triggered decision for communication}\label{sec:SelfComm}
We present an algorithm \textsc{EnforceAC} that enforces \mrac via \textsc{comm}s when \textsc{VerifyAC} fails to find \mrac.
Each robot assesses the requirement of a \textsc{comm} by its own reasoning and it \emph{self-triggers} a \textsc{comm} whenever needed.

When \eqref{no-mrac} holds, a \textsc{comm} is required. A \textsc{comm} can send a local observation, either from $r$ to $r'$, from $r'$ to $r$, or in both directions, based on the conditions specified below, and presented in Figure~\ref{fig:verify-ac-flow}.

Let $\bar{a}$ be the best joint action calculated by robot $r$ in Step 1. Robot $r$ reasons about necessity of a \textsc{comm} from itself to robot $r'$ and self-triggers the \textsc{comm} if:
\begin{itemize}
	\item From the perspective of $r$, the observations in Step 3 of \textsc{VerifyAC} are \emph{not consistent} in favor of the same action $\bar{a}$, i.e.~$\mathtt{con}_{\bar{a}}(\Delta \mathcal{Z}_k^{r',r}) = \mathtt{false}$.
\end{itemize}
In such a case, robot $r$ deduces that robot $r'$, which reasons about observations of robot $r$ by considering the observation space $\Delta \mathcal{Z}^{r',r}_k$ (as part of Step 2 of $r'$), will find some inconsistent observation realizations of $r$ in $\Delta \mathcal{Z}^{r',r}_k$. Therefore,  $r$ sends its local observation(s) from $\Delta  H^{r',r}_k$ to  $r'$.

Additionally, a \textsc{comm} from $r$ to $r'$ will be triggered if:
\begin{itemize}
	\item Step 2 of $r$ gives $\mathtt{con}_{a'}(\Delta \mathcal{Z}_k^{r,r'}) = \mathtt{true}$ where $a'\neq \bar{a}$ ($\bar{a}$ is the action chosen in Step 1 of $r$).
	In the case when $\Delta \mathcal{Z}_k^{r,r'} = \emptyset$, i.e. all the observations of $r'$ were communicated to $r$, this calculation is performed relatively just to the common history.
\end{itemize}
In this case, $r$ perceives that the observations in Step 2 are consistent in favor of $a'$, though $a'$ does not match with the best action $\bar{a}$ in Step 1 of $r$.
Also, $r$ does not know whether $r'$ detects the same inconsistency on its side (comparing Step 1 and Step 3 of $r'$).
However, it is required to modify the action preferences in Step 2 of $r$ via \textsc{comm}.
So, $r$ sends its local observation to $r'$.
This \textsc{comm} modifies $\leftidx{^c} H^{r,r'}_k$ which, in turn, modifies the action preferences in Step 2 of $r$.

Similarly, robot $r$ reasons about a \textsc{comm} from $r'$ to $r$ if:
\begin{itemize}
	\item From the perspective of $r$, the observations in step 2 of \textsc{VerifyAC} are not consistent in favor of action $\bar{a}$, i.e. $\mathtt{con}_{\bar{a}}(\Delta \mathcal{Z}_k^{r,r'}) = \mathtt{false}$.
\end{itemize}
Due to the inconsistent observations in Step 2, robot $r$ needs access to more observations of $r'$; in other words, robot $r$ understands the necessity of a \textsc{comm} from $r'$.
For robot $r$, one possibility is to ask $r'$ to initiate a \textsc{comm}. However, $r$ knows that the \textsc{comm} from $r'$ to $r$ will happen automatically without any intervention by $r$.
This is because robot $r$ deduces that robot $r'$, on its side, analyzes the necessity of the same \textsc{comm}.
That is, when $r'$ will execute its Step 3 it will find inconsistent observations of $r$ and then $r'$ will trigger a \textsc{comm} from itself to $r$.

\begin{figure}[t]
	\centering
	\includegraphics[width=0.45\textwidth]{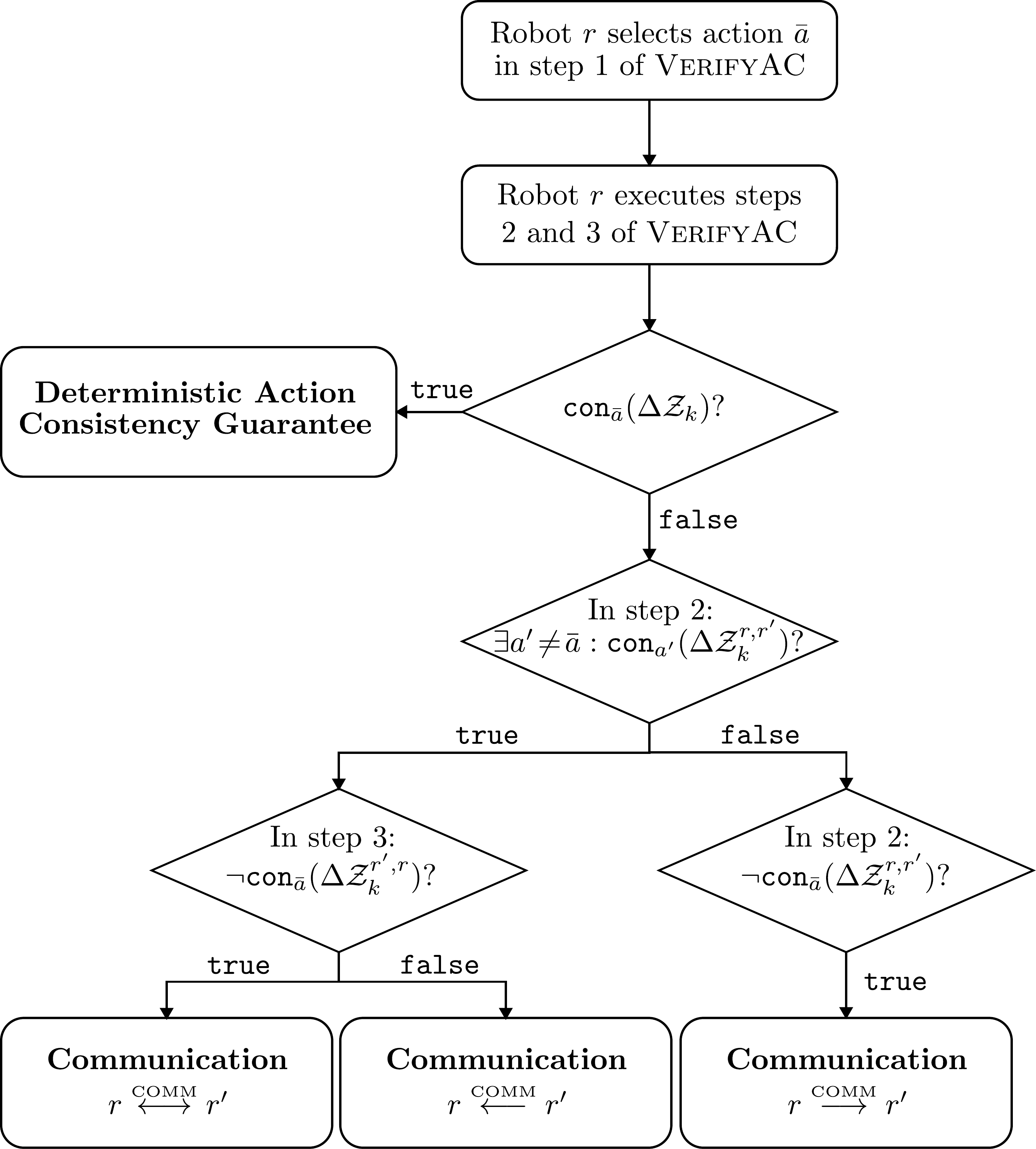}
	\caption{Flowchart of \rev{\textsc{EnforceAC}} and the conditions for communications.}
	\label{fig:verify-ac-flow}
\end{figure}

After a \textsc{comm}, $\leftidx{^c} H^{r,r'}_k$, $\Delta  H^{r',r}_k$, $\Delta  H^{r,r'}_k$, $\Delta \mathcal{Z}^{r',r}_k$ and $\Delta \mathcal{Z}^{r,r'}_k$ are updated with the transmitted observations. As a result, at least one of the histories $H^r_k$ and $H^{r'}_k$ gets updated.

For instance, if robot $r$ communicated some observation $z^r_j \in \Delta  H^{r',r}_k$ to robot $r'$, then, from the perspective of robot $r$, it updates  $\leftidx{^c} H^{r,r'}_k \leftarrow \leftidx{^c} H^{r,r'}_k \cup \{z^r_j\}$, $\Delta  H^{r,r'}_k \leftarrow \Delta  H^{r,r'}_k \setminus \{z^r_j\}$. In this case,  $H^{r}_k = \{\leftidx{^c} H_k^{r,r'} , \Delta  H^{r',r}_k\}$ remains the same as robot $r$ did not receive any new observation(s). Robot $r$ also updates the joint observation space $\Delta \mathcal{Z}^{r',r}_k$ from \eqref{eq:dZrtagr} to exclude the corresponding observation space $\mathcal{Z}^r_j$ of the actual observation $z^r_j$. Robot $r'$, upon receiving the observation $z^r_j$, does a similar update to $\leftidx{^c} H^{r,r'}_k$, $\Delta  H^{r,r'}_k$ and $\Delta \mathcal{Z}^{r',r}_k$. Consequently, $H^{r'}_k = \{\leftidx{^c} H_k^{r,r'} , \Delta  H^{r,r'}_k\}$ is updated.

Given the updated histories $H^{r}_k$ and $H^{r'}_k$, we update the beliefs $b^r_k$ and $b^{r'}_k$ according to \eqref{eq:br2}, and execute \textsc{VerifyAC} to check if \mrac exists with the updated beliefs. If \mrac is not satisfied, again \textsc{comm}s are triggered.
This continues until \mrac is achieved.
Thus, \textsc{EnforceAC} enforces \mrac via \textsc{comm}s even though \mrac is not satisfied initially.

%


\rev{
\textit{Time complexity of \textsc{EnforceAC}:} The \emph{worst case} time complexity of \textsc{EnforceAC} is $\mathcal{O} (n^p)$, where $p$ is the number of previous time points before which the robots had a consistent history and $n \triangleq \max \{ \vert \mathcal{Z}^r \vert, \vert \mathcal{Z}^{r'} \vert\}$.}
\begin{proof}
	\rev{
	\textsc{EnforceAC} calls \textsc{VerifyAC}. First we analyze the time complexity of \textsc{VerifyAC} for a specific value of $p$, which iterates over all possible observations in $\Delta \mathcal{Z}_{k}^{r,r'}$ and $\Delta \mathcal{Z}_{k}^{r',r}$ in Steps 2 and 3, respectively. For each such observation \textsc{VerifyAC} calculates and compares between the corresponding $J$-values for different candidate joint actions. Whenever \textsc{VerifyAC} finds an inconsistent observation, a \textsc{comm} is triggered. Overall, \textsc{VerifyAC} incurs a runtime of $\mathcal{O}(|\Delta \mathcal{Z}_{k}^{r',r}| + |\Delta \mathcal{Z}_{k}^{r,r'}|) = \mathcal{O}(|\mathcal{Z}^{r}|^p + |\mathcal{Z}^{r'}|^p) = \mathcal{O}(n^p)$. Note that calculation and comparison of $J$-values for each observation takes a constant amount of time. 
	
	Inconsistent observations in \textsc{VerifyAC} lead to triggering at most $p$ number of \textsc{comm}s for each robot by \textsc{EnforceAC}. 
	This happens when during every \textsc{comm} each of the two robots sends its unshared local observation from every time point in the time range $[k-p+1,k]$.
	After each such \textsc{comm}, \textsc{VerifyAC} is invoked again by \textsc{EnforceAC} to check for \mrac, where the value of $p$ is decreased by 1 after each \textsc{comm}.
	So, \textsc{EnforceAC} invokes at most $p+1$ number of \textsc{VerifyAC} runs, where each run $i$ performs $n^{p+1-i}$ comparisons of $J$-values, i.e. $\sum_{i=1}^{p+1} n^{p+1-i} = \sum_{i=0}^{p} n^i = \frac{n^{p+1}-1}{n-1} \approx n^{p}$ for big $n$ values.
	Hence the complexity of \textsc{EnforceAC} is $\mathcal{O}(n^p)$.\qed}
\end{proof}


\begin{theorem}
	\textsc{EnforceAC} converges to \mrac in a finite amount of time, even if \mrac does not exist initially.
\end{theorem}

\begin{proof}
	During a \textsc{comm}, the unknown observation sequences can be shared partially, i.e. some of the observations in the time range $[k-p+1,k]$ can be shared to the other robot.
	In the worst case, to ensure \mrac, robot $r$ may have to send its entire actual observation in time range $[k-p+1,k]$ to $r'$.
	Sharing all the observations takes finite amount of time. 
	Hence, \textsc{EnforceAC} achieves \mrac in finite time.\qed
\end{proof}





\section{Relaxation and Simplification of \textsc{VerifyAC} with performance guarantees}
\label{sec:simplification}

In the previous section we described the algorithm \textsc{VerifyAC} in which robot $r$ requires \emph{all the observations} in steps 2 and 3 to be in favor of a particular joint action (the action selected in step 1).
In the variants of \textsc{VerifyAC} discussed thus far, all the observations favoring a particular action contribute to deterministic guarantee on action consistency (\mrac).
However, satisfying consistency for \emph{all} the observations in step 2 ($\Delta \mathcal{Z}^{r,r'}_{k}$) and step 3 ($\Delta \mathcal{Z}^{r',r}_{k}$) is quite demanding and less likely to happen without a good number of \textsc{comm}s.
Thus, a high number of \textsc{comm}s is inevitable for satisfying \mrac with the previous variants of \textsc{VerifyAC}.

In section~\ref{subsec:sverifyac}, we design a new variant of \textsc{VerifyAC} named \textsc{R-VerifyAC} that offers a relaxation for satisfying \mrac when \emph{some}, instead of all, the observations are in favor of a particular action.
This relaxation in \mrac satisfaction contributes to further reduction in the number of \textsc{comm}s.
Here, 'R' in \textsc{R-VerifyAC} stands for relaxation of the notion of action consistency.
Though \textsc{R-VerifyAc} provides performance gain in terms of reduced \textsc{comm}s it still needs to iterate over all the observations in $\Delta \mathcal{Z}^{r,r'}_k$ and $\Delta \mathcal{Z}^{r',r}_k$.
In Section~\ref{subsec:rverifyac-simp} we propose a simplified variant of \textsc{R-VerifyAC} called \textsc{R-VerifyAC-simp} that aims to reduce the computation time. \textsc{R-VerifyAC-simp}, in contrast to \textsc{R-VerifyAC}, iterates over a smaller subset of all possible observations in steps 2 and 3 of \textsc{R-VerifyAC}.

\subsection{Algorithm \textsc{R-VerifyAC}}\label{subsec:sverifyac}
Algorithm \textsc{R-VerifyAC} modifies the criteria of satisfying \mrac in \textsc{VerifyAC}.
Suppose, in step 1, robot $r$ selects the most preferred (highest $J(.)$ value) action $\bar{a}\in \mathcal{A}_{k+}$.
Unlike \textsc{VerifyAC}, algorithm \textsc{R-VerifyAC} requires only some of the observations in steps 2 and 3 to be in favor of action $\bar{a}$, instead of having exhaustively all the observations in favor of $\bar{a}$. 
Consequently, more than one rank-1 actions may exist because different subsets of observations may choose different rank-1 actions, instead of choosing a single unanimous rank-1 action corresponding to all the observations.  
To address this, the algorithm takes into account the \emph{likelihood} of each observation in steps 2 and 3, and computes the \emph{cumulative likelihood of the observations} favoring individual actions.
The rank-1 actions with the highest cumulative likelihood values, in each of step 2 and 3, is considered to be the most preferred one.

Next, we define the cumulative likelihood of observations for an action $a$. Let $Z$ be the entire set of observations (in step 2 or 3), and $Z' \triangleq \mathtt{cobs}_a(Z) \subseteq Z$ the consistent set of observations in favor of $a\in \mathcal{A}_{k+}$.  
We have already the latter in Section \ref{subsec:ACwithDiffBeliefs} (Definition \ref{def:consis-obs}). 

\begin{definition}[Cumulative likelihood of obs. favoring $a$]\label{def:lkl}
	From the perspective of robot $r$, consider a set of observations $Z$, a consistent set of observations $\mathtt{cobs}_a (Z)\in Z$ in favor of  the joint action $a$ and the consistent history $\leftidx{^c} H_k^{r,r'}$ of the robots. 
	The cumulative likelihood of the observations in favor of an action $a\in \mathcal{A}_{k+}$, denoted $\mathtt{Cl}_a(Z)$, is the summation of likelihood of the observations in $\mathtt{cobs}_a (Z)$, i.e. :
	\begin{equation}\label{eq:cumulative-likelihood}
		\mathtt{Cl}_a(Z) \triangleq \sum_{z\in \mathtt{cobs}_{a}(Z)}^{} \prob{z\mid \leftidx{^c} H_k^{r,r'}}.
	\end{equation}
\end{definition}
Now, $\prob{z\mid  \leftidx{^c} H_k^{r,r'}}$ is the likelihood of an observation sequence $z$ from the set of observations $\mathtt{cobs}_a(Z)$. This can be calculated by marginalizing over the corresponding sequence of states $x$ as follows: 
\begin{equation}\label{eq:meas-likelihood}
	\prob{z\mid  \leftidx{^c} H_k^{r,r'}} =\sum_{x\in \mathcal{X}} \prob{z \mid x} \prob{x \mid  \leftidx{^c} H_k^{r,r'}}.
\end{equation}


\begin{figure}[t]
\begin{center}
\subfigure[Prob. guarantee on \mrac]{
\resizebox{0.47\linewidth}{!}{
\begin{tikzpicture}
	\pgfplotsset{every tick label/.append style={font=\Large}}
	\begin{axis}[
		symbolic x coords={1,2,3,4,5,6,7,8,9,10,11,12},
		enlargelimits=0.05,
		legend style={at={(0.5,1.22)}, anchor=north,legend columns=-1},
		ylabel=$J(.)$,
		ymax=90, ymin=40,
		x tick label style={rotate=0, anchor=east, align=center, yshift=-0.5cm, font=\LARGE},
		y tick label style={color=white, font=\normalsize},
		xtick={1,2,3,4,5,6,7,8,9,10,11,12},
		xticklabels={,,$\bar{a}$,,,,,,,,,$\bar{a}'$,},
		xmin=1, xmax=12,
		ymin=0, ymax=180,
		height=6cm, width=8cm,
		]
		\addplot[black, mark=triangle*, mark options={scale=3}]
		coordinates { (2, 177) (11,156)};
		\addlegendentry{\normalsize$J(b^{r'\mid r}(z^{r'}),a))$}
		
		\addplot[cyan, mark=square*, mark options={scale=2}]
		coordinates { (2, 83) (11,65)};
		\addlegendentry{\normalsize$J(b^{r\mid r' \mid r}(z^{r}),a))$}
		
		\addplot[black, mark=triangle*, mark options={scale=3}]
		coordinates { (2, 147) (11,133)};
		
		\addplot[black, mark=triangle*, mark options={scale=3}]
		coordinates { (2, 100) (11,118)};

		\addplot[cyan, mark=square*, mark options={scale=2}]
		coordinates { (2, 35) (11,50)};
		\addplot[cyan, mark=square*, mark options={scale=2}]
		coordinates { (2, 20) (11,2)};
		
		

		\draw (50,180) node [anchor=north west][inner sep=0.75pt]   [align=left] {0.3};
		\draw (50,153) node [anchor=north west][inner sep=0.75pt]   [align=left] {0.5};
		\draw (50,123) node [anchor=north west][inner sep=0.75pt]   [align=left] {0.2};
		
		\draw (50,88) node [anchor=north west][inner sep=0.75pt]   [align=left] {0.4};
		\draw (50,56) node [anchor=north west][inner sep=0.75pt]   [align=left] {0.3};
		\draw (50,26) node [anchor=north west][inner sep=0.75pt]   [align=left] {0.3};		
		
\end{axis}
\label{rverify-det}
\end{tikzpicture}
}}
\subfigure[\textsc{comm} from $r$ to $r'$]{
\resizebox{0.47\linewidth}{!}{
	\begin{tikzpicture}
		\pgfplotsset{every tick label/.append style={font=\Large}}
		\begin{axis}[
			symbolic x coords={1,2,3,4,5,6,7,8,9,10,11,12},
			enlargelimits=0.05,
			legend style={at={(0.5,1.22)}, anchor=north,legend columns=-1},
			ylabel=$J(.)$,
			ymax=90, ymin=40,
			x tick label style={rotate=0, anchor=east, align=center, yshift=-0.5cm, font=\LARGE},
			y tick label style={color=white, font=\normalsize},
			xtick={1,2,3,4,5,6,7,8,9,10,11,12},
			xticklabels={,,$\bar{a}$,,,,,,,,,$\bar{a}'$,},
			xmin=1, xmax=12,
			ymin=0, ymax=180,
			height=6cm, width=8cm,
			]
			\addplot[black, mark=triangle*, mark options={scale=3}]
			coordinates { (2, 177) (11,156)};
			
			\addplot[black, mark=triangle*, mark options={scale=3}]
			coordinates { (2, 133) (11,151)};
			
			\addplot[black, mark=triangle*, mark options={scale=3}]
			coordinates { (2, 118) (11,100)};
			
			\addplot[cyan, mark=square*, mark options={scale=2}]
			coordinates { (2, 93) (11,65)};
			
			\addplot[cyan, mark=square*, mark options={scale=2}]
			coordinates { (2, 60) (11,35)};
			\addplot[cyan, mark=square*, mark options={scale=2}]
			coordinates { (2, 2) (11,20)};
			
			
			
			\draw (50,180) node [anchor=north west][inner sep=0.75pt]   [align=left] {0.2};
			\draw (50,156) node [anchor=north west][inner sep=0.75pt]   [align=left] {0.6};
			\draw (50,123) node [anchor=north west][inner sep=0.75pt]   [align=left] {0.2};
			
			\draw (50,93) node [anchor=north west][inner sep=0.75pt]   [align=left] {0.3};
			\draw (50,61) node [anchor=north west][inner sep=0.75pt]   [align=left] {0.4};
			\draw (50,26) node [anchor=north west][inner sep=0.75pt]   [align=left] {0.3};		
			
		\end{axis}
		\label{rverify-prob}
	\end{tikzpicture}
}}
\caption{Illustration of \textsc{R-VerifyAC} from the perspective of robot $r$. Consider, robot $r$ selects action $\bar{a}$ in step 1. Even though all the observations are not in favor of $\bar{a}$, \textsc{R-VerifyAC} provides \textbf{(a)} probabilistic guarantee on \mrac, or \textbf{(b)} a \textsc{comm} is triggered from $r$ to $r'$. Triangles and squares represent the $J$-values corresponding to the observations in steps 2 and 3 respectively. Let, $1-\epsilon = 0.1$. 
\textbf{(a)} For action $\bar{a}$, in step 2, $\leftidx{^2}{\mathtt{Cl}_{\bar{a}}} = \leftidx{^2}{\mathtt{Cl}^*} = 0.3+0.5 = 0.8 > 1-\epsilon$; in step 3, $\leftidx{^3}{\mathtt{Cl}_{\bar{a}}} = 0.4+0.3 = 0.7 > 1-\epsilon$. So, according to Definition \ref{def:epsilon-mrac}, $\emrac^r = \mathtt{true}$, and according to Theorem \ref{thm:probabilistic-guarantees-rverifyac}, robot $r$ will not trigger a \textsc{comm}. Probability of \mrac with both robots choosing $\bar{a}$ is $\leftidx{^2}{\mathtt{Cl}_{\bar{a}}} = 0.8$; probability of robot $r$ choosing $\bar{a}$ but $r'$ choosing an inconsistent action $a$ ($\neq \bar{a}$) is $0.2$; and probability of $r'$ triggering a \textsc{comm} is $1-0.8-0.2 = 0$.
\textbf{(b)} For action $\bar{a}$, in step 2, $\leftidx{^2}{\mathtt{Cl}_{\bar{a}}} = 0.2+0.2 = 0.4 < 1-\epsilon$ and $\leftidx{^2}{\mathtt{Cl}_{\bar{a}}} \not = \leftidx{^2}{\mathtt{Cl}^*}$. So, according to Definition \ref{def:epsilon-mrac}, $\emrac^r = \mathtt{false}$, and according to Theorem \ref{thm:probabilistic-guarantees-rverifyac}, robot $r$ will trigger a \textsc{comm} from itself to $r'$. 
}
\label{ex-relaxed-policy}
\end{center}
\end{figure}

%
\noindent Here, we describe in more detail how the criteria of satisfying \mrac is relaxed in steps 2 and 3 of \textsc{R-VerifyAC}, using the notion of cumulative likelihood of observations.
Recall, robot $r$ chooses the joint action $\bar{a}\in \mathcal{A}_{k+}$ in its step 1.
Consider, in steps 2 and 3, the consistent sets of observations in favor of action $\bar{a}$ are respectively
\begin{equation}\label{cobs}
	\begin{split}
	\mathtt{cobs}_{\bar{a}}(\Delta \mathcal{Z}^{r,r'}_k) \subseteq \Delta \mathcal{Z}^{r,r'}_k,\\
	\mathtt{cobs}_{\bar{a}}(\Delta \mathcal{Z}^{r',r}_k) \subseteq \Delta \mathcal{Z}^{r',r}_k.
	\end{split}
\end{equation}
This implies there may exist a non-empty set of observations $\Delta \mathcal{Z}^{r,r'}_k \setminus\mathtt{cobs}_{\bar{a}}(\Delta \mathcal{Z}^{r,r'}_k)$ and $\Delta \mathcal{Z}^{r',r}_k \setminus\mathtt{cobs}_{\bar{a}}(\Delta \mathcal{Z}^{r',r}_k)$ which are not in favor of $\bar{a}$.
Even though some of the observations are not in favor of $\bar{a}$, using the notion of $\mathtt{Cl}(.)$, \textsc{R-VerifyAC} provides guarantees on \mrac, deterministically or probabilistically, while reducing the number of \textsc{comm}s.

\rev{For brevity, we respectively denote the cumulative likelihood values of action $a\in \mathcal{A}_{k+}$ in steps 2 and 3 as
\begin{equation}\label{eq:cumulative-likelihood-steps-2-3}
	\begin{split}
	\leftidx{^2}{\mathtt{Cl}_a} \equiv {\mathtt{Cl}_a}(\Delta \mathcal{Z}^{r,r'}_k), \\
	\leftidx{^3}{\mathtt{Cl}_a} \equiv {\mathtt{Cl}_a}(\Delta \mathcal{Z}^{r',r}_k).
	\end{split}
\end{equation}
In step $i\in \{2,3\}$, the action having the highest $\mathtt{Cl}$ value is denoted $\leftidx{^i}{\mathtt{Cl}^*}$.} When the action $a$ has the highest $\mathtt{Cl}$ value:
\begin{equation}
	\begin{split}
		\leftidx{^i}{\mathtt{Cl}^*} = \leftidx{^i}{\mathtt{Cl}_{a}}
		\text{ such that } 
		\leftidx{^i}{\mathtt{Cl}_{a}} > \leftidx{^i}{\mathtt{Cl}_{a'}} \ \ \forall a' \in \mathcal{A}_{k+}\setminus \{a\}.	
	\end{split}
\end{equation}
Note:
\begin{equation}\label{eq:cumulative-likelihood-consistent-observations-relation}
	\begin{split}
		\mathtt{con}_{a}(\Delta \mathcal{Z}^{r,r'}_k) = \mathtt{true} \Longleftrightarrow \leftidx{^2}{\mathtt{Cl}}_a = \leftidx{^2}{\mathtt{Cl}}^* \equiv 1, \\
		\mathtt{con}_{a}(\Delta \mathcal{Z}^{r',r}_k) = \mathtt{true} \Longleftrightarrow \leftidx{^3}{\mathtt{Cl}}_a = \leftidx{^3}{\mathtt{Cl}}^* \equiv 1.
	\end{split}
\end{equation}


	From the perspective of robot $r$, we define an indicator function $\emrac^r$ that indicates if a given joint action $a$ satisfies consistency criteria for both the robots, based on cumulative likelihood of the joint action $a$.

\begin{definition}[$\emrac^r$]\label{def:epsilon-mrac}
	Consider robot $r$ executes steps 2 and 3 of \textsc{R-VerifyAC}.
	For any action $a \in \mathcal{A}_{k+}$ and a user-provided threshold value $1-\epsilon \in (0,1]$, we define the function $\emrac^r : a \mapsto \{ \mathtt{false},\mathtt{true} \}$ as:
	\begin{equation}\label{eq:epsilon-mrac}
		\begin{aligned}
			\emrac^r(a) \triangleq &
			\{ \{ \leftidx{^2}{\mathtt{Cl}}_{a} = \leftidx{^2}{\mathtt{Cl}}^{*} \} \lor \{ \leftidx{^2}{\mathtt{Cl}}_{a} > 1-\epsilon \} \} \land \\
			& \{ \{ \leftidx{^3}{\mathtt{Cl}}_{a} = \leftidx{^3}{\mathtt{Cl}}^{*}\} \lor \{ \leftidx{^3}{\mathtt{Cl}}_{a} > 1-\epsilon \} \}.
		\end{aligned}
	\end{equation}
\end{definition}

\begin{theorem}[$\emrac$ Symmetry]\label{cor:epsilon-mrac-symmetry}
	Consider two robots $\{r,r'\}$ executing individually steps 2 and 3 of \textsc{R-VerifyAC}.
	Then, for any joint	action $a \in \mathcal{A}_{k+}$:
	\begin{equation}\label{eq:epsilon-mrac-symmetry}
		\emrac^{r}(a) = \emrac^{r'}(a).
	\end{equation}
\end{theorem}
\begin{proof}
	Due to the symmetry of the calculations in steps 2 and 3, calculations of step 2 of robot $r$ are identical to calculations of step 3 of robot $r'$, and calculations of step 3 of robot $r$ are identical to calculations of step 2 of robot $r'$.
	Therefore, for any action $a \in \mathcal{A}_{k+}$, the conditions of $\emrac^{r}$ for step 2 of robot $r$ are identical to the conditions of $\emrac^{r'}$ for step 3 of $r'$, and the conditions of $\emrac^{r}$ for step 3 of $r$ are identical to the conditions of $\emrac^{r'}$ for step 2 of $r'$.
	Additionally, according to Equation~\eqref{eq:epsilon-mrac}, we note that the conditions of $\emrac^{r}$ for step 2 and for step 3 are the same.
	Therefore, the conditions are symmetric between steps 2 and 3 and between robots $r$ and $r'$, so $\emrac^{r}(a) = \emrac^{r'}(a)$.\qed
\end{proof}
%
%
Since theorem~(\ref{cor:epsilon-mrac-symmetry}) indicates that $\emrac^{r}(a)$ and $\emrac^{r'}(a)$ are the same, we can drop the indices and remain with a unified notation of $\emrac(a)$ for both robots.
Still, we keep the indices for the sake of emphasizing which robot executes the calculations.

Consider robot $r$ selects action $\bar{a}$ in step 1 of \textsc{R-VerifyAC}.
According to definition \ref{def:epsilon-mrac}, if $\emrac(\bar{a}) = \mathtt{true}$, then robot $r$ will declare \mrac and provides \emph{probabilistic guarantees} without triggering a \textsc{comm} to robot $r'$.
Otherwise, a \textsc{comm} will be triggered from $r$ to $r'$.

When robot $r$ selects action $\bar{a}$ and does not trigger a \textsc{comm} to $r'$, there are 3 possible outcomes from the perspective of robot $r'$:
\begin{itemize}
	\item Robot $r'$ selects in its step 1 the same action $\bar{a}$ as robot $r$. Therefore, Action Consistency occurs.
	\item Robot $r'$ selects in its step 1 an action $a$ different than $\bar{a}$ for which $\emrac^{r'}(a) = \mathtt{true}$. Therefore, $r'$ will not trigger a \textsc{comm} to $r$, and Action Inconsistency occurs.
	\item Robot $r'$ selects in its step 1 an action $a$ different than $\bar{a}$, for which $\emrac^{r'}(a) = \mathtt{false}$. Therefore, a \textsc{comm} from $r'$ to $r$ will be triggered.
\end{itemize}
For each of the above outcomes, \textsc{R-VerifyAC} provides \emph{probabilistic guarantees} on their occurrence from the perspective of robot $r$.


We give a theorem related to the probabilistic guarantees on \mrac from the perspective of robot $r$.



\begin{theorem}[Probabilistic guarantees on  \textsc{R-VerifyAC}]\label{thm:probabilistic-guarantees-rverifyac}
	Consider two robots $\{r,r'\}$ executing algorithm \textsc{R-VerifyAC} in a decentralized way. Consider, robot $r$ chooses action $\bar{a}\in \mathcal{A}_{k+}$ in step 1.
	If $\emrac^{r}(\bar{a}) = \mathtt{false}$, robot $r$ triggers a \textsc{comm} to robot $r'$.
	If $\emrac^{r}(\bar{a}) = \mathtt{true}$, robot $r$ does not trigger a \textsc{comm} to robot $r'$ and \textsc{R-VerifyAC} provides the following probabilistic guarantees:
	\begin{itemize}
		\item Multi-Robot Action Consistency:
		\begin{equation}\label{eq:r-verifyac-ac-prob}
			\begin{aligned}[t]
				\prob{\textsc{\mrac} \mid \emrac^{r}(\bar{a}) = \mathtt{true}, H_k^r} = \leftidx{^2}{\mathtt{Cl}}_{\bar{a}}
			\end{aligned}
		\end{equation}
		\item Multi-Robot Action Inconsistency:
		\begin{equation}\label{eq:r-verifyac-not-ac-prob}
			\begin{aligned}[t]
				\prob{\neg \textsc{\mrac} \mid \emrac^{r}(\bar{a}) = \mathtt{true}, H_k^r} = & \\
				\sum_{a \in \mathcal{A}_{k+}} \leftidx{^2}{\mathtt{Cl}}_{a} 
				\cdot \mathbb{I} \{ a \neq \bar{a} \}
				\cdot \mathbb{I} \{ \emrac^{r}(a) \} &
			\end{aligned}
		\end{equation}
		
		
		\item Communication from $r'$ to $r$:
		\begin{equation}\label{eq:r-verifyac-comm-prob}
			\begin{aligned}[t]
				\prob{r \xleftarrow{\textsc{comm}} r' \mid \emrac^{r}(\bar{a}) = \mathtt{true}, H_k^r} = & \\
				\sum_{a \in \mathcal{A}_{k+}} \leftidx{^2}{\mathtt{Cl}}_{a}
				\cdot \mathbb{I} \{ a \neq \bar{a} \}
				\cdot \mathbb{I} \{ \neg \emrac^{r}(a) \}
			\end{aligned}
		\end{equation}
		
		
	\end{itemize}
\end{theorem}

\begin{proof}
	Let, robot $r$ selects action $\bar{a}$ in step 1, and $\emrac^{r}(a) = \mathtt{true}$, meaning the action satisfies the conditions for declaring \mrac.
	According to Theorem~\ref{cor:epsilon-mrac-symmetry}, robot $r$ can calculate $\emrac^{r'}(a)$ for any $a \in \mathcal{A}_{k+}$, by calculating $\emrac^{r}(a)$.
	Therefore, if robot $r'$ selects action $\bar{a}$, it will also declare \mrac, since $\emrac^{r'}(\bar{a})=\emrac^{r}(\bar{a})$, and \emph{Action Consistency} will occur.
	From the perspective of robot $r$, the probability of $r'$ selecting action $\bar{a}$ in its step 1 is given by $\leftidx{^2}{\mathtt{Cl}}_{\bar{a}}$.

	If robot $r'$ selects an action $a\neq \bar{a}$, it will not trigger a \textsc{comm} to $r$ if the action satisfies the conditions for declaring \mrac, i.e. if $\emrac^{r'}(a) = \mathtt{true}$. In this case, \emph{Action Inconsistency} will occur.
	Therefore, from the perspective of robot $r$, the probability of action inconsistency is given by the summation of probabilities of actions which are not $\bar{a}$ and for which $\emrac^{r'}(a) = \mathtt{true}$.

	If robot $r'$ selects an action $a\neq \bar{a}$, it will trigger a \textsc{comm} to $r$ if the action does not satisfy the conditions for declaring \mrac, i.e. if $\emrac^{r'}(a) = \mathtt{false}$. In this case, \emph{Communication} from $r'$ to $r$ will occur.
	Therefore, from the perspective of robot $r$, the probability of communication from $r'$ to $r$ is given by the summation of probabilities of actions which are not $\bar{a}$ and for which $\emrac^{r'}(a) = \mathtt{false}$.\qed
\end{proof}

Now, we discuss the criteria for some special cases of the general probabilistic guarantees given by Theorem \ref{thm:probabilistic-guarantees-rverifyac}.

\begin{corollary}[Deterministic guarantee on \mrac]\label{cor:r-verifyac-deterministic-guarantee}
	Consider, robot $r$ chooses action $\bar{a}\in \mathcal{A}_{k+}$ in step 1 of \textsc{R-VerifyAC}, and $\epsilon\textsc{-MRAC}(\bar{a}) = \mathtt{true}$. \textsc{R-VerifyAC} can provide a \emph{deterministic guarantee} on \mrac, i.e. probability of Action Consistency is 1, similarly to \textsc{VerifyAC}, when $\leftidx{^2}{\mathtt{Cl}}_{\bar{a}} = 1$ and $\leftidx{^3}{\mathtt{Cl}}_{\bar{a}} = 1$.
	These conditions correspond to the conditions presented in \textsc{VerifyAC}, according to \eqref{eq:cumulative-likelihood-consistent-observations-relation}.
%
\end{corollary}


\begin{corollary}[Zero Probability for Action Inconsistency]\label{cor:r-verifyac-no-inconsistency-guarantee}
	Consider, robot $r$ chooses action $\bar{a}\in \mathcal{A}_{k+}$ in step 1 of \textsc{R-VerifyAC}, and $\epsilon\textsc{-MRAC}(\bar{a}) = \mathtt{true}$.
	If all other actions do not satisfy the conditions for declaring \mrac, i.e. $\emrac^{r}(a) = \mathtt{false}$ for all $a \neq \bar{a}$, the probability of \emph{Action Inconsistency} is 0.
\end{corollary}

In other words, in such a case, Theorem~\ref{thm:probabilistic-guarantees-rverifyac} degenerates to the following:
the probability of an inconsistent action selection is 0, while the probability for a consistent action selection remains $\leftidx{^2}{\mathtt{Cl}_{\bar{a}}}$ and the probability for communication from $r'$ to $r$ becomes the complement, $1-\leftidx{^2}{\mathtt{Cl}_{\bar{a}}}$.
\begin{remark}
	Moreover, when provided with the threshold value $\epsilon = 0$, probability of an inconsistent action selection is \emph{always} 0.
	This is because, in this case, the condition $\leftidx{^i}{\mathtt{Cl}}_{a} > 1-\epsilon$ is never satisfied (by probability property),
	so all other action are below the threshold.
\end{remark}

\begin{figure}
	\includegraphics[width=0.45\textwidth]{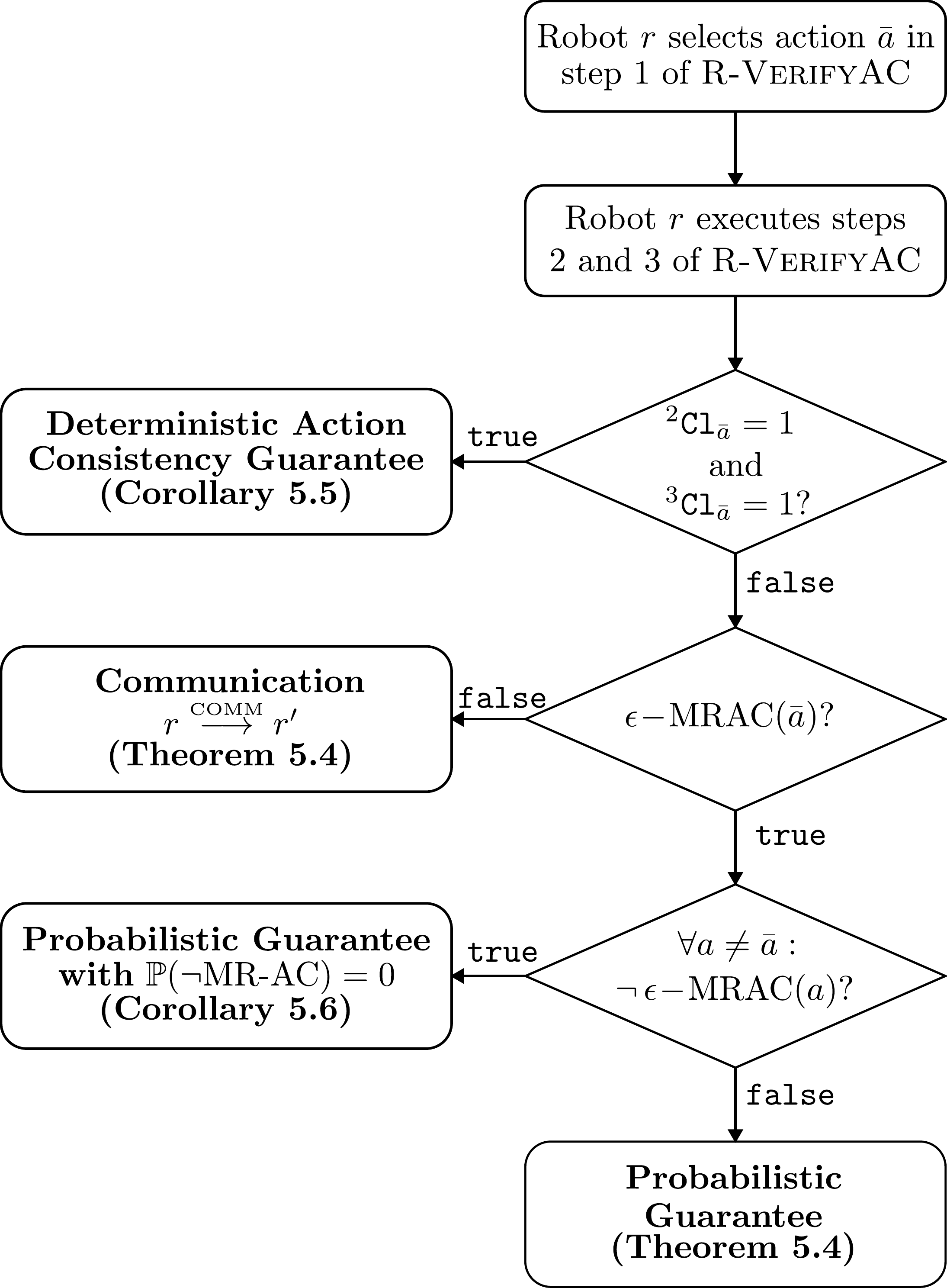}
	\caption{\label{fig:r-verifyac-flow-chart} Flowchart of guarantees provided \rev{by \textsc{EnforceAC} that calls}  \textsc{R-VerifyAC}.}
\end{figure}

\subsection{\textsc{R-VerifyAC-simp}: Simplification of \textsc{R-VerifyAC}} \label{subsec:rverifyac-simp}

While \textsc{R-VerifyAC} relaxes the criteria for \mrac and reduces the number of communications between the robots, it still requires computation exhaustively over the joint observation spaces $\Delta \mathcal{Z}_k^{r,r'}$ and $\Delta \mathcal{Z}_k^{r',r}$ in steps 2 and 3 respectively. 
We now propose another variant of \mbox{\textsc{R-VerifyAC}}, named \textsc{R-VerifyAC-simp} which computes over a \emph{reduced} set of observations $\Delta \mathcal{Z}' \subseteq \Delta \mathcal{Z}_k^{r,r'}$  in step 2 and $\Delta \mathcal{Z}'' \subseteq \Delta \mathcal{Z}_k^{r',r}$ in step 3. \rev{Generally, each agent might select a different subset of observations in execution of the algorithm.
	For two different robots, we claim (Theorem \ref{thm:DifferentSubsets}) that our mechanism \textsc{R-VerifyAC-simp} works even when the robots select different subsets of observations.} 
We now describe how the algorithm computes on the reduced set of observations in steps 2 and 3.

Since $\Delta \mathcal{Z}' \subseteq \Delta \mathcal{Z}_k^{r,r'}$, for any joint action $a$, the consistent set of observations $\mathtt{cobs}$~\eqref{cobs} are:
\begin{equation}\label{eq:Consistent_subset}
	\mathtt{cobs}_{a}(\Delta \mathcal{Z}') \subseteq \mathtt{cobs}_{a}(\Delta \mathcal{Z}_k^{r,r'}).
\end{equation}
Consider, action $\bar{a}$ is selected in step 1.
The reduced observation space $\Delta \mathcal{Z}'$ ($\subseteq \Delta \mathcal{Z}_k^{r,r'}$) will yield only some of the consistent observations in favor of action $\bar{a}$, compared to the entire observation space $\Delta \mathcal{Z}_k^{r,r'}$. 

Therefore, accounting for a subset $\Delta \mathcal{Z}' \subseteq \Delta \mathcal{Z}_k^{r,r'}$ provides a \emph{lower bound} $\mathtt{lb}_{\bar{a}}(\Delta \mathcal{Z}')$ on the likelihood of consistency for the original set $\Delta \mathcal{Z}_k^{r,r'}$, i.e.~a lower bound on the  probability of choosing the same action $\bar{a}$ by the other robot $r'$ (by the description of Step 2):
\begin{align}\label{eq:lbcomparison}
	\mathtt{lb}_{\bar{a}}(\Delta \mathcal{Z}') \  &\triangleq  \!\!\!\!\!\!\!\!\sum_{z\in \mathtt{cobs}_{\bar{a}}(\Delta \mathcal{Z}')}^{} \!\!\!\! \!\!\!\!\!\!\!\! \prob{z\mid  		\leftidx{^c} H_k^{r,r'}}
	\\
	&\leq
	\sum_{z\in \mathtt{cobs}_{\bar{a}}(\Delta \mathcal{Z}_k^{r,r'})}^{} \!\!\!\!\!\!\!\!\!\!\!\!\prob{z\mid  		\leftidx{^c} H_k^{r,r'}} \equiv \mathtt{Cl}_{\bar{a}}( \Delta \mathcal{Z}_k^{r,r'}).
\end{align}
Moreover, by noting that $\sum_{a \in \mathcal{A}_{k+}}\mathtt{Cl}_a( \Delta \mathcal{Z}_k^{r,r'}) =1$, we can also define an \emph{upper bound} $\mathtt{ub}_{\bar{a}}(\Delta \mathcal{Z}')$ on $\mathtt{Cl}_{\bar{a}}( \Delta \mathcal{Z}_k^{r,r'})$ as a function of the lower bounds of all other actions in $\mathcal{A}_{k+}$, 
\begin{align}
	\mathtt{ub}_{\bar{a}}({\Delta \mathcal{Z}'}) \triangleq 1- \sum_{a \in \mathcal{A}_{k+} \setminus \{ \bar{a} \} } \mathtt{lb}_a({\Delta \mathcal{Z}'}).
\end{align}

\rev{
\begin{claim}\label{cla:bounds-shrink}
	If we augment $\Delta \mathcal{Z}'$ by adding an observation $z \in \Delta \mathcal{Z}_k^{r,r'} \setminus \mathcal{Z}'$, the upper limit does not increase and the lower limit does not decrease, for any $a\in \mathcal{A}_{k+}$. That is,
	$\mathtt{ub}_{a}({\Delta \mathcal{Z}'\cup \{z\}}) \not > \mathtt{ub}_{a}({\Delta \mathcal{Z}'}) $ and \mbox{ $\mathtt{lb}_{a}({\Delta \mathcal{Z}'\cup \{z\}}) \not < \mathtt{lb}_{a}({\Delta \mathcal{Z}'})$} for any $a\in \mathcal{A}_{k+}$.
\end{claim}

As per the above claim, if we add observations iteratively to $\Delta \mathcal{Z}'$ the upper bound may decrease and the lower bound may increase. Moreover, eventually we have $\Delta \mathcal{Z}' = \Delta \mathcal{Z}_k^{r,r'}$ and  both the upper and the lower bounds converge to the true value $\mathtt{Cl}_{a}(\Delta \mathcal{Z}_k^{r,r'})$, 
i.e. the gap between the bounds reduces to a point value.
}

\begin{lemma}\label{lem:ub-sverifyacpart}
	Consider two actions $a, a'\in \mathcal{A}_{k+}$ and $a\neq a'$. If $\mathtt{lb}_{a}({\Delta \mathcal{Z}'}) > \mathtt{ub}_{a'}({\Delta \mathcal{Z}'})$ then action $a$ can be given higher preference over $a'$ and hence action $a'$ can be pruned out from the list of candidate actions.
\end{lemma}

In Section~\ref{subsec:sverifyac}, we introduced a threshold $1-\epsilon$ on the value of $\mathtt{Cl}(.)$ . Analogously, for \textsc{R-VerifyAC-simp}, we extend the threshold $1-\epsilon$ for the lower bound $\mathtt{lb}$. 
Consider, for robot $r$, action $\bar{a}$ is selected in step 1.
\textsc{R-VerifyAC-simp} will trigger a \textsc{comm} from $r$ to $r'$ when one of the following holds:
\begin{itemize}
	\item In step 2, $\mathtt{ub}(\leftidx{^2}{\mathtt{Cl}}_{\bar{a}}) < 1-\epsilon$ and $\bar{a}$ is not the rank-1 action, Or
	\item In step 3, $\mathtt{ub}(\leftidx{^3}{\mathtt{Cl}}_{\bar{a}}) < 1-\epsilon$ and $\bar{a}$ is not the rank-1 action.
\end{itemize}
Otherwise, similar to \textsc{R-VerifyAC}, a probabilistic guarantee on \mrac in favor of action $\bar{a}$ can be provided : 
\begin{theorem}\label{thm:ub-prob-guar}
	Consider the action selected in step 1 be $\bar{a}$. 
	If $\bar{a}$ is not the rank-1 action and $\mathtt{lb}_{\bar{a}}(.)\ge 1-\epsilon$ in steps 2 and 3, then \textsc{R-VerifyAC-simp} provides a probabilistic guarantee on \mrac in favor of action $\bar{a}$.
	The probability of \mrac in favor of $\bar{a}$ is in the range $[\mathtt{lb}_{\bar{a}}({\Delta \mathcal{Z}'}), \mathtt{ub}_{\bar{a}}({\Delta \mathcal{Z}'})]$.
\end{theorem}

\subsubsection*{Adaptive technique to select an action from the bounds:}
We design an \rev{\emph{iterative} technique} that enables \textsc{R-VerifyAC-simp} to adapt to non-overlapping bounds, in order to find the rank-1 action with the highest $\mathtt{Cl}(.)$ value. 
Consider, in step 2 or 3, the original set of observations is $\Delta \mathcal{Z}$. Initially, the reduced set of observations is $\Delta \mathcal{Z}' \subseteq \Delta \mathcal{Z}$.
As illustrated in Fig.~\ref{fig:illu-reduced-obs}(a), at some iteration $i$, the candidate actions have overlapping bounds. It is not possible to find the most preferred action with overlapping bounds.
\textsc{Adaptive-Bounds} iteratively adds more observations from \mbox{$\Delta \mathcal{Z}\setminus \Delta \mathcal{Z}'$} into $\Delta \mathcal{Z}'$ that results in shrinking of the bounds; and eventually we find the non-overlapping bounds on the most preferred action $a^*$ (Figure~\ref{fig:illu-reduced-obs}(b)). 

\rev{
\begin{theorem}\label{thm:DifferentSubsets}
Suppose robots $r$ and $r'$ choose subsets $\mathcal{Z}_1' \subset \Delta \mathcal{Z}_k^{r,r'}$ and $\mathcal{Z}_2' \subset \Delta \mathcal{Z}_k^{r,r'}$ respectively, and $\mathcal{Z}_1' \neq \mathcal{Z}_2'$. Even though the different subsets correspond to different lower and upper bounds, algorithm \textsc{Adaptive-Bounds} augments the subsets to converge to the same action $\bar{a}$.
\end{theorem}}

	\rev{
\begin{proof}	
	Given $\Delta \mathcal{Z}_k^{r,r'}$, consider that the robots $r$ and $r'$ select two subsets of observations $\mathcal{Z}_1' \subset \Delta \mathcal{Z}_k^{r,r'}$ and \mbox{$\mathcal{Z}_2' \subset \Delta \mathcal{Z}_k^{r,r'}$} respectively, and $\mathcal{Z}_1'\neq \mathcal{Z}_2'$. Suppose $\bar{a}\in \mathcal{A}_{k+}$ is the most preferred joint action. We have to show that for both the robots \textsc{Adaptive-Bounds} converge to the same joint action $\bar{a}$, i.e. consistent action selection for both the robots, even though the chosen subsets of observations are different. 
	
	When robot $r$ runs the algorithm \textsc{Adaptive-Bounds}(.) for $\mathcal{Z}_1'$, and iteratively augments $\mathcal{Z}_1'$ with more observations, the $\mathtt{lb}_{\bar{a}}(\mathcal{Z}_1')$ goes up or remains unchanged, but does not go down. Similarly, $\mathtt{ub}_{\bar{a}}(\mathcal{Z}_1')$ goes down or does not change, but does not go up. Importantly, the overlap between the bounds $[\mathtt{lb}_{\bar{a}}(\mathcal{Z}_1'), \mathtt{ub}_{\bar{a}}(\mathcal{Z}_1')]$ of action $\bar{a}$ and the bounds $[\mathtt{lb}_{a}(\mathcal{Z}_1'), \mathtt{ub}_{a}(\mathcal{Z}_1')]$ of other actions $a\in \mathcal{A}_{k+}\setminus \{\bar{a}\}$ diminishes monotonically as more observations are added to $\mathcal{Z}_1'$. 
		 Moreover, as $\bar{a}$ has the highest $\mathtt{Cl}$ value, after a finite number of iterations, its lower bound $\mathtt{lb}_{\bar{a}}(\mathcal{Z}_1')$ will become higher than the upper bound $\mathtt{ub}_{a}(\mathcal{Z}_1')$ of any other action $a \in \mathcal{A}\setminus \{\bar{a}\}$. At this point, robot $r$ selects the joint action $\bar{a}$ as the most preferred action (this case is similar to \textsc{R-VerifyAC}). In the worst case, $r$ augments  $\mathcal{Z}_1'$ up to the entire set of observations $\Delta \mathcal{Z}_k^{r,r'}$. 


	Similarly, when robot $r'$ runs \textsc{Adaptive-Bounds}(.) for $\mathcal{Z}_2'$, the overlapping of the bounds for $\bar{a}$ with the bounds of other actions diminishes monotonically; eventually, after some iteration, the bounds for $\bar{a}$ will no longer overlap with the bounds of other actions, and at this point, robot $r'$ will also select $\bar{a}$ to be the most preferred action. 
	In the worst case, this will happen when $\mathcal{Z}_2'$ is augmented to $\Delta\mathcal{Z}_k^{r,r'}$.
	
	Hence, both the robots iteratively converge to the same action $\bar{a}$ even though they start with different subsets of observations. \qed
\end{proof}
}

\begin{theorem}\label{thm:ub-det-guar}
	Consider the action selected in step 1 is $\bar{a}$. 
	The reduced set of observations are $\Delta \mathcal{Z}'$ and $\Delta \mathcal{Z}''$ in step 2 and 3, respectively.
	If for any other action $a\in \mathcal{A}_{k+}\setminus \{\bar{a}\}$ we have
	\begin{enumerate}
	 \item $\mathtt{lb}_{\bar{a}}({\Delta \mathcal{Z}'})> \mathtt{ub}_{a}({\Delta \mathcal{Z}'})$ in step 2, 
	 \item $\mathtt{lb}_{\bar{a}}({\Delta \mathcal{Z}''})> \mathtt{ub}_a({\Delta \mathcal{Z}''})$ in step 3, and
	 \item $\mathtt{ub}_{a}({\Delta \mathcal{Z}'}) < 1-\epsilon$ and $\mathtt{ub}_a({\Delta \mathcal{Z}''}) < 1-\epsilon$
	\end{enumerate}
	 then \textsc{R-VerifyAC-simp} \rev{guarantees the probability of Action Inconsistency is Zero (similar to Corollary~\ref{cor:r-verifyac-no-inconsistency-guarantee})}.

\end{theorem}


\pgfplotsset{compat=1.8}
\usepgfplotslibrary{statistics}
\makeatletter
\pgfplotsset{
	boxplot prepared from table/.code={
		\def\tikz@plot@handler{\pgfplotsplothandlerboxplotprepared}%
		\pgfplotsset{
			/pgfplots/boxplot prepared from table/.cd,
			#1,
		}
	},
	/pgfplots/boxplot prepared from table/.cd,
	table/.code={\pgfplotstablecopy{#1}\to\boxplot@datatable},
	row/.initial=0,
	make style readable from table/.style={
		#1/.code={
			\pgfplotstablegetelem{\pgfkeysvalueof{/pgfplots/boxplot prepared from table/row}}{##1}\of\boxplot@datatable
			\pgfplotsset{boxplot/#1/.expand once={\pgfplotsretval}}
		}
	},
	make style readable from table=lower whisker,
	make style readable from table=upper whisker,
	make style readable from table=lower quartile,
	make style readable from table=upper quartile,
	make style readable from table=median,
	make style readable from table=lower notch,
	make style readable from table=upper notch
}
\makeatother

\pgfplotstableread{
	lw uw name
	8.5 10 $\bar{a}$
	4  6.5 $\bar{a}'$
	5.5  7.5 $\bar{a}''$
}\lefttab

\pgfplotstableread{
	lw uw name
	5 10 $\bar{a}$
	4  7 $\bar{a}'$
	6  8.5 $\bar{a}''$
}\righttab


\begin{figure}[t]
\subfigure[Iteration $i$]{
	\resizebox{0.47\linewidth}{!}{
	\begin{tikzpicture}
		\pgfplotsset{every tick label/.append style={font=\LARGE}}
		\draw [thick,dashed] (0,1.4) -- (6.9,1.4);
		\begin{axis}
			[boxplot/draw direction=y,
			ylabel=$\mathtt{Cl}(.)$,
			x tick label style={rotate=0, anchor=east, align=center, yshift=-0.5cm, font=\LARGE},
			y tick label style={color=white, font=\LARGE},
			xtick={0.5, 1, 1.5, 2, 2.5, 3},
			xticklabels={,$\bar{a}$,,$\bar{a}'$,,$\bar{a}''$},
			]
			\pgfplotstablegetrowsof{\righttab}
			\pgfmathtruncatemacro\TotalRows{\pgfplotsretval-1}
			\pgfplotsinvokeforeach{0,...,\TotalRows}
			{
				\addplot+[
				boxplot prepared from table={
					table=\righttab,
					row=#1,
					lower whisker=lw,
					upper whisker=uw,
				},
				area legend
				]
				coordinates {};
				
			}
		\end{axis}
	\end{tikzpicture}
}}
\subfigure[Iteration $i+j$]{
\resizebox{0.47\linewidth}{!}{
\begin{tikzpicture}
	\pgfplotsset{every tick label/.append style={font=\LARGE}}
	\draw [thick,dashed] (0,1.4) -- (6.9,1.4);
	\begin{axis}
		[boxplot/draw direction=y,
		ylabel=$\mathtt{Cl}(.)$,
		x tick label style={rotate=0, anchor=east, align=center, yshift=-0.5cm, font=\LARGE},
		y tick label style={color=white, font=\LARGE},
		xtick={0.5, 1, 1.5, 2, 2.5, 3},
		xticklabels={,$\bar{a}$,,$\bar{a}'$,,$\bar{a}''$},
		]
		\pgfplotstablegetrowsof{\lefttab}
		\pgfmathtruncatemacro\TotalRows{\pgfplotsretval-1}
		\pgfplotsinvokeforeach{0,...,\TotalRows}
		{
			\addplot+[
			boxplot prepared from table={
				table=\lefttab,
				row=#1,
				lower whisker=lw,
				upper whisker=uw,
			},
			area legend
			]
			coordinates {};
			
		}
	\end{axis}
	
\end{tikzpicture}
}}
\caption{Illustration of \textsc{Adaptive-Bounds} in steps 2 and 3 of \textsc{R-VerifyAC-simp}. From the perspective of robot $r$, the figures are shown for step 3. 
Dashed line represents the threshold ($1-\epsilon$). Upper $(\mathtt{ub})$ and lower bounds $(\mathtt{lb})$ of $\mathtt{Cl(.)}$ are shown. 
\textbf{(a)} After iteration $i$, the bounds of different actions are overlapping, thus making it difficult to choose a particular action.
\textbf{(b)} Up to iteration $i+j$, \textsc{Adaptive-Bounds} iteratively adds new observations in the reduced set of observations. 
\textbf{Case-1:}  Action $\bar{a}$ is chosen in step 1. Action $\bar{a}$ has higher $\mathtt{lb}(.)$ than the $\mathtt{ub}(.)$ of any other action. So, $\bar{a}$ can be chosen in Step 3 using deterministic bounds. 
\textbf{Case-2:} Action $\bar{a}'$ is chosen in step 1. \textsc{comm} is self-triggered by robot $r$ as $\mathtt{lb}(.)$ of $\bar{a}'$ is below the threshold.
\textbf{Case-3:} Action $\bar{a}''$ is chosen in step 1. $\mathtt{lb}(.)$ of $\bar{a}''$ is higher than the threshold, however $\bar{a}''$ is not the rank-1 action with highest $\mathtt{Cl}(.)$ value. Probabilistic guarantee can be provided in favor of $\bar{a}''$.
}
\label{fig:illu-reduced-obs}
\end{figure}

\SetKwComment{Comment}{/* }{ */}

\begin{algorithm}
\DontPrintSemicolon
\begin{footnotesize}
\KwIn{\\
	\ \ $\Delta \mathcal{Z}$: The original set of observations in Step 2 (or 3).\\
	\ \ $\Delta \mathcal{Z}'$: Initially, the reduced set of observations. $\Delta \mathcal{Z}' \subseteq \Delta \mathcal{Z}$. \\
	\ \ $m$: No. of observations to be added to $\Delta \mathcal{Z}'$ upon receiving overlapping bounds.
	 
}
\smallskip
\KwOut{\\
	\ \ $a^*$: The action emerging out to be the rank-1 action with the\\
	\ \ \ \ \ \ \ highest $\mathtt{Cl}(.)$ value, with non-overlapping bounds.\\
}%
\smallskip
\smallskip
\SetKwFunction{KwFn}{{\bf Algorithm \textsc{\textbf{ Adaptive-Bounds }}}}
\KwFn{$\Delta \mathcal{Z}$, $\Delta \mathcal{Z}'$, $m$} \\
\Begin{
	$a^*\gets \mathtt{null}$\;
	\While{$a^*$ is $\mathtt{null}$}
	{
		$a^* \gets \exists \bar{a}:\  \mathtt{lb}_{\bar{a}}(\Delta \mathcal{Z}') > \mathtt{ub}_{a}(\Delta \mathcal{Z}')\ \ \forall a\in \mathcal{A}_{k+}\setminus \{\bar{a}\}$\;
		\If{$a^* \neq \mathtt{null}$}
		{
			$\mathtt{return}\ \ a^* $ \;
		}
		$\Delta \mathcal{Z}' \gets \Delta \mathcal{Z}' \cup \{m \text{ observations from } \Delta \mathcal{Z}\setminus \Delta \mathcal{Z}'\}$\;
		
	}
}
\end{footnotesize}
\caption{\textsc{Adaptive-Bounds:} \rev{An incremental approach to adaptively select a joint action from the bounds in steps 2 and 3 of \textsc{R-VerifyAC-simp}.}} \label{algo:adaptive-sverifyac-part}
\end{algorithm}

\begin{lemma}\label{lem:perf-sverifyacpart}
	\mbox{\textsc{R-VerifyAC-simp}} and \mbox{\textsc{R-VerifyAC}} select the same joint action, for a given problem instance.
\end{lemma}
\begin{proof}
	\textsc{R-VerifyAC} aims to find the \emph{exact} $\mathtt{Cl}(.)$ value of the rank-1 action with the highest $\mathtt{Cl}(.)$ value in steps 2 and 3.
	On the other hand, \mbox{\textsc{R-VerifyAC-simp}} aims to find the non-overlapping bounds of the rank-1 action with the highest $\mathtt{Cl}(.)$ value. The exact value and the non-overlapping bounds encompassing the exact value correspond to the same joint action. Using \textsc{Adaptive-Bounds}, as soon as \mbox{\textsc{R-VerifyAC-simp}} deterministically finds the most preferred action, it stops.
	The outcome of both  algorithms are same -- the same most preferred action -- however the approaches are different. \qed
\end{proof}

\begin{theorem}
	Given a problem instance and a threshold $1- \epsilon$, \mbox{\textsc{R-VerifyAC-simp}} incurs the same number of \textsc{comm}s, as in \mbox{\textsc{R-VerifyAC}}, \revnew{generally with a lower computation time and in the worst case at most the same computation time compared to \textsc{R-VerifyAC}.}
\end{theorem}
\begin{proof}
	According to Lemma~\ref{lem:perf-sverifyacpart}, both the algorithms find the same joint action as the most preferred action in steps 2 and 3. This implies both the algorithms find the same number time instants satisfying the conditions for triggering \textsc{comm}s, and hence the same number of \textsc{comm}s. However, using \textsc{Adaptive-Bounds}, algorithm \textsc{R-VerifyAC-simp} finds the deterministic bounds instead of the exact value of $\mathtt{Cl}$(.).
	\textsc{R-VerifyAC-simp} does so by exploring a reduced set of observations $\Delta \mathcal{Z}' \subseteq \Delta \mathcal{Z}$, whereas \textsc{R-VerifyAC} explores the entire observation space $\Delta \mathcal{Z}$ exhaustively.
	Hence, the computation time of \textsc{R-VerifyAC-simp} is lower compared to \textsc{R-VerifyAC}.
	\revnew{In the worst case, the subset $\Delta \mathcal{Z}'$ equals the observation space $\Delta \mathcal{Z}$, making the computation time of \textsc{R-VerifyAC-simp} the same as that of \textsc{R-VerifyAC}.} \qed
\end{proof}


%

\revnew{
	\section{Planning Implications of MR-AC}\label{sec:PlanningImplications}

	\begin{figure*}[t]
		\centering
		\includegraphics[ width=0.7\textwidth]{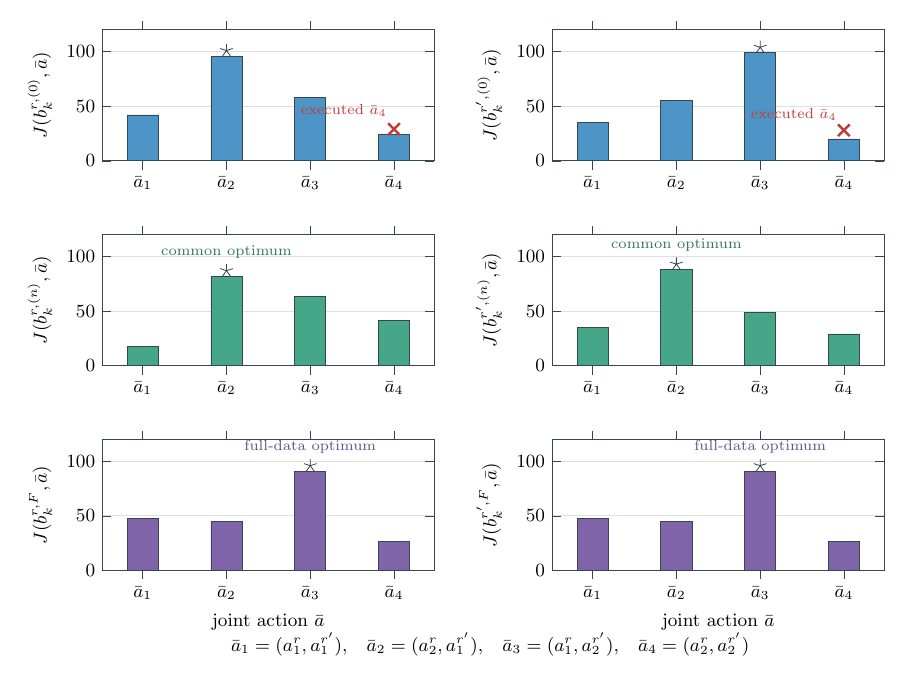}
		\caption{\label{fig:action-consistency-slices}\scriptsize 
			\revnew{Belief-dependent objective profiles for robots $r$ (left) and $r'$ (right). Initially, incompatible local maximizers result in the executed joint action $\bar{a}_4$ which is suboptimal under both beliefs $b^{r,(0)_k}$ and $b^{r',(0)_k}$ (first row); after $n$ data sharing iterations, robots $r$ and $r'$ have access to updated beliefs $b_k^{r,(n)}$ and $b_k^{r',(n)}$. The corresponding distinct objective profiles may have a common maximizer $\bar{a}_2$ (second row), a situation that our approach identifies with deterministic or probabilistic guarantees. Full data sharing yields identical beliefs ($b^{r,F}_k$ and $b^{r',F}_k$) and objective profiles (third row). See Section \ref{sec:PlanningImplications} for details.}\vspace{-10pt}
		}
	\end{figure*}
	
	After $n$ data-sharing iterations, let $b^{r,(n)}_k$ and $b^{r',(n)}_k$ denote the updated local beliefs of robots $r$ and $r'$, respectively, conditioned on the corresponding updated histories available to each robot. These beliefs, and hence their objective profiles $J\!\left(b_k^{r,(n)},.\right)$ and $J\!\left(b_k^{r',(n)},.\right)$, are generally not identical.
	Nevertheless, because each robot selects a joint action that maximizes its own local objective, MR-AC implies the existence of a \textit{common maximizer}, i.e.~
	\begin{align}
		\bar{a}^{\star}
		\in
		\arg\max_{\bar{a}\in\mathcal{A}}J\!\left(b_k^{r,(n)},\bar{a}\right)
		\;\cap\;
		\arg\max_{\bar{a}\in\mathcal{A}}J\!\left(b_k^{r',(n)},\bar{a}\right).
	\end{align}
	Thus, MR-AC guarantees not only that the \textbf{robots select the same joint action}, but also that the selected action is \textbf{optimal with respect to each robot's current local belief}. Importantly, in such a setting, generally,  $J\!\left(b_k^{r,(n)},\bar{a}^{\star}\right)\neq J\!\left(b_k^{r',(n)},\bar{a}^{\star}\right)$.

	Such a situation is illustrated in  Fig.~\ref{fig:action-consistency-slices}.  
	Initially,  robots $r$ and $r'$ have local beliefs $b_k^{r,(0)}$ and $b_k^{r',(0)}$; robot $r$ identifies $\bar{a}_2=(a^r_2,a^{r'}_1) \in\arg\max_{\bar{a}\in\mathcal{A}} J(b_k^{r,(0)},\bar{a})$, whereas robot $r'$ identifies $\bar{a}_3=(a^r_1, a^{r'}_2) \in\arg\max_{\bar{a}\in\mathcal{A}} J(b_k^{r',(0)},\bar{a})$ as the best joint actions (first row). Independent execution of their respective components produces $\bar{a}_4=(a^r_2, a^{r'}_2)$ which is \textit{suboptimal under both local beliefs}. After $n$ data sharing iterations, the objective profiles remain different but with a common maximizer, i.e.~$\bar{a}_2 \in\arg\max_{\bar{a}\in\mathcal{A}} J(b_k^{r,(n)},\bar{a})$ and $\bar{a}_2 \in\arg\max_{\bar{a}\in\mathcal{A}} J(b_k^{r',(n)},\bar{a})$ (second row). 
	Full data sharing, shown in the last row of Fig.~\ref{fig:action-consistency-slices}, makes the beliefs, and hence the objective profiles, identical, but is not required for the common-local-optimality guarantee.

	\textsc{EnforceAC} enables to identify MR-AC with a deterministic guarantee, when invoked with \textsc{VerifyAC}. The relaxed variants developed in Section \ref{sec:simplification}, provide probabilistic MR-AC guarantees (see Fig.~\ref{fig:r-verifyac-flow-chart}). 
	
	Finally, we note that these guarantees are fundamentally stronger than agreement obtained by always selecting a predetermined joint action: such a rule guarantees consistency only, whereas it provides no optimality guarantee under either local belief and can be arbitrarily suboptimal.

}

\newcommand{\cH}[1]{\leftidx{^c} H_{{#1}}^{r,r'}}

\newcommand{\dZk}{\Delta \mathcal{Z}_k}
\newcommand{\dZkr}{\Delta \mathcal{Z}_k^{r',r}}
\newcommand{\dZkrprime}{\Delta \mathcal{Z}_k^{r,r'}}

\newcommand{\probest}[2]{\hat{\mathbb{P}}_{#1} ({#2})}
\newcommand{\probestnamed}[3]{\hat{\mathbb{P}}_{#1}^{#2} ({#3})}

\newcommand{\expec}[1]{\mathbb{E}_{#1}}
\newcommand{\expectation}[2]{\mathbb{E}_{ #1 \sim \prob{ \cdot \mid #2 } }}

\newcommand{\expecest}[3]{\mathbb{E}_{#1 \sim \probest{#2}{#3}}}
\newcommand{\expectationest}[3]{\mathbb{E}_{ #1 \sim \probest{#2}{ \cdot \mid #3 } }}

\newcommand{\verifyac}{\textsc{VerifyAC}\xspace}

\section{Continuous and High-Dimensional Cases}\label{sec:continuous-and-high-dim-cases}

In the previous section we presented simplifications of the \verifyac algorithm that can provide probabilistic or deterministic action-consistency guarantees.
Yet, all these variations of the algorithm require calculating the cumulative likelihood of the joint actions $\mathtt{Cl}_a(Z)$ 
via the likelihoods $\prob{ z \mid \cH{k} }$ of the realizations of unshared observations~\eqref{eq:cumulative-likelihood}.
According to equation \eqref{eq:meas-likelihood} calculating the likelihood \rev{of some realization $z \in \dZk$} requires marginalizing over the corresponding sequence of states.
This operation can become intractable when the state space is huge (i.e. multi-dimensional), and even essentially impossible if the state space is continuous.
\rev{Moreover, when the observation space is large, calculating the cumulative likelihood also becomes intractable.
Even when using the simplified variation \textsc{R-VerifyAC-simp}, in order for the adaptive bounds of joint actions to not overlap the algorithm might require a massive amount of realizations until converging to non-overlapping bounds.}

To overcome this issue, we propose using \textit{estimators} to calculate an estimation of the likelihood of the observations \rev{and the joint actions cumulative likelihoods}.
In Section \ref{subsec:observation-sequence-likelihood-estimation}, we present a method for estimating the likelihood of a specific observation sequence realization without marginalizing over an entire continuous or high-dimensional state space using an estimator of the agent's Belief.
Moreover, in Section \ref{subsec:cumulative-likelihood-estimation} we present a method for estimating the cumulative likelihood of the actions according to $\mathtt{cobs}_a(\dZk)$ when the observation space is continuous or high-dimensional using an estimator of the observation models.

\subsection{Observation Sequence Likelihood Estimation}\label{subsec:observation-sequence-likelihood-estimation}

The likelihood of an observation sequence requires marginalizing over the corresponding sequence of states at the observations time steps~\eqref{eq:meas-likelihood}.
For the following formulations, we focus only on the observation \rev{space} $\dZkr$ (from Step 3), but the same \rev{technique applies}  to $\dZkrprime$ (from Step 2) as well.

\rev{
Using the general definition of the unshared observations space \eqref{eq:dZrtagr},
we rewrite~\eqref{eq:meas-likelihood} explicitly to express the likelihood of some realization of the unshared observations $\tilde{z}^{r} = \{ z^{r}_{{i}_{1}}, z^{r}_{{i}_{2}}, ..., z^{r}_{{i}_{C}} \} \in \dZkr$ by marginalizing over the state space at the time steps of these observations $\tilde{x} \triangleq \{ x_{{i}_{1}}, x_{{i}_{2}}, ..., x_{{i}_{C}} \}$:
\begin{equation}\label{eq:observation-sequence-likelihood}
	\prob{ \tilde{z}^{r} \mid \cH{k} } = \expectation{ \tilde{x} }{\cH{k}} \Big[ \prod_{j=1}^{C} \prob{ z^{r}_{{i}_{j}} \mid x_{{i}_{j}} } \Big].
\end{equation}
}
\rev{When the state space is continuous, this expression is impossible to calculate exactly in practice.
Instead of marginalizing using the theoretical belief $\prob{ \tilde{x} \mid \cH{k}}$, we can estimate the observations likelihood using \emph{samples} of $\tilde{x}$.
}

\rev{
We can estimate the belief over $\tilde{x}$ using $N_X$ 
weighted samples $\left\{ ( w^{(s)}, \tilde{x}^{(s)} ) \right\}_{s=1}^{N_X}$, where $\tilde{x}^{(s)} = \left\{ x_{{i}_{1}}^{(s)}, x_{{i}_{2}}^{(s)}, \cdots, x_{{i}_{C}}^{(s)} \right\}$ and $w^{(s)} \in \mathbb{R}$.
We note that if these samples are drawn directly from $\prob{ \tilde{x} \mid \cH{k}}$, e.g.~if the latter is a parametric belief, then the weights become uniform, and can be replaced with the constant ${N}^{-1}_X$. 
The belief estimator using the samples is given by:
\begin{equation}\label{eq:belief-i1-ic-estimator-continuous}
	\probest{N_X}{ \tilde{x} \mid \cH{k} } = \sum_{s=1}^{N_X} \Big[ w^{(s)} \cdot \! \prod_{j=1}^{C} \delta \left( x_{{i}_{j}} - x_{{i}_{j}}^{ (s) } \right) \Big]
	,
\end{equation}
where $\delta (x - x^{(s)})$ is the Dirac-delta function over a continuous variable $x$ and a constant value $x^{(s)}$.
}

Then, \rev{using the belief estimator~\eqref{eq:belief-i1-ic-estimator-continuous},} the likelihood estimator of the observation sequence $\tilde{z}^{r}$ \eqref{eq:observation-sequence-likelihood} is given by:
\rev{
\begin{equation}\label{eq:observation-sequence-likelihood-estimator}
	\probest{N_X}{ \tilde{z}^{r} \mid \cH{k} } = \sum_{s=1}^{N_X} \Big[ w^{(s)} \cdot \! \prod_{j=1}^{C} \prob{ z^{r}_{{i}_{j}} \mid x_{{i}_{j}}^{(s)} } \Big]
	.
\end{equation}
}
%
\rev{Note that these results are true for continuous and discrete state spaces, where for the latter we only replace the Dirac-delta function in~\eqref{eq:belief-i1-ic-estimator-continuous} with discrete indicators.} 

\subsection{Cumulative Likelihood Estimation}\label{subsec:cumulative-likelihood-estimation}

In Section~\ref{subsec:sverifyac} we described the algorithm \textsc{R-VerifyAC}, which in Steps 2 and 3 reasons over the entire observation spaces $\dZkr$ and $\dZkrprime$ to calculate the cumulative likelihoods of the consistent observations in favor of the joint action selected in Step 1.
Then in Section~\ref{subsec:rverifyac-simp} we relaxed this exhaustive reasoning in the algorithm \textsc{R-VerifyAC-simp}, by selecting only a subset of the entire observation space $\Delta \mathcal{Z}'_k \subseteq \dZk$ and reasoning over it to calculate bounds on the cumulative likelihood of the consistent observations.
However, if the observation space is continuous it is not possible to \emph{select} a subset of the observation space, and sampling must be done in order to marginalize over the distribution.

If the observation space is continuous, the formulation of \eqref{eq:cumulative-likelihood} changes:
\rev{
\begin{equation}\label{eq:cumulative-likelihood-continuous-observations}
	\begin{aligned}
		& \mathtt{Cl}_{a} ( \dZkr )
		= \expec{ \tilde{z}^{r} \mid \cH{k} } \left[ \mathbb{I}_{\{ \tilde{z}^{r} \in \mathtt{cobs}_{a}(\dZkr) \}} (\tilde{z}^{r}) \right] \\
		& = \int_{ \tilde{z}^{r} \in \dZkr } \prob{ z \mid \cH{k}} \mathbb{I}_{\{ \tilde{z}^{r} \in \mathtt{cobs}_{a}(\dZkr) \}} (\tilde{z}^{r}) ~ d \tilde{z}^{r}
		,
	\end{aligned}
\end{equation}
}
and it is effectively impossible to go over the entire observation space, or a selected subset of it, to calculate the cumulative likelihood.
In order to estimate the cumulative likelihood of the consistent observations we need to acquire observations from the likelihood distribution.

As shown in the previous section, the observation likelihood can be expressed using the belief of \rev{the state at the missing observations time steps $\tilde{x}$}. 
Using the samples described in the previous section to estimate the likelihood of a specific observation sequence \eqref{eq:observation-sequence-likelihood}, we can additionally sample the observation model conditioned on each sample of the state space.
Given the states samples, for each state sample $x_{l}^{(s)}$ we can then sample $N_Z$ samples from the observation model $\prob{z^{r}_{l} \mid x_{l}^{(s)}}$, denoting the samples as $ \{ {z^{r}_{l}}^{(m \mid s)} \}_{m=1}^{N_Z} $.
Using the observations samples, we define a sampled estimator of the observations likelihood:
\rev{
\begin{equation}\label{eq:observations-likelihood-continuous-sampled-estimator}
	\probest{N}{\tilde{z}^{r} \mid \cH{k}} \!
	= \! \sum_{s=1}^{N_X} \frac{w^{(s)}}{N_Z} \! \cdot \! \sum_{m=1}^{N_Z} \! \Big[ \prod_{j=1}^{C} \delta \left( z^{r}_{{i}_{j}} \!\! - \! {z^{r}_{{i}_{j}}}^{(m|s)} \right) \Big]
	,
\end{equation}
where $N = N_Z \cdot N_X$, to shorten the formulations.
}

Then, the cumulative likelihood estimator of some action $a$ according to the missing observations space $\dZkr$ is given by:
\rev{
\begin{equation}\label{eq:cumulative-likelihood-estimator}
	\begin{aligned}
		& \hat{\mathtt{Cl}}_{a} (\dZkr)_{N} \\
		& = \sum_{s=1}^{N_X} \frac{w^{(s)}}{N_Z} \cdot \sum_{m=1}^{N_Z} \mathbb{I}_{ \{ {\tilde{z}}^{ r^{(m|s)} } \in \mathtt{cobs}_{a}(\dZkr) \} } ( {\tilde{z}}^{r^{(m|s)}} )
		,
	\end{aligned}
\end{equation}
where $\mathbb{I}_{\{ z \in \mathtt{cobs}_{a}(\dZkr) \}}(z)$ is the indicator function over the event ``$z$ is in favor of action $a$'', considering different values $z \in \dZkr$.
}

\rev{In other words}, the cumulative likelihood is estimated as the respective ratio of sampled observations that are in favor of the action $a$ out of all the samples.

Given the estimator in \eqref{eq:cumulative-likelihood-estimator} we can calculate probabilistic bounds over the theoretic cumulative likelihood of the actions in Steps 2 and 3 of the \textsc{R-VerifyAC-simp}. 
\rev{Each $\mathtt{Cl}_{a} (\dZkr)$ represents a likelihood between 0 to 1, i.e. a bounded function for each action $a$ over the observations space $\dZkr$.
Therefore, if the samples of the estimator $\hat{\mathtt{Cl}}_{a} (\dZkr)_{N}$ are i.i.d, we can use Hoeffding's inequality to calculate a probabilistic bound over the difference between the theoretical cumulative likelihood and the estimated cumulative likelihood of each action $a$,
\begin{equation}\label{eq:cl-hoeffding-inequality}
	\begin{aligned}
		\prob{ \vert \mathtt{Cl}_{a} (\dZkr) - \hat{\mathtt{Cl}}_{a} (\dZkr)_{N} \vert \leq \delta }  \geq 1- \alpha
		,
	\end{aligned}
\end{equation}
where $\alpha = 2 \exp (-2 N \delta^2)$.
}
\rev{
Let $\alpha_0$ be a user pre-defined confidence level.
The corresponding $\delta_0$ is then calculated as 
$ \delta_0 (\alpha_0, N) = \sqrt{\frac{\log (2 / \alpha_0) }{2 N }}$.
Then, using the Hoeffding's inequality~\eqref{eq:cl-hoeffding-inequality}, we calculate probabilistic bounds on the theoretical cumulative likelihood of each action,
\begin{equation}\label{eq:cl-hoeffding-bounds}
	\begin{aligned}
		& \hat{\mathtt{lb}}_{a} (\dZkr)_{N,\alpha_0}
		= \hat{\mathtt{Cl}}_{a} (\dZkr)_{N} - \delta_0 (\alpha_0, N) , \\
		& \hat{\mathtt{ub}}_{a} (\dZkr)_{N,\alpha_0}
		= \hat{\mathtt{Cl}}_{a} (\dZkr)_{N} + \delta_0 (\alpha_0, N) , \\
	\end{aligned}
\end{equation}
which hold with probability of $1-\alpha_0$.
We emphasize that these bounds only depend on the estimated cumulative likelihood $\hat{\mathtt{Cl}}_{a} (\dZkr)_{N}$ and $\delta_0 (\alpha_0, N)$.
}

\rev{
Similarly to Section \ref{subsec:rverifyac-simp}, we require the bounds of the cumulative likelihoods~\eqref{eq:cl-hoeffding-bounds} to not overlap with each other, in order to be able to determine probabilistically which action has the highest cumulative likelihood. 
When the bounds overlap, as can be seen in figure~\ref{fig:illu-cl-overlap-bounds}, we can tighten the bounds by improving $\delta_0$ by adding more samples, i.e.~increase $N$. 
Adding more samples
can also change  the value of the estimated cumulative likelihood $\hat{\mathtt{Cl}}_{a} (\dZkr)_{N}$, which also affects the bounds.
Yet, adding more samples tightens the probabilistic bounds, and after adding a sufficient amount of samples, there will be no overlap anymore between an action with the highest lower bound and the upper bounds of all other actions,
as can be seen in figure~\ref{fig:illu-cl-no-overlap-bounds}.
In such case, we can probabilistically guarantee which action  has the highest cumulative likelihood.
}
\pgfplotsset{compat=1.8}
\usepgfplotslibrary{statistics}
\makeatletter
\pgfplotsset{
	boxplot prepared from table/.code={
		\def\tikz@plot@handler{\pgfplotsplothandlerboxplotprepared}%
		\pgfplotsset{
			/pgfplots/boxplot prepared from table/.cd,
			#1,
		}
	},
	/pgfplots/boxplot prepared from table/.cd,
	table/.code={\pgfplotstablecopy{#1}\to\boxplot@datatable},
	row/.initial=0,
	make style readable from table/.style={
		#1/.code={
			\pgfplotstablegetelem{\pgfkeysvalueof{/pgfplots/boxplot prepared from table/row}}{##1}\of\boxplot@datatable
			\pgfplotsset{boxplot/#1/.expand once={\pgfplotsretval}}
		}
	},
	make style readable from table=lower whisker,
	make style readable from table=upper whisker,
	make style readable from table=lower quartile,
	make style readable from table=upper quartile,
	make style readable from table=median,
	make style readable from table=lower notch,
	make style readable from table=upper notch
}
\makeatother

\pgfplotstableread{
	lw uw  med name
	6  14  10  $\bar{a}$
	1  9   5   $\bar{a}'$
	3  11  7   $\bar{a}''$
}\lefttab

\pgfplotstableread{
	lw uw  med name
	9  13  11  $\bar{a}$
	0  4   2   $\bar{a}'$
	5  9   7   $\bar{a}''$
}\righttab


\begin{figure}[t]
\subfigure[Overlapping bounds of $\mathtt{Cl}_{a}(.)$.]{
	\label{fig:illu-cl-overlap-bounds}
	\centering
	\resizebox{0.47\linewidth}{!}{
	\begin{tikzpicture}
		\pgfplotsset{every tick label/.append style={font=\LARGE}}
		\begin{axis}
			[boxplot/draw direction=y,
			y label style={font=\LARGE},
			y tick label style={color=white},
			ylabel=$\mathtt{Cl}(.)$,
			x tick label style={rotate=0, anchor=east, align=center, yshift=-0.5cm, font=\LARGE},
			xtick={ 1, 2, 3 },
			xticklabels={$\bar{a}$,$\bar{a}'$,$\bar{a}''$},
			]
			\pgfplotstablegetrowsof{\lefttab}
			\pgfmathtruncatemacro\TotalRows{\pgfplotsretval-1}
			\pgfplotsinvokeforeach{0,...,\TotalRows}
			{
				\addplot+[
				boxplot prepared from table={
					table=\lefttab,
					row=#1,
					lower whisker=lw,
					upper whisker=uw,
				},
				area legend
				]
				coordinates {};
				
			}
			
			\addplot[blue, mark=square*, mark size=4pt] coordinates {
				(1, 10)
			};
			\addplot[red, mark=square*, mark size=4pt] coordinates {
				(2, 5)
			};
			\addplot[orange, mark=square*, mark size=4pt] coordinates {
				(3, 7)
			};
			\node at (axis cs:1, 10) [right, xshift=2pt, font=\Large] {$\hat{\mathtt{Cl}}_{\bar{a}}$};
			\node at (axis cs:1-0.4, 8.0) [right, xshift=2pt, font=\Large] {$\delta_0$};
			\node at (axis cs:1-0.4, 12.0) [right, xshift=2pt, font=\Large] {$\delta_0$};
			\node at (axis cs:2, 5) [right, xshift=2pt, font=\Large] {$\hat{\mathtt{Cl}}_{\bar{a}'}$};
			\node at (axis cs:2-0.4, 3.0) [right, xshift=2pt, font=\Large] {$\delta_0$};
			\node at (axis cs:2-0.4, 7.0) [right, xshift=2pt, font=\Large] {$\delta_0$};
			\node at (axis cs:3, 7) [right, xshift=2pt, font=\Large] {$\hat{\mathtt{Cl}}_{\bar{a}''}$};
			\node at (axis cs:3-0.4, 5.0) [right, xshift=2pt, font=\Large] {$\delta_0$};
			\node at (axis cs:3-0.4, 9.0) [right, xshift=2pt, font=\Large] {$\delta_0$};
		\end{axis}
	\end{tikzpicture}
}
}
\subfigure[Non overlapping bounds of $\mathtt{Cl}_{a}(.)$, after increasing the number of samples $N$.]{
	\label{fig:illu-cl-no-overlap-bounds}
	\centering
	\resizebox{0.47\linewidth}{!}{
	\begin{tikzpicture}
		\pgfplotsset{every tick label/.append style={font=\LARGE}}
		\begin{axis}
			[boxplot/draw direction=y,
			y label style={font=\LARGE},
			y tick label style={color=white},
			ylabel=$\mathtt{Cl}(.)$,
			x tick label style={rotate=0, anchor=east, align=center, yshift=-0.5cm, font=\LARGE},
			xtick={ 1, 2, 3 },
			xticklabels={$\bar{a}$,$\bar{a}'$,$\bar{a}''$},
			]
			\pgfplotstablegetrowsof{\righttab}
			\pgfmathtruncatemacro\TotalRows{\pgfplotsretval-1}
			\pgfplotsinvokeforeach{0,...,\TotalRows}
			{
				\addplot+[
				boxplot prepared from table={
					table=\righttab,
					row=#1,
					lower whisker=lw,
					upper whisker=uw,
				},
				area legend
				]
				coordinates {};
				
			}
			
			\addplot[blue, mark=square*, mark size=4pt] coordinates {
				(1, 11)
			};
			\addplot[red, mark=square*, mark size=4pt] coordinates {
				(2, 2)
			};
			\addplot[orange, mark=square*, mark size=4pt] coordinates {
				(3, 7)
			};
			\node at (axis cs:1, 10) [right, xshift=2pt, font=\Large] {$\hat{\mathtt{Cl}}_{\bar{a}}$};
			\node at (axis cs:1 -0.4, 10.0) [right, xshift=2pt, font=\Large] {$\delta_0$};
			\node at (axis cs:1 -0.4, 12.0) [right, xshift=2pt, font=\Large] {$\delta_0$};
			\node at (axis cs:2, 2) [right, xshift=2pt, font=\Large] {$\hat{\mathtt{Cl}}_{\bar{a}'}$};
			\node at (axis cs:2 -0.4, 1.0) [right, xshift=2pt, font=\Large] {$\delta_0$};
			\node at (axis cs:2 -0.4, 3.0) [right, xshift=2pt, font=\Large] {$\delta_0$};
			\node at (axis cs:3, 7) [right, xshift=2pt, font=\Large] {$\hat{\mathtt{Cl}}_{\bar{a}''}$};
			\node at (axis cs:3 -0.4, 6.0) [right, xshift=2pt, font=\Large] {$\delta_0$};
			\node at (axis cs:3 -0.4, 8.0) [right, xshift=2pt, font=\Large] {$\delta_0$};
		\end{axis}
	\end{tikzpicture}
}
}
\caption{
	\rev{
	Illustration of the bounds of the cumulative likelihoods of different actions, according to $\delta_0$ which is calculated according to a pre-defined $\alpha_0$ and number of samples in the estimators $N$.
	Figure~\ref{fig:illu-cl-overlap-bounds} shows overlapping bounds between different actions for a pre-defined confidence level $\alpha_0$ and the number of samples $N$.
	To tighten the bounds, more samples are added to the estimators, which changes the value of the estimators a little, but also decreases $\delta_0$.
	Figure~\ref{fig:illu-cl-no-overlap-bounds} shows the resulting bounds after increasing the number of samples, where the bounds are tighter and do not overlap.
}
}
\label{fig:illu-cl-cont-est}
\end{figure}
\section{Implementation and results}\label{sec:results}

\rev{We evaluate our approach in simulation considering a \emph{search and rescue} scenario in a disaster ravaged area, and  on real robots data from experiments of \emph{active multi-robot visual SLAM}.} 
In the \emph{search and rescue} simulations \rev{(Section \ref{subsec:simulation-results})}, we compare our algorithms \textsc{EnforceAC},\textsc{R-EnforceAC} and \textsc{R-EnforceAC-simp} with two baseline algorithms -- Baseline-I and Baseline-II. In Baseline-I robots do two-way \textsc{comm}s at each time point, and that leads to consistent beliefs of the robots at each time point. In Baseline-II robots do not communicate at all, and end up having inconsistent beliefs at every time step. Simulations were carried out on an Intel Core i7-6500U with 2.5 GHz clock. The algorithms are implemented in \texttt{Julia}.

In the \emph{active multi-robot visual SLAM} experiments \rev{(Section \ref{subsec:slam-results})}, we compare our algorithm \textsc{R-EnforceAC}, in continuous observation and state spaces setting, with two baseline algorithms -- Baseline-I and Baseline-II. \rev{Computations} were carried out on an Intel Xeon E5-1620 with 3.50GHz clock. The algorithms for this evaluation are implemented in \texttt{Python}.

\rev{In both the simulation and real-data experiment evaluations, we aim to demonstrate the advantages of our approach in terms of action consistency occurrences and amount of communications that were performed. \revnew{As discussed in Section~\ref{sec:PlanningImplications}, our approach ensures, with a deterministic or  probabilistic certificate, that both robots choose the same joint action, and that this action is optimal with respect to each robot's current local beliefs upon termination, conditioned on data available to that robot, including data communicated by the other robot. We do not claim that this action is optimal with respect to the full-data-sharing belief, nor do we optimize its performance gap relative to that benchmark; extending the framework with such guarantees is currently an ongoing work (e.g. \cite{Rafaeli26arxiv}) and outside the scope of this paper.}


}



\subsection{Search and Rescue in a Disaster Affected Region}\label{subsec:simulation-results}
Consider that a team of two robots $\Gamma=\{r,r'\}$ are engaged in finding an \emph{unknown} number of targets (for example, victims) in a disaster affected region. Task of the robot team is to collaboratively find targets with high confidence (reduced uncertainty) which is facilitated by having different observations at different locations.
Due to poor bandwidth and other connectivity constraints, the robots have limited scope of communication between them. 
The simulations are done in a 2-D occupancy grid sub-divided into discrete cells with unique identifiers: $\{x_{i}\}$ for $i=[1\dots X]$ where $X$ is the maximum index of the cells. 
A target is either \emph{present} or \emph{not present} in a given cell $x_i$ and the presence of a target in $x_i$ is denoted as $x_i=1$, else  $x_i=0$.  
\rev{Each robot acquires an observation of its corresponding cell, with a probability of 25\% of acquiring the opposite value of the ground truth of that cell.}

We assume the following for ease of implementation and our algorithm is not limited to these assumptions:
(a) each robot precisely knows the locations  of all the robots in $\Gamma$; thus the belief is the probability distribution over the joint state $x\triangleq \{x_{i}\}$. Denote the known pose of any robot $r\in \Gamma$ at time instant $k$ by $\xi^r_k$; (b) the initial beliefs $b_0^r$ and $b_0^{r'}$ of both robots are given and  consistent, i.e. $b_0^r = b_0^{r'}=b_0$. The cells are initially independent of each other, i.e.~$b_0=\prob{x\mid p_0}=\prod_i \mathbb{P}_0(x_i)$; (c) at any time instant, each robot observes a single cell where it is located.  Based on these assumptions, it is not difficult to show that the cells remain independent of each other at any time instant $k$, i.e.~$b_k=\prob{x\mid \his_k, \xi^r_{0:k}, \xi^{r'}_{0:k}}=\prod_i \prob{x_i \mid \his_k, \xi^r_{0:k}, \xi^{r'}_{0:k}}\triangleq \prod_i b_k[x_i] $ for any history  $\his_k$.

To reduce the uncertainty of target occurrences in the locations of the workspace, we choose an information-theoretic reward function. Specifically, we consider (minus) entropy, which 
measures uncertainty over presence of targets at different locations. Recalling that according to the assumptions above, the cells are independent for any time instant $k$, we get
$\rho(b_k) \triangleq -H[x]=  \sum_i \sum_{j \in \{0,1\}} b_k[x_i = j]$ $\cdot \log b_k[x_i = j]$.


For a given epoch $\langle 1,2,\dots, E\rangle$, we do planning using Baseline I, Baseline II and our algorithms \mbox{\textsc{EnforceAC}}, \mbox{\textsc{R-EnforceAC}}, \mbox{\textsc{R-EnforceAC-simp}} at every time step with planning horizon $L=1$.
Each robot has four (\texttt{N}, \texttt{S}, \texttt{E}, \texttt{W}) motion primitives each of which moves the robot to a unit distance in the respective direction.
We consider 3 possible initializations for the prior belief of the robots:
\begin{itemize}
	\item  \textit{MaxEntropy} -- where $\prob{x_i}=0.5$ for all cells,
	\item  \textit{PriorKnowledge} -- where $\prob{x_i}=0.7$ if in the ground truth the cell is occupied and $\prob{x_i}=0.3$ otherwise,
	\item  \textit{Random} -- where each cell is assigned a random probability of occupancy.
\end{itemize}
At any time $k\in [1,E]$, for robot $r$, beliefs $b_k^r$, $b_k^{r'\mid r}$ and $b_k^{r\mid r'\mid r}$ are updated as formulated in Section \ref{subsec:verifyac} for the joint state $x$. 
Also, at some time instances in $E$ there can be a  restriction in communication even if the algorithm decides to communicate.
We denote a scenario with $m$ such time instances by \textit{comm-restr}=$m$, while \textit{comm-restr}=$0$ corresponds to a scenario with no communication restrictions. This restriction is not available, in advance, to any of the algorithms.

\begin{table}[t]
\caption{\scriptsize Search and Rescue simulation runtime statistics.}
\begin{center}
\label{table:simulation-runtimes}
\resizebox{\linewidth}{!}{

%
\begin{tabular}{ l c c c }
	\hline
					   & \textbf{\textit{Max-Entropy}} & \textbf{\textit{Prior-Knowledge}} & \textbf{\textit{Random}} \\
	\textbf{Algorithm} & \textbf{Initialization} & \textbf{Initialization} & \textbf{Initialization} \\
	\hline
	
	{\color{black!40!green}{Baseline-I}} 							& 1.2 $\pm$ 0.06 sec & 1.3 $\pm$ 0.06 sec & 1.4 $\pm$ 0.14 sec \\
	{\color{red}{Baseline-II}} 										& 1.3 $\pm$ 0.07 sec & 1.2 $\pm$ 0.07 sec & 1.3 $\pm$ 0.07 sec \\
	\hline
	
	{\color{cyan}\textsc{EnforceAC}} 								& 17.8 $\pm$ 0.70s sec & 13.1 $\pm$ 2.24 sec & 13.4 $\pm$ 0.20 sec \\
	\hline
	
	{\color{orange} \textsc{R-EnforceAC}}		& 							&						& \\
	{\color{orange} $\quad \epsilon=0.3$} 		& 23.1 $\pm$ 0.85s sec		& 18.1 $\pm$ 0.26 sec	& 15.5 $\pm$ 0.42 sec \\
	{\color{orange} $\quad \epsilon=0.7$} 		& 87.7 $\pm$ 4.79s sec		& 21.2 $\pm$ 0.72 sec	& 29.0 $\pm$ 0.40 sec \\
	{\color{orange} $\quad \epsilon=0.9$} 		& 408.6 $\pm$ 12.93s sec	& 52.7 $\pm$ 0.89 sec	& 145.6 $\pm$ 13.18 sec \\
	\hline
	
	{\color{blue} \textsc{R-EnforceAC}} 	& 						& 						& \\
	{\color{blue} $\quad \epsilon=0.3$} 		& 16.5 $\pm$ 3.12 sec	& 14.0 $\pm$ 0.21 sec	& 13.5 $\pm$ 0.31 sec \\
	{\color{blue} $\quad \epsilon=0.7$} 		& 48.1 $\pm$ 6.51 sec	& 16.3 $\pm$ 0.28 sec	& 22.5 $\pm$ 0.34 sec \\
	{\color{blue} $\quad \epsilon=0.9$} 		& 147.2 $\pm$ 5.16 sec	& 29.8 $\pm$ 0.39 sec	& 77.3 $\pm$ 0.93 sec \\
	\hline
\end{tabular}
}

\end{center}

\end{table}

\begin{table*}[t]
\caption{\scriptsize Search and Rescue simulation Action Inconsistency and Communications.}
\begin{center}
\label{table:simulation-performance}
\resizebox{0.99\textwidth}{!}{

%
\begin{tabular}{ l c c c c c }
	\hline
	& \multicolumn{5}{c}{\textbf{Action Inconsistencies}} \\
	\cline{2-6}
	\textbf{Algorithm} & \textbf{\textit{Max-Entropy}-Init} & \textbf{\textit{Prior-Knowledge}-Init} & \textbf{\textit{Random}-Init} &\shortstack{\textbf{\textit{Max-Entropy}-Init} \\ \textbf{comm-restric=20}} &  \shortstack{\textbf{\textit{Max-Entropy}-Init} \\ \textbf{comm-restric=30}} \\
	\hline
	
	{\color{black!40!green}{Baseline I}}				& 0 (0.0\%)							& 0 (0.0\%)							& 0 (0.0\%)							& 14 $\pm$ 2 (7.0 $\pm$ 1.0\%)		& 21 $\pm$ 3 (10.5 $\pm$ 1.5\%)		\\
	{\color{red}{Baseline II}}							& 179 $\pm$ 6 (89.5 $\pm$ 3.0\%)	& 180 $\pm$ 5 (90.0 $\pm$ 2.5\%)	& 173 $\pm$ 6 (86.5 $\pm$ 3.0\%)	& 179 $\pm$ 6 (89.5 $\pm$ 3.0\%)	& 179 $\pm$ 6 (89.5 $\pm$ 3.0\%)	\\
	\hline
	{\color{cyan}\textsc{EnforceAC}}							& 0 (0.0\%)							& 0 (0.0\%)							& 0 (0.0\%)							& --								& 16 $\pm$ 7 (8.0 $\pm$ 3.5\%)		\\
	\hline

	{\color{orange}\textsc{R-EnforceAC} $(\epsilon=0.3)$}			& 0 (0.0\%)							& 0 (0.0\%)							& 0 (0.0\%)							& 14 $\pm$ 1 (7.0 $\pm$ 0.5\%)		& 21 $\pm$ 2 (10.5 $\pm$ 1.0\%)		\\
	{\color{orange}\textsc{R-EnforceAC} $(\epsilon=0.7)$}			& 0 (0.0\%)							& 1 $\pm$ 1 (0.5 $\pm$ 0.5\%)		& 1 $\pm$ 1 (0.5 $\pm$ 0.5\%)		& 14 $\pm$ 1 (7.0 $\pm$ 0.5\%)		& 21 $\pm$ 2 (10.5 $\pm$ 1.0\%)		\\
	{\color{orange}\textsc{R-EnforceAC} $(\epsilon=0.9)$}			& 0 $\pm$ 1 (0.0 $\pm$ 0.5\%)		& 2 $\pm$ 3 (1.0 $\pm$ 1.5\%)		& 2 $\pm$ 1 (1.0 $\pm$ 0.5\%)		& --								& --								\\
	\hline

	{\color{blue}\textsc{R-EnforceAC-simp} $(\epsilon=0.3)$}	& 0 (0.0\%)							& 0 (0.0\%)							& 0 (0.0\%)							& 14 $\pm$ 1 (7.0 $\pm$ 0.5\%)		& 21 $\pm$ 2 (10.5 $\pm$ 1.0\%)		\\
	{\color{blue}\textsc{R-EnforceAC-simp} $(\epsilon=0.7)$}	& 0 (0.0\%)							& 0 $\pm$ 1 (0.0 $\pm$ 0.5\%)		& 1 $\pm$ 1 (0.5 $\pm$ 0.5\%)		& 14 $\pm$ 1 (7.0 $\pm$ 0.5\%)		& 21 $\pm$ 2 (10.5 $\pm$ 1.0\%)		\\
	{\color{blue}\textsc{R-EnforceAC-simp} $(\epsilon=0.9)$}	& 0 $\pm$ 1 (0.0 $\pm$ 0.5\%)		& 2 $\pm$ 3 (1.0 $\pm$ 1.5\%)		& 2 $\pm$ 1 (1.0 $\pm$ 0.5\%)		& 15 $\pm$ 2 (7.5 $\pm$ 1.0\%)		& --								\\
	\hline
	\hline

	& \multicolumn{5}{c}{\textbf{Communications}} \\
	\cline{2-6}
	\textbf{Algorithm} & \textbf{\textit{Max-Entropy}-Init} & \textbf{\textit{Prior-Knowledge}-Init} & \textbf{\textit{Random}-Init} & \shortstack{\textbf{\textit{Max-Entropy}-Init} \\ \textbf{comm-restric=20}} &  \shortstack{\textbf{\textit{Max-Entropy}-Init} \\ \textbf{comm-restric=30}} \\
	\hline
	
	{\color{black!40!green}\textsc{Baseline I}}					& 400 (100.0\%)						& 400 (100.0\%)						& 400 (100.0\%)							& 360 (90.0\%)							& 340 (85.0\%) \\
	{\color{red}\textsc{Baseline II}}							& 0 (0.0\%)							& 0 (0.0\%)							& 0 (0.0\%)								& 0 (0.0\%)								& 0 (0.0\%) \\
	\hline	
	{\color{cyan}\textsc{EnforceAC}}							& 238 $\pm$ 10 (59.5 $\pm$ 2.5\%)	& 250 $\pm$ 8 (62.5 $\pm$ 2.0\%)	& 241 $\pm$ 7 (60.25 $\pm$ 1.75\%)		& --									& 169 $\pm$ 59 (42.25 $\pm$ 14.75\%) \\
	\hline	
	{\color{orange}\textsc{R-EnforceAC} $(\epsilon=0.3)$}			& 215 $\pm$ 9 (53.75 $\pm$ 2.25\%)	& 225 $\pm$ 10 (56.25 $\pm$ 2.5\%)	& 237 $\pm$ 12 (59.25 $\pm$ 3.0\%)		& 194 $\pm$ 12 (48.5 $\pm$ 3.0\%)		& 185 $\pm$ 12 (46.25 $\pm$ 3.0\%) \\
	{\color{orange}\textsc{R-EnforceAC} $(\epsilon=0.7)$}			& 213 $\pm$ 10 (53.25 $\pm$ 2.5\%)	& 214 $\pm$ 5 (53.5 $\pm$ 1.25\%)	& 205 $\pm$ 13 (51.25 $\pm$ 3.25\%)		& 191 $\pm$ 13 (47.75 $\pm$ 3.25\%)		& 182 $\pm$ 12 (45.5 $\pm$ 3.0\%) \\
	{\color{orange}\textsc{R-EnforceAC} $(\epsilon=0.9)$}			& 175 $\pm$ 7 (43.75 $\pm$ 1.75\%)	& 167 $\pm$ 9 (41.75 $\pm$ 2.25\%)	& 169 $\pm$ 8 (42.25 $\pm$ 2.0\%)		& --									& -- \\
	\hline	
	{\color{blue}\textsc{R-EnforceAC-simp} $(\epsilon=0.3)$}	& 218 $\pm$ 10 (54.5 $\pm$ 2.5\%)	& 225 $\pm$ 10 (56.25 $\pm$ 2.5\%)	& 237 $\pm$ 12 (59.25 $\pm$ 3.0\%)		& 197 $\pm$ 9 (49.25 $\pm$ 2.25\%)		& 188 $\pm$ 9 (47.0 $\pm$ 2.25\%) \\
	{\color{blue}\textsc{R-EnforceAC-simp} $(\epsilon=0.7)$}	& 213 $\pm$ 10 (53.25 $\pm$ 2.5\%)	& 200 $\pm$ 7 (50.0 $\pm$ 1.75\%)	& 205 $\pm$ 13 (51.25 $\pm$ 3.25\%)		& 193 $\pm$ 10 (48.25 $\pm$ 2.5\%)		& 184 $\pm$ 10 (46.0 $\pm$ 2.5\%) \\
	{\color{blue}\textsc{R-EnforceAC-simp} $(\epsilon=0.9)$}	& 182 $\pm$ 7 (45.5 $\pm$ 1.75\%)	& 167 $\pm$ 9 (41.75 $\pm$ 2.25\%)	& 169 $\pm$ 8 (42.25 $\pm$ 2.0\%)		& 162 $\pm$ 8 (40.5 $\pm$ 2.0\%)		& -- \\
	\hline

\end{tabular}

}

\end{center}


\end{table*}

We run Baseline I, Baseline II, \mbox{\textsc{EnforceAC}}, \mbox{\textsc{R-EnforceAC}} and \mbox{\textsc{R-EnforceAC-simp}} with the above setting for epoch $E=200$ without communication restrictions, i.e.~\textit{comm-restr} = 0, with several communication restrictions (\textit{comm-restr} = 20 and \textit{comm-restr} = 30).


\begin{figure*}[t]
\centering


\subfigure[]
{
	\resizebox{0.45\linewidth}{!}
	{
\begin{tikzpicture}
\begin{axis}[
    legend style={at={(0.5, 1.2)}, anchor=north, legend columns=3},
    xlabel={\Large Time steps},
    ylabel={\Large Return Value},
    xmin=0, xmax=200,
    ymin=-62, ymax=-18,
    xtick={0, 50, 100, 150, 200},
    ytick={ -60, -50, -40, -30, -20},
    ymajorgrids=true,
    grid style=dashed,
    height=0.34\textheight,
    width=0.75\textwidth,
]

\addplot[mark size=1pt, color=black!40!green, style=thick,]
coordinates {
    (0, -61.086430205489336) (1, -60.61824192389243) (2, -60.458597688315535) (3, -60.15005364229551) (4, -59.6818653606986) (5, -59.5222211251217) (6, -59.096264222234595) (7, -58.727537000264924) (8, -58.727537000264924) (9, -58.87643681070805) (10, -58.522707890946045) (11, -58.87643681070805) (12, -58.727537000264924) (13, -58.56789276468803) (14, -58.727537000264924) (15, -58.56789276468803) (16, -58.30158009737781) (17, -57.8756231944907) (18, -57.49189767031337) (19, -57.33225343473648) (20, -56.86406515313956) (21, -56.555521107119546) (22, -56.08733282552264) (23, -55.927688589945745) (24, -55.45950030834883) (25, -55.45950030834882) (26, -55.45950030834882) (27, -55.048541707669365) (28, -54.58035342607245) (29, -54.11216514447554) (30, -53.64397686287863) (31, -53.175788581281715) (32, -52.7076002996848) (33, -52.23941201808789) (34, -52.079767782511) (35, -51.61157950091409) (36, -51.200620900234625) (37, -50.7746639973475) (38, -50.306475715750594) (39, -50.1468314801737) (40, -49.98718724459682) (41, -49.67864319857679) (42, -49.370099152556776) (43, -48.90191087095987) (44, -48.43372258936296) (45, -47.96553430776605) (46, -47.49734602616913) (47, -47.029157744572224) (48, -46.56096946297531) (49, -46.0927811813784) (50, -45.66682427849128) (51, -45.549411421624185) (52, -45.492181740706734) (53, -45.22586907339651) (54, -44.91732502737649) (55, -44.91732502737649) (56, -44.7576807917996) (57, -44.28949251020269) (58, -43.98094846418267) (59, -43.597222940005345) (60, -43.17126603711823) (61, -43.17126603711823) (62, -43.01162180154134) (63, -42.58566489865422) (64, -42.319352231344) (65, -42.3615836100538) (66, -42.3615836100538) (67, -42.2441707531867) (68, -42.1267578963196) (69, -41.86044522900937) (70, -41.44948662832991) (71, -40.981298346733) (72, -40.92406866581554) (73, -40.498111762928424) (74, -40.1895677169084) (75, -39.76361081402129) (76, -39.6039665784444) (77, -39.17800967555728) (78, -39.17800967555728) (79, -38.99913720823595) (80, -38.83949297265905) (81, -38.50097626976083) (82, -38.19243222374081) (83, -38.458744891051026) (84, -38.299100655474135) (85, -38.35633033639159) (86, -37.98760311442193) (87, -37.93037343350448) (88, -37.91537513129681) (89, -37.91537513129682) (90, -37.649062463986596) (91, -37.34051841796658) (92, -36.914561515079455) (93, -36.48860461219233) (94, -36.1923192880039) (95, -35.80859376382658) (96, -35.51230843963815) (97, -35.49731013743049) (98, -35.11358461325317) (99, -34.7448573912835) (100, -35.128582915460825) (101, -35.11358461325317) (102, -34.729859089075845) (103, -34.55098662175451) (104, -34.34214149755498) (105, -33.916184594667854) (106, -33.70733947046831) (107, -33.66510809175852) (108, -33.36882276757009) (109, -33.353824465362436) (110, -33.266384265373546) (111, -33.087511798052205) (112, -33.129743176762005) (113, -32.79122647386378) (114, -32.76125381698557) (115, -32.46496849279714) (116, -32.19865582548692) (117, -31.814930301309595) (118, -31.814930301309595) (119, -31.47641359841137) (120, -31.092688074234047) (121, -30.88384295003451) (122, -30.587557625846078) (123, -30.587557625846078) (124, -30.119369344249165) (125, -29.65118106265225) (126, -29.267455538474927) (127, -28.84149863558781) (128, -28.45777311141049) (129, -28.074047587233167) (130, -27.865202463033626) (131, -27.65635733883409) (132, -27.360072014645656) (133, -27.063786690457228) (134, -27.106018069167025) (135, -26.809732744978596) (136, -26.341544463381677) (137, -26.254104263392787) (138, -26.296335642102584) (139, -26.296335642102584) (140, -26.17892278523549) (141, -26.061509928368388) (142, -25.81019556326582) (143, -25.76796418455602) (144, -26.17892278523549) (145, -25.752965882348366) (146, -25.444421836328345) (147, -25.574093415027036) (148, -25.10590513343012) (149, -24.86642273355882) (150, -24.749009876691723) (151, -24.4826972093815) (152, -24.482697209381495) (153, -24.05674030649438) (154, -23.588552024897467) (155, -23.292266700709035) (156, -22.908541176531713) (157, -22.908541176531713) (158, -23.37672945812863) (159, -22.95077255524151) (160, -22.95077255524151) (161, -22.612255852343285) (162, -22.65448723105308) (163, -23.080444133940198) (164, -22.920799898363306) (165, -22.65448723105308) (166, -22.537074374185984) (167, -22.35820190686465) (168, -21.93224500397753) (169, -21.54851947980021) (170, -21.54851947980021) (171, -21.678191058498896) (172, -21.294465534321574) (173, -21.25223415561178) (174, -20.95594883142335) (175, -20.998180210133143) (176, -20.789335085933608) (177, -20.7593624290554) (178, -20.550517304855866) (179, -20.26923028287509) (180, -20.695187185762208) (181, -20.767629083543444) (182, -20.65021622667635) (183, -20.471343759355015) (184, -20.58875661622211) (185, -20.250239913323885) (186, -20.13282705645679) (187, -20.471343759355015) (188, -20.58875661622211) (189, -20.250239913323885) (190, -19.824283010436766) (191, -19.736842810447875) (192, -19.866514389146563) (193, -19.953954589135456) (194, -19.67224081055436) (195, -19.43275841068306) (196, -19.34531821069417) (197, -19.10583581082287) (198, -18.896990686623333) (199, -18.896990686623333) (200, -18.984430886612223)
};
\addlegendentry{ Baseline I }

\addplot[mark size=1pt, color=red, style=thick,]
coordinates {
    (0, -61.086430205489336) (1, -60.85233606469088) (2, -60.926785969912444) (3, -60.85233606469088) (4, -60.61824192389243) (5, -60.384147783093965) (6, -60.15005364229551) (7, -59.915959501497056) (8, -59.990409406718626) (9, -59.81354494683761) (10, -59.88799485205918) (11, -59.81354494683761) (12, -59.57945080603916) (13, -59.75631526592016) (14, -59.681865360698595) (15, -59.44777121990014) (16, -59.213677079101686) (17, -58.979582938303224) (18, -58.74548879750477) (19, -58.81993870272633) (20, -58.58584456192788) (21, -58.511394656706315) (22, -58.277300515907854) (23, -58.351750421129424) (24, -58.11765628033096) (25, -58.043206375109406) (26, -58.11765628033096) (27, -57.883562139532515) (28, -57.809112234310945) (29, -57.5750180935125) (30, -57.340923952714036) (31, -57.10682981191558) (32, -56.91496704982693) (33, -56.680872909028466) (34, -56.755322814250036) (35, -56.680872909028466) (36, -56.44677876823001) (37, -56.21268462743156) (38, -55.9785904866331) (39, -56.053040391854665) (40, -56.12749029707622) (41, -56.053040391854665) (42, -55.9785904866331) (43, -55.74449634583465) (44, -55.510402205036186) (45, -55.27630806423774) (46, -55.04221392343928) (47, -54.80811978264082) (48, -54.57402564184237) (49, -54.33993150104391) (50, -54.10583736024546) (51, -54.180287265467015) (52, -54.10583736024546) (53, -53.871743219447) (54, -53.63764907864855) (55, -53.40355493785009) (56, -53.47800484307166) (57, -53.2439107022732) (58, -53.16946079705164) (59, -52.93536665625318) (60, -52.701272515454725) (61, -52.50940975336606) (62, -52.58385965858763) (63, -52.349765517789166) (64, -52.27531561256761) (65, -52.08345285047895) (66, -51.84935870968049) (67, -51.657495947591826) (68, -51.42340180679337) (69, -51.18930766599492) (70, -50.955213525196456) (71, -50.72111938439801) (72, -50.48702524359955) (73, -50.25293110280109) (74, -50.01883696200264) (75, -49.784742821204176) (76, -49.55064868040572) (77, -49.31655453960727) (78, -49.39100444482883) (79, -49.156910304030376) (80, -49.082460398808806) (81, -48.89059763672015) (82, -48.6565034959217) (83, -48.73095340114325) (84, -48.80540330636482) (85, -48.62853884648382) (86, -48.45167438660281) (87, -48.217580245804356) (88, -47.983486105005895) (89, -47.74939196420745) (90, -47.515297823408986) (91, -47.28120368261053) (92, -47.04710954181208) (93, -46.813015401013615) (94, -46.57892126021516) (95, -46.344827119416706) (96, -46.110732978618245) (97, -45.87663883781979) (98, -45.642544697021336) (99, -45.450681934932675) (100, -45.52513184015424) (101, -45.333269078065584) (102, -45.258819172844014) (103, -45.15439661074424) (104, -44.920302469945796) (105, -44.815879907846025) (106, -44.99274436772703) (107, -44.76195242831143) (108, -44.68750252308987) (109, -44.4956397610012) (110, -44.57008966622277) (111, -44.80088160563837) (112, -44.7264317004168) (113, -44.903296160297806) (114, -44.66920201949935) (115, -44.59475211427778) (116, -44.40288935218912) (117, -44.168795211390666) (118, -44.40288935218912) (119, -44.168795211390666) (120, -43.97693244930201) (121, -43.785069687213344) (122, -43.59320692512469) (123, -43.82730106592314) (124, -43.722878503823374) (125, -43.48878436302492) (126, -43.38436180092515) (127, -43.19249903883649) (128, -43.08807647673672) (129, -42.896213714648056) (130, -42.704350952559395) (131, -42.512488190470734) (132, -42.32062542838207) (133, -42.12876266629341) (134, -42.362856807091866) (135, -42.170994045003205) (136, -41.93689990420475) (137, -42.170994045003205) (138, -42.24544395022477) (139, -42.068579490343765) (140, -41.83778755092817) (141, -41.64592478883951) (142, -41.88001892963797) (143, -41.954468834859526) (144, -41.88001892963797) (145, -41.64592478883951) (146, -41.45406202675085) (147, -41.34963946465108) (148, -41.245216902551306) (149, -41.197597264768675) (150, -41.093174702668904) (151, -40.98875214056913) (152, -41.22284628136759) (153, -40.98875214056913) (154, -40.79688937848047) (155, -40.60502661639181) (156, -40.41316385430315) (157, -40.64725799510161) (158, -40.721707900323175) (159, -40.64725799510161) (160, -40.721707900323175) (161, -40.64725799510161) (162, -40.41316385430315) (163, -40.64725799510161) (164, -40.721707900323175) (165, -40.64725799510161) (166, -40.721707900323175) (167, -40.64725799510161) (168, -40.41316385430315) (169, -40.22130109221449) (170, -40.45539523301294) (171, -40.689489373811405) (172, -40.45539523301295) (173, -40.22130109221449) (174, -40.02943833012583) (175, -40.26353247092428) (176, -40.07166970883562) (177, -40.146119614057184) (178, -40.04169705195742) (179, -39.86483259207641) (180, -40.05669535416508) (181, -40.23355981404608) (182, -40.04169705195742) (183, -39.96724714673586) (184, -40.15910990882452) (185, -39.92501576802606) (186, -39.733153005937396) (187, -39.96724714673586) (188, -40.15910990882452) (189, -39.92501576802606) (190, -39.733153005937396) (191, -39.541290243848735) (192, -39.436867681748964) (193, -39.628730443837625) (194, -39.436867681748964) (195, -39.2450049196603) (196, -39.479099060458765) (197, -39.287236298370104) (198, -39.05314215757164) (199, -38.94871959547187) (200, -39.18281373627033)
};
\addlegendentry{ Baseline II }

\addplot[mark size = 0.5pt, color=cyan, style=thick, ]
coordinates {
    (0, -61.086430205489336) (1, -60.85233606469088) (2, -60.69269182911398) (3, -60.384147783093965) (4, -59.915959501497056) (5, -59.44777121990015) (6, -59.25590845781148) (7, -58.904401460145934) (8, -58.96163114106338) (9, -58.80198690548649) (10, -58.87643681070805) (11, -58.80198690548649) (12, -58.80198690548649) (13, -58.727537000264924) (14, -58.49344285946647) (15, -58.49344285946647) (16, -58.259348718668015) (17, -58.06748595657935) (18, -57.77120063239092) (19, -57.813432011100716) (20, -57.34524372950381) (21, -56.8770554479069) (22, -56.408867166309975) (23, -56.24922293073308) (24, -55.940678884713066) (25, -55.47249060311616) (26, -55.54694050833771) (27, -55.31284636753926) (28, -54.9018877668598) (29, -54.667793626061346) (30, -53.96551120366597) (31, -53.731417062867514) (32, -53.02913464047215) (33, -52.56094635887524) (34, -52.6353962640968) (35, -52.34407244238089) (36, -52.269622537159336) (37, -52.035528396360874) (38, -51.72698435034086) (39, -51.80143425556242) (40, -51.44992725789687) (41, -51.18361459058664) (42, -50.91730192327642) (43, -50.4913450203893) (44, -50.25725087959084) (45, -49.63943121461507) (46, -49.40533707381661) (47, -48.78751740884084) (48, -48.55342326804238) (49, -47.9356036030666) (50, -47.70150946226815) (51, -47.350002464602596) (52, -47.115908323804135) (53, -46.65773289440525) (54, -46.423638753606795) (55, -45.980699488608806) (56, -46.24701215591903) (57, -45.82105525303191) (58, -45.642182785710574) (59, -45.303666082812356) (60, -44.96514937991412) (61, -44.68343560133304) (62, -44.86230806865437) (63, -44.68343560133304) (64, -44.401721822751945) (65, -44.06320511985372) (66, -43.829110979055265) (67, -44.06320511985372) (68, -43.94579226298662) (69, -43.75392950089796) (70, -43.519835360099506) (71, -43.327972598010845) (72, -43.09387845721239) (73, -42.859784316413936) (74, -42.31714612959546) (75, -41.89118922670834) (76, -41.465232323821226) (77, -41.465232323821226) (78, -41.305588088244335) (79, -41.0392754209341) (80, -40.613318518046995) (81, -40.37922437724853) (82, -40.18736161515987) (83, -40.42145575595833) (84, -40.49590566117989) (85, -40.336261425603) (86, -40.159396965721996) (87, -40.21461863919912) (88, -40.391483099080126) (89, -40.15738895828167) (90, -39.97851649096034) (91, -39.78665372887168) (92, -39.21404288517499) (93, -39.02218012308633) (94, -38.49180065809944) (95, -38.25770651730098) (96, -37.87157017662012) (97, -37.63747603582166) (98, -37.10709657083477) (99, -36.873002430036316) (100, -37.10709657083477) (101, -36.94745233525788) (102, -36.713358194459424) (103, -36.521495432370756) (104, -35.786994483463616) (105, -35.403268959286294) (106, -35.21140619719763) (107, -35.06177481381877) (108, -34.67804928964145) (109, -34.252092386754335) (110, -34.48618652755279) (111, -34.29432376546413) (112, -33.57208153838858) (113, -33.57208153838858) (114, -33.38021877629992) (115, -33.27579621420015) (116, -33.08393345211149) (117, -32.595785365834395) (118, -32.725456944533086) (119, -32.533594182444425) (120, -32.09065491744643) (121, -31.898792155357768) (122, -31.455852890359775) (123, -31.530302795581342) (124, -31.087363530583346) (125, -31.01291362536178) (126, -30.77881948456332) (127, -30.10178607876687) (128, -29.633597797169955) (129, -29.165409515573046) (130, -28.73945261268593) (131, -28.682222931768475) (132, -28.448128790970017) (133, -27.94772198286134) (134, -27.83030912599424) (135, -27.59621498519578) (136, -27.521765079974216) (137, -27.225479755785784) (138, -27.195507098907584) (139, -27.078094242040482) (140, -26.79680722005971) (141, -26.617934752738375) (142, -26.57900557541144) (143, -26.653455480633006) (144, -26.165307394355914) (145, -26.090857489134347) (146, -25.898994727045686) (147, -25.560478024147457) (148, -25.560478024147457) (149, -25.320995624276158) (150, -25.216573062176387) (151, -25.02471030008773) (152, -25.154381878786417) (153, -24.81586517588819) (154, -24.57638277601689) (155, -24.336900376145593) (156, -24.520072677778547) (157, -24.754166818577005) (158, -24.428135603596132) (159, -24.194041462797674) (160, -24.10090628989663) (161, -23.909043527807974) (162, -23.652578765825794) (163, -23.99109546872402) (164, -24.065545373945586) (165, -23.99109546872402) (166, -24.065545373945582) (167, -23.83905326884161) (168, -23.604959128043156) (169, -23.317857090389218) (170, -23.50971985247788) (171, -23.591771793393924) (172, -23.35767765259547) (173, -23.16581489050681) (174, -23.14344426932309) (175, -23.33530703141175) (176, -22.804927566424865) (177, -23.03902170722332) (178, -22.69511674525225) (179, -22.721947012691178) (180, -22.721947012691174) (181, -22.881591248268066) (182, -22.45563434538095) (183, -22.221540204582496) (184, -22.45563434538095) (185, -22.029677442493835) (186, -21.837814680405174) (187, -22.07190882120363) (188, -22.306002962002086) (189, -22.07190882120363) (190, -21.880046059114967) (191, -21.34966659412808) (192, -21.47933817282677) (193, -21.78788221884679) (194, -21.521569551536565) (195, -21.05338126993965) (196, -21.011149891229852) (197, -20.81928712914119) (198, -20.24369884287521) (199, -19.85997331869789) (200, -20.094067459496344)
};
\addlegendentry{ \textsc{EnforceAC} }

\addplot[mark size = 0.5pt, color=orange, style= very thick, dash pattern=on 5pt off 10pt, dash phase=0pt, legend image post style={dashed}]
coordinates {
    (0, -61.086430205489336) (1, -60.85233606469088) (2, -60.69269182911398) (3, -60.61824192389243) (4, -59.915959501497056) (5, -59.6818653606986) (6, -59.33035836303305) (7, -58.904401460145934) (8, -59.13849560094439) (9, -58.80198690548649) (10, -58.87643681070805) (11, -58.80198690548649) (12, -58.80198690548649) (13, -58.727537000264924) (14, -58.49344285946647) (15, -58.259348718668015) (16, -58.49344285946647) (17, -58.06748595657935) (18, -57.8756231944907) (19, -57.813432011100716) (20, -57.57933787030226) (21, -56.8770554479069) (22, -56.642961307108436) (23, -56.24922293073308) (24, -55.940678884713066) (25, -55.47249060311616) (26, -55.621390413559276) (27, -55.31284636753926) (28, -54.9018877668598) (29, -54.43369948526289) (30, -54.19960534446443) (31, -53.49732292206906) (32, -53.263228781270605) (33, -52.56094635887524) (34, -52.6353962640968) (35, -52.34407244238089) (36, -52.269622537159336) (37, -52.195172631937766) (38, -51.96107849113931) (39, -51.567340114763965) (40, -51.64179001998552) (41, -51.567340114763965) (42, -50.91730192327642) (43, -50.4913450203893) (44, -50.25725087959084) (45, -49.63943121461507) (46, -49.21347431172795) (47, -48.78751740884084) (48, -48.55342326804238) (49, -47.9356036030666) (50, -47.70150946226815) (51, -47.350002464602596) (52, -47.115908323804135) (53, -46.65773289440525) (54, -46.423638753606795) (55, -45.980699488608806) (56, -46.24701215591903) (57, -46.012918015120576) (58, -45.642182785710574) (59, -45.303666082812356) (60, -45.069571942013894) (61, -44.83547780121545) (62, -44.80550514433723) (63, -44.626632677015905) (64, -44.34491889843481) (65, -44.06320511985372) (66, -43.93353354115503) (67, -43.93353354115503) (68, -43.97576491986483) (69, -43.78390215777617) (70, -43.54980801697771) (71, -43.432395160110616) (72, -43.198301019312154) (73, -42.96420687851371) (74, -42.22970592960657) (75, -41.89118922670834) (76, -41.65709508590989) (77, -41.567244069417455) (78, -41.44983121255036) (79, -41.18351854524013) (80, -40.84500184234191) (81, -40.84500184234191) (82, -40.37681356074499) (83, -40.610907701543454) (84, -40.72758898547481) (85, -40.80203889069637) (86, -40.62517443081536) (87, -40.55072452559381) (88, -40.301632082587695) (89, -40.06753794178924) (90, -40.03531941527747) (91, -39.84345665318881) (92, -39.552132831472896) (93, -39.360270069384235) (94, -38.85986326127555) (95, -38.625769120477095) (96, -37.96571807679152) (97, -37.4975297951946) (98, -37.26343565439615) (99, -36.56115323200077) (100, -36.869697278020794) (101, -36.869697278020794) (102, -36.795247372799224) (103, -36.56115323200077) (104, -35.75147080493633) (105, -35.559608042847664) (106, -34.94178837787189) (107, -35.01623828309345) (108, -34.44064999682747) (109, -34.366200091605904) (110, -34.304008908215934) (111, -34.06991476741747) (112, -33.626975502419484) (113, -33.43511274033082) (114, -33.24324997824216) (115, -33.285481356951955) (116, -33.093618594863294) (117, -32.47579892988752) (118, -32.70989307068598) (119, -32.221744984408886) (120, -31.795788081521764) (121, -31.603925319433102) (122, -31.02833703316712) (123, -31.262431173965577) (124, -30.774283087688488) (125, -30.540188946890034) (126, -30.05204086061294) (127, -29.817946719814483) (128, -29.242358433548496) (129, -28.85863290937118) (130, -28.520116206472952) (131, -28.328253444384288) (132, -28.22383088228452) (133, -27.417125897789383) (134, -27.374894519079586) (135, -27.183031756990925) (136, -26.52298071330535) (137, -26.714843475394012) (138, -26.438517955885757) (139, -26.67261209668421) (140, -26.438517955885757) (141, -26.2044238150873) (142, -26.119229484731967) (143, -26.278873720308866) (144, -26.2044238150873) (145, -25.81068543871195) (146, -25.50214139269193) (147, -25.31027863060327) (148, -25.11841586851461) (149, -24.62025493003949) (150, -24.428392167950825) (151, -24.483613841427953) (152, -24.717707982226408) (153, -24.145097138529728) (154, -23.911002997731266) (155, -23.422854911454177) (156, -23.230992149365516) (157, -22.95168918728796) (158, -22.95168918728796) (159, -22.847266625188194) (160, -23.123592144696445) (161, -23.01916958259668) (162, -22.78507544179822) (163, -22.78507544179822) (164, -23.01916958259668) (165, -22.872515641787114) (166, -22.94696554700868) (167, -22.899345909226042) (168, -22.824896004004476) (169, -22.184705717844974) (170, -22.25915562306654) (171, -22.15473306096677) (172, -22.03266351796257) (173, -21.928240955862798) (174, -21.694146815064343) (175, -21.838389939370362) (176, -21.6465271772817) (177, -21.542104615181934) (178, -21.437682053082163) (179, -21.462931069681083) (180, -21.567353631780854) (181, -21.333259490982396) (182, -21.166645745492655) (183, -20.97478298340399) (184, -21.07920554550376) (185, -20.870360421304223) (186, -20.661515297104685) (187, -20.765937859204456) (188, -21.22411328860334) (189, -21.03225052651468) (190, -20.661515297104685) (191, -20.613895659322047) (192, -20.52404464282961) (193, -20.628467204929382) (194, -20.623078945856538) (195, -20.51865638375677) (196, -20.197122042969422) (197, -19.90083671878099) (198, -19.853217080998352) (199, -19.337409161618798) (200, -19.38502879940144)
};
\addlegendentry{ \textsc{R-EnforceAC} ($\epsilon=0.3$) }

\addplot[mark size = 0.5pt, color=blue, style=thick, dash pattern=on 5pt off 10pt, dash phase=7.5pt, legend image post style={dashed}]
coordinates {
    (0, -61.086430205489336) (1, -60.85233606469088) (2, -60.69269182911398) (3, -60.61824192389243) (4, -59.915959501497056) (5, -59.6818653606986) (6, -59.33035836303305) (7, -58.904401460145934) (8, -59.13849560094439) (9, -58.80198690548649) (10, -58.87643681070805) (11, -58.80198690548649) (12, -58.80198690548649) (13, -58.727537000264924) (14, -58.49344285946647) (15, -58.259348718668015) (16, -58.49344285946647) (17, -58.06748595657935) (18, -57.8756231944907) (19, -57.813432011100716) (20, -57.57933787030226) (21, -56.8770554479069) (22, -56.642961307108436) (23, -56.24922293073308) (24, -55.940678884713066) (25, -55.47249060311616) (26, -55.621390413559276) (27, -55.31284636753926) (28, -54.9018877668598) (29, -54.43369948526289) (30, -54.19960534446443) (31, -53.49732292206906) (32, -53.263228781270605) (33, -52.56094635887524) (34, -52.6353962640968) (35, -52.34407244238089) (36, -52.269622537159336) (37, -52.195172631937766) (38, -51.96107849113931) (39, -51.567340114763965) (40, -51.64179001998552) (41, -51.567340114763965) (42, -50.91730192327642) (43, -50.4913450203893) (44, -50.25725087959084) (45, -49.63943121461507) (46, -49.21347431172795) (47, -48.78751740884084) (48, -48.55342326804238) (49, -47.9356036030666) (50, -47.70150946226815) (51, -47.350002464602596) (52, -47.115908323804135) (53, -46.65773289440525) (54, -46.423638753606795) (55, -45.980699488608806) (56, -46.24701215591903) (57, -46.012918015120576) (58, -45.642182785710574) (59, -45.303666082812356) (60, -45.069571942013894) (61, -44.83547780121545) (62, -44.80550514433723) (63, -44.626632677015905) (64, -44.34491889843481) (65, -44.06320511985372) (66, -43.93353354115503) (67, -43.93353354115503) (68, -43.97576491986483) (69, -43.78390215777617) (70, -43.54980801697771) (71, -43.432395160110616) (72, -43.198301019312154) (73, -42.96420687851371) (74, -42.22970592960657) (75, -41.89118922670834) (76, -41.65709508590989) (77, -41.567244069417455) (78, -41.44983121255036) (79, -41.18351854524013) (80, -40.84500184234191) (81, -40.84500184234191) (82, -40.37681356074499) (83, -40.610907701543454) (84, -40.72758898547481) (85, -40.80203889069637) (86, -40.62517443081536) (87, -40.55072452559381) (88, -40.301632082587695) (89, -40.06753794178924) (90, -40.03531941527747) (91, -39.84345665318881) (92, -39.552132831472896) (93, -39.360270069384235) (94, -38.85986326127555) (95, -38.625769120477095) (96, -37.96571807679152) (97, -37.4975297951946) (98, -37.26343565439615) (99, -36.56115323200077) (100, -36.869697278020794) (101, -36.869697278020794) (102, -36.795247372799224) (103, -36.56115323200077) (104, -35.75147080493633) (105, -35.559608042847664) (106, -34.94178837787189) (107, -35.01623828309345) (108, -34.44064999682747) (109, -34.366200091605904) (110, -34.304008908215934) (111, -34.06991476741747) (112, -33.626975502419484) (113, -33.43511274033082) (114, -33.24324997824216) (115, -33.285481356951955) (116, -33.093618594863294) (117, -32.47579892988752) (118, -32.70989307068598) (119, -32.221744984408886) (120, -31.795788081521764) (121, -31.603925319433102) (122, -31.02833703316712) (123, -31.262431173965577) (124, -30.774283087688488) (125, -30.540188946890034) (126, -30.05204086061294) (127, -29.817946719814483) (128, -29.242358433548496) (129, -28.85863290937118) (130, -28.520116206472952) (131, -28.328253444384288) (132, -28.22383088228452) (133, -27.417125897789383) (134, -27.374894519079586) (135, -27.183031756990925) (136, -26.52298071330535) (137, -26.714843475394012) (138, -26.438517955885757) (139, -26.67261209668421) (140, -26.438517955885757) (141, -26.2044238150873) (142, -26.119229484731967) (143, -26.278873720308866) (144, -26.2044238150873) (145, -25.81068543871195) (146, -25.50214139269193) (147, -25.31027863060327) (148, -25.11841586851461) (149, -24.62025493003949) (150, -24.428392167950825) (151, -24.483613841427953) (152, -24.717707982226408) (153, -24.145097138529728) (154, -23.911002997731266) (155, -23.422854911454177) (156, -23.230992149365516) (157, -22.95168918728796) (158, -22.95168918728796) (159, -22.847266625188194) (160, -23.123592144696445) (161, -23.01916958259668) (162, -22.78507544179822) (163, -22.78507544179822) (164, -23.01916958259668) (165, -22.872515641787114) (166, -22.94696554700868) (167, -22.899345909226042) (168, -22.824896004004476) (169, -22.184705717844974) (170, -22.25915562306654) (171, -22.15473306096677) (172, -22.03266351796257) (173, -21.928240955862798) (174, -21.694146815064343) (175, -21.838389939370362) (176, -21.6465271772817) (177, -21.542104615181934) (178, -21.437682053082163) (179, -21.462931069681083) (180, -21.567353631780854) (181, -21.333259490982396) (182, -21.166645745492655) (183, -20.97478298340399) (184, -21.07920554550376) (185, -20.870360421304223) (186, -20.661515297104685) (187, -20.765937859204456) (188, -21.22411328860334) (189, -21.03225052651468) (190, -20.661515297104685) (191, -20.613895659322047) (192, -20.52404464282961) (193, -20.628467204929382) (194, -20.623078945856538) (195, -20.51865638375677) (196, -20.197122042969422) (197, -19.90083671878099) (198, -19.853217080998352) (199, -19.337409161618798) (200, -19.38502879940144)
};
\addlegendentry{ \textsc{R-EnforceAC-simp} ($\epsilon=0.3$) }

\end{axis}
\label{plot:simulation-results-objectives}
\end{tikzpicture}


	}
}
\subfigure[]
{
	\resizebox{0.45\linewidth}{!}
	{
\begin{tikzpicture}
\begin{axis}[
    legend style={at={(0.5, 1.2)}, anchor=north, legend columns=3},
    xlabel={\Large Time steps},
    ylabel={\Large Time steps since last \textsc{comm} ($p$)},
    xmin=0, xmax=200,
    ymin=-0.1, ymax=5,
    xtick={0, 50, 100, 150, 200},
    ytick={0, 1, 2, 3, 4, 5},
    ymajorgrids=true,
    grid style=dashed,
    height=0.33\textheight,
    width=0.75\textwidth,
]

\addplot[mark size=1pt, color=black!40!green, style=thick,]
coordinates {
    (0, 0) (1, 0) (2, 0) (3, 0) (4, 0) (5, 0) (6, 0) (7, 0) (8, 0) (9, 0) (10, 0) (11, 0) (12, 0) (13, 0) (14, 0) (15, 0) (16, 0) (17, 0) (18, 0) (19, 0) (20, 0) (21, 0) (22, 0) (23, 0) (24, 0) (25, 0) (26, 0) (27, 0) (28, 0) (29, 0) (30, 0) (31, 0) (32, 0) (33, 0) (34, 0) (35, 0) (36, 0) (37, 0) (38, 0) (39, 0) (40, 0) (41, 0) (42, 0) (43, 0) (44, 0) (45, 0) (46, 0) (47, 0) (48, 0) (49, 0) (50, 0) (51, 0) (52, 0) (53, 0) (54, 0) (55, 0) (56, 0) (57, 0) (58, 0) (59, 0) (60, 0) (61, 0) (62, 0) (63, 0) (64, 0) (65, 0) (66, 0) (67, 0) (68, 0) (69, 0) (70, 0) (71, 0) (72, 0) (73, 0) (74, 0) (75, 0) (76, 0) (77, 0) (78, 0) (79, 0) (80, 0) (81, 0) (82, 0) (83, 0) (84, 0) (85, 0) (86, 0) (87, 0) (88, 0) (89, 0) (90, 0) (91, 0) (92, 0) (93, 0) (94, 0) (95, 0) (96, 0) (97, 0) (98, 0) (99, 0) (100, 0) (101, 0) (102, 0) (103, 0) (104, 0) (105, 0) (106, 0) (107, 0) (108, 0) (109, 0) (110, 0) (111, 0) (112, 0) (113, 0) (114, 0) (115, 0) (116, 0) (117, 0) (118, 0) (119, 0) (120, 0) (121, 0) (122, 0) (123, 0) (124, 0) (125, 0) (126, 0) (127, 0) (128, 0) (129, 0) (130, 0) (131, 0) (132, 0) (133, 0) (134, 0) (135, 0) (136, 0) (137, 0) (138, 0) (139, 0) (140, 0) (141, 0) (142, 0) (143, 0) (144, 0) (145, 0) (146, 0) (147, 0) (148, 0) (149, 0) (150, 0) (151, 0) (152, 0) (153, 0) (154, 0) (155, 0) (156, 0) (157, 0) (158, 0) (159, 0) (160, 0) (161, 0) (162, 0) (163, 0) (164, 0) (165, 0) (166, 0) (167, 0) (168, 0) (169, 0) (170, 0) (171, 0) (172, 0) (173, 0) (174, 0) (175, 0) (176, 0) (177, 0) (178, 0) (179, 0) (180, 0) (181, 0) (182, 0) (183, 0) (184, 0) (185, 0) (186, 0) (187, 0) (188, 0) (189, 0) (190, 0) (191, 0) (192, 0) (193, 0) (194, 0) (195, 0) (196, 0) (197, 0) (198, 0) (199, 0) (200, 0)
};
\addlegendentry{ Baseline I }

\addplot[mark size=1pt, color=red, style=thick,]
coordinates {
    (0, 0) (1, 1) (2, 2) (3, 3) (4, 4) (5, 5) (6, 6) (7, 7) (8, 8) (9, 9) (10, 10) (11, 11) (12, 12) (13, 13) (14, 14) (15, 15) (16, 16) (17, 17) (18, 18) (19, 19) (20, 20) (21, 21) (22, 22) (23, 23) (24, 24) (25, 25) (26, 26) (27, 27) (28, 28) (29, 29) (30, 30) (31, 31) (32, 32) (33, 33) (34, 34) (35, 35) (36, 36) (37, 37) (38, 38) (39, 39) (40, 40) (41, 41) (42, 42) (43, 43) (44, 44) (45, 45) (46, 46) (47, 47) (48, 48) (49, 49) (50, 50) (51, 51) (52, 52) (53, 53) (54, 54) (55, 55) (56, 56) (57, 57) (58, 58) (59, 59) (60, 60) (61, 61) (62, 62) (63, 63) (64, 64) (65, 65) (66, 66) (67, 67) (68, 68) (69, 69) (70, 70) (71, 71) (72, 72) (73, 73) (74, 74) (75, 75) (76, 76) (77, 77) (78, 78) (79, 79) (80, 80) (81, 81) (82, 82) (83, 83) (84, 84) (85, 85) (86, 86) (87, 87) (88, 88) (89, 89) (90, 90) (91, 91) (92, 92) (93, 93) (94, 94) (95, 95) (96, 96) (97, 97) (98, 98) (99, 99) (100, 100) (101, 101) (102, 102) (103, 103) (104, 104) (105, 105) (106, 106) (107, 107) (108, 108) (109, 109) (110, 110) (111, 111) (112, 112) (113, 113) (114, 114) (115, 115) (116, 116) (117, 117) (118, 118) (119, 119) (120, 120) (121, 121) (122, 122) (123, 123) (124, 124) (125, 125) (126, 126) (127, 127) (128, 128) (129, 129) (130, 130) (131, 131) (132, 132) (133, 133) (134, 134) (135, 135) (136, 136) (137, 137) (138, 138) (139, 139) (140, 140) (141, 141) (142, 142) (143, 143) (144, 144) (145, 145) (146, 146) (147, 147) (148, 148) (149, 149) (150, 150) (151, 151) (152, 152) (153, 153) (154, 154) (155, 155) (156, 156) (157, 157) (158, 158) (159, 159) (160, 160) (161, 161) (162, 162) (163, 163) (164, 164) (165, 165) (166, 166) (167, 167) (168, 168) (169, 169) (170, 170) (171, 171) (172, 172) (173, 173) (174, 174) (175, 175) (176, 176) (177, 177) (178, 178) (179, 179) (180, 180) (181, 181) (182, 182) (183, 183) (184, 184) (185, 185) (186, 186) (187, 187) (188, 188) (189, 189) (190, 190) (191, 191) (192, 192) (193, 193) (194, 194) (195, 195) (196, 196) (197, 197) (198, 198) (199, 199) (200, 200)
};
\addlegendentry{ Baseline II }

\addplot[mark size = 0.5pt, color=cyan, style=thick,]
coordinates {
    (0, 0) (1, 1) (2, 1) (3, 1) (4, 1) (5, 1) (6, 1) (7, 1) (8, 1) (9, 1) (10, 2) (11, 3) (12, 1) (13, 2) (14, 3) (15, 1) (16, 2) (17, 3) (18, 1) (19, 1) (20, 1) (21, 1) (22, 1) (23, 1) (24, 1) (25, 1) (26, 1) (27, 1) (28, 1) (29, 1) (30, 1) (31, 2) (32, 1) (33, 1) (34, 2) (35, 1) (36, 2) (37, 1) (38, 2) (39, 1) (40, 2) (41, 1) (42, 2) (43, 1) (44, 2) (45, 1) (46, 1) (47, 1) (48, 2) (49, 1) (50, 2) (51, 1) (52, 2) (53, 1) (54, 2) (55, 1) (56, 2) (57, 1) (58, 2) (59, 1) (60, 2) (61, 1) (62, 1) (63, 1) (64, 1) (65, 1) (66, 1) (67, 1) (68, 1) (69, 2) (70, 1) (71, 2) (72, 1) (73, 2) (74, 1) (75, 1) (76, 1) (77, 1) (78, 1) (79, 1) (80, 1) (81, 2) (82, 1) (83, 2) (84, 1) (85, 2) (86, 1) (87, 2) (88, 3) (89, 4) (90, 1) (91, 2) (92, 1) (93, 2) (94, 1) (95, 2) (96, 1) (97, 1) (98, 1) (99, 2) (100, 1) (101, 2) (102, 1) (103, 2) (104, 1) (105, 1) (106, 2) (107, 1) (108, 2) (109, 1) (110, 2) (111, 1) (112, 2) (113, 1) (114, 2) (115, 1) (116, 2) (117, 1) (118, 2) (119, 1) (120, 1) (121, 2) (122, 1) (123, 2) (124, 1) (125, 2) (126, 1) (127, 1) (128, 1) (129, 1) (130, 1) (131, 1) (132, 1) (133, 1) (134, 1) (135, 2) (136, 3) (137, 1) (138, 1) (139, 2) (140, 1) (141, 1) (142, 2) (143, 3) (144, 1) (145, 2) (146, 1) (147, 2) (148, 1) (149, 1) (150, 2) (151, 1) (152, 2) (153, 1) (154, 1) (155, 1) (156, 1) (157, 2) (158, 1) (159, 2) (160, 3) (161, 4) (162, 1) (163, 1) (164, 2) (165, 1) (166, 2) (167, 1) (168, 2) (169, 1) (170, 2) (171, 1) (172, 2) (173, 1) (174, 2) (175, 1) (176, 2) (177, 1) (178, 2) (179, 1) (180, 2) (181, 1) (182, 2) (183, 3) (184, 4) (185, 1) (186, 2) (187, 1) (188, 1) (189, 2) (190, 1) (191, 2) (192, 1) (193, 2) (194, 1) (195, 2) (196, 3) (197, 4) (198, 1) (199, 1) (200, 2)
};
\addlegendentry{ \textsc{EnforceAC} }

\addplot[mark size = 0.5pt, color=orange, style=very thick, dash pattern=on 5pt off 10pt, dash phase=0pt, legend image post style={dashed}]
coordinates {
    (0, 0) (1, 1) (2, 1) (3, 1) (4, 1) (5, 2) (6, 1) (7, 1) (8, 2) (9, 1) (10, 2) (11, 3) (12, 1) (13, 2) (14, 3) (15, 4) (16, 1) (17, 1) (18, 2) (19, 1) (20, 2) (21, 1) (22, 2) (23, 1) (24, 1) (25, 1) (26, 1) (27, 1) (28, 1) (29, 1) (30, 2) (31, 1) (32, 2) (33, 1) (34, 2) (35, 1) (36, 2) (37, 1) (38, 2) (39, 1) (40, 2) (41, 1) (42, 1) (43, 1) (44, 2) (45, 1) (46, 1) (47, 1) (48, 2) (49, 1) (50, 2) (51, 1) (52, 2) (53, 1) (54, 2) (55, 1) (56, 1) (57, 2) (58, 1) (59, 1) (60, 2) (61, 3) (62, 1) (63, 1) (64, 1) (65, 1) (66, 1) (67, 1) (68, 1) (69, 2) (70, 3) (71, 1) (72, 2) (73, 3) (74, 1) (75, 1) (76, 2) (77, 1) (78, 1) (79, 1) (80, 1) (81, 2) (82, 1) (83, 2) (84, 1) (85, 2) (86, 1) (87, 2) (88, 3) (89, 4) (90, 1) (91, 2) (92, 1) (93, 2) (94, 1) (95, 2) (96, 1) (97, 1) (98, 2) (99, 1) (100, 2) (101, 1) (102, 2) (103, 3) (104, 1) (105, 2) (106, 1) (107, 2) (108, 1) (109, 2) (110, 1) (111, 2) (112, 1) (113, 2) (114, 3) (115, 1) (116, 2) (117, 1) (118, 2) (119, 1) (120, 1) (121, 2) (122, 1) (123, 2) (124, 1) (125, 2) (126, 1) (127, 2) (128, 1) (129, 1) (130, 1) (131, 2) (132, 3) (133, 1) (134, 1) (135, 2) (136, 1) (137, 2) (138, 1) (139, 2) (140, 1) (141, 2) (142, 1) (143, 1) (144, 2) (145, 1) (146, 1) (147, 2) (148, 3) (149, 1) (150, 2) (151, 1) (152, 2) (153, 1) (154, 2) (155, 1) (156, 2) (157, 1) (158, 1) (159, 2) (160, 1) (161, 2) (162, 3) (163, 1) (164, 2) (165, 1) (166, 2) (167, 3) (168, 4) (169, 1) (170, 2) (171, 1) (172, 1) (173, 2) (174, 3) (175, 1) (176, 2) (177, 1) (178, 2) (179, 1) (180, 2) (181, 3) (182, 1) (183, 2) (184, 3) (185, 1) (186, 1) (187, 2) (188, 1) (189, 2) (190, 1) (191, 2) (192, 1) (193, 2) (194, 1) (195, 2) (196, 1) (197, 1) (198, 2) (199, 1) (200, 2)
};
\addlegendentry{ \textsc{R-EnforceAC} ($\epsilon=0.3$) }

\addplot[mark size = 0.5pt, color=blue, style=thick, dash pattern=on 5pt off 10pt, dash phase=7.5pt, legend image post style={dashed}]
coordinates {
    (0, 0) (1, 1) (2, 1) (3, 1) (4, 1) (5, 2) (6, 1) (7, 1) (8, 2) (9, 1) (10, 2) (11, 3) (12, 1) (13, 2) (14, 3) (15, 4) (16, 1) (17, 1) (18, 2) (19, 1) (20, 2) (21, 1) (22, 2) (23, 1) (24, 1) (25, 1) (26, 1) (27, 1) (28, 1) (29, 1) (30, 2) (31, 1) (32, 2) (33, 1) (34, 2) (35, 1) (36, 2) (37, 1) (38, 2) (39, 1) (40, 2) (41, 1) (42, 1) (43, 1) (44, 2) (45, 1) (46, 1) (47, 1) (48, 2) (49, 1) (50, 2) (51, 1) (52, 2) (53, 1) (54, 2) (55, 1) (56, 1) (57, 2) (58, 1) (59, 1) (60, 2) (61, 3) (62, 1) (63, 1) (64, 1) (65, 1) (66, 1) (67, 1) (68, 1) (69, 2) (70, 3) (71, 1) (72, 2) (73, 3) (74, 1) (75, 1) (76, 2) (77, 1) (78, 1) (79, 1) (80, 1) (81, 2) (82, 1) (83, 2) (84, 1) (85, 2) (86, 1) (87, 2) (88, 3) (89, 4) (90, 1) (91, 2) (92, 1) (93, 2) (94, 1) (95, 2) (96, 1) (97, 1) (98, 2) (99, 1) (100, 2) (101, 1) (102, 2) (103, 3) (104, 1) (105, 2) (106, 1) (107, 2) (108, 1) (109, 2) (110, 1) (111, 2) (112, 1) (113, 2) (114, 3) (115, 1) (116, 2) (117, 1) (118, 2) (119, 1) (120, 1) (121, 2) (122, 1) (123, 2) (124, 1) (125, 2) (126, 1) (127, 2) (128, 1) (129, 1) (130, 1) (131, 2) (132, 3) (133, 1) (134, 1) (135, 2) (136, 1) (137, 2) (138, 1) (139, 2) (140, 1) (141, 2) (142, 1) (143, 1) (144, 2) (145, 1) (146, 1) (147, 2) (148, 3) (149, 1) (150, 2) (151, 1) (152, 2) (153, 1) (154, 2) (155, 1) (156, 2) (157, 1) (158, 1) (159, 2) (160, 1) (161, 2) (162, 3) (163, 1) (164, 2) (165, 1) (166, 2) (167, 3) (168, 4) (169, 1) (170, 2) (171, 1) (172, 1) (173, 2) (174, 3) (175, 1) (176, 2) (177, 1) (178, 2) (179, 1) (180, 2) (181, 3) (182, 1) (183, 2) (184, 3) (185, 1) (186, 1) (187, 2) (188, 1) (189, 2) (190, 1) (191, 2) (192, 1) (193, 2) (194, 1) (195, 2) (196, 1) (197, 1) (198, 2) (199, 1) (200, 2)
};
\addlegendentry{ \textsc{R-EnforceAC-simp} ($\epsilon=0.3$) }

\end{axis}
\label{plot:simulation-results-p-values}
\end{tikzpicture}


	}
}

\caption{
Results \rev{from a specific seed run} of simulations of different algorithms with \textit{Prior-Knowledge} initialization and zero communication restrictions, running for $E = 200$ time steps.
(a) Comparison of the return values over time for different algorithms.
(b) Comparison of the number of time steps since last communication ($p$-values) over time for different algorithms.
}
\label{fig:simulation-results}
\end{figure*}


Table~\ref{table:simulation-runtimes} presents statistical runtime summary of the algorithms for different scenarios, where each scenario starts from a different initial belief (\textit{MaxEntropy}, \textit{PriorKnowledge} or \textit{Random}).
In order to ensure that all of the runs of the same scenario are identical, a fixed seed value is used for all the stochastic calculations (i.e. observations sampling, motion models).
For statistical runtime evaluation, each scenario is repeated 10 times with the same seed value.
Table~\ref{table:simulation-performance} presents statistical summary of the algorithms for different scenarios, where each scenario starts from a different initial belief (\textit{MaxEntropy}, \textit{PriorKnowledge} or \textit{Random}) and with different communication restrictions (0, 20, or 30).
The \rev{algorithms are evaluated} by the amount of Action Inconsistencies that occurred and the amount of \textsc{comm}s that were triggered in the simulations.
For statistical performance evaluation, each scenario is repeated 10 times, where each repetition is run with a different seed value to simulate different possible outcomes for the same scenario.
Figure~\ref{fig:simulation-results} illustrates the return function values (evaluated at inference),
and the number of time steps since the last communication, both for a specific scenario input. 

We can see that for \mbox{\textsc{EnforceAC}}, the robots do not communicate for up to a maximum of 5 consecutive time steps; nevertheless, as shown in Table~\ref{table:simulation-performance}, in all planning sessions there were no action inconsistencies despite the robots having inconsistent beliefs (\textit{Not-AC} is zero for these scenario inputs, unless there are communication restrictions).
On the other hand, Baseline-I communicated two-way at each planning session, resulting in no action inconsistencies, while Baseline-II did not communicate at all, resulting in numerous inconsistent actions.

Considering no communication restrictions (\textit{comm-restr} = 0), \mbox{\textsc{EnforceAC}} reduces the number of one-way \textsc{comm}s by $40$-$60\%$ compared to Baseline-I, and in all cases ensures consistent decision making between the robots, despite having  inconsistent beliefs, as shown in Table~\ref{table:simulation-performance}.
However, we can see an increase by a factor of 10 in the runtime of \mbox{\textsc{EnforceAC}} compared to Baseline I.
As mentioned before, we note that for no communication-restrictions, \mbox{\textsc{EnforceAC}} never results in action inconsistency, by providing deterministic guarantee of Action Consistency or triggering \textsc{comm} to reach Action Consistency.
The only cases where \mbox{\textsc{EnforceAC}} results in action inconsistency is when there is a communication restriction, which is due to the fact that a \textsc{comm} trigger was not able to occur due to the restriction, and therefore action inconsistency occurred.

\revnew{However, compared with Baseline-I, \textsc{EnforceAC} produces somewhat lower reported objective values in this run, although the values remain well above those obtained by Baseline-II, as shown in Fig.~\ref{plot:simulation-results-objectives}. As discussed in Sections~\ref{sec:ProblemFormulation} and~\ref{sec:PlanningImplications} and illustrated in Fig.~\ref{fig:objective-surface}, beliefs conditioned on different data may induce different objective profiles, objective values, and optimal joint action sequences. 
}


In comparison, algorithms \mbox{\textsc{R-EnforceAC}} and \mbox{\textsc{R-EnforceAC-simp}}, with different values of $\epsilon$, provide a compromise between the number of inconsistencies, the amount of communications, and the runtime.
As can be seen in Table~\ref{table:simulation-performance}, in the scenarios where \textit{comm-restr}=0, the amount of action inconsistencies is between 0-1.5\%, the amount of communication relatively to \mbox{\textsc{EnforceAC}} decreased in most scenarios by 5-20\%. However, the runtime of the algorithms is increasing the higher the value of $\epsilon$. The decrease in the amount of communication causes the $p$-values to be higher and therefore increases time steps runtime exponentially.
Algorithm \mbox{\textsc{R-EnforceAC-simp}} shows improvement of 28-63\% in runtime compared to \mbox{\textsc{R-EnforceAC}}, with almost similar results of action inconsistencies and communications.

\begin{figure*}[htbp]
	\centering
	\subfigure[  
	]{
		\resizebox{0.9\linewidth}{!}
		{
\begin{tikzpicture}
    \begin{axis} [
        hide x axis, xmin=-2, xmax=202,
        axis y line*=right, ymin=0, ymax=1.05,
        ytick={0, 0.2, 0.4, 0.6, 0.8, 1.0},
        height=0.2\textheight,
        width=0.99\textwidth,
        ]
    \end{axis}
    \begin{axis} [
        legend style={at={(0.5, 1.2)}, anchor=north, legend columns=3, column sep=10pt},
        xlabel={Time step},
        ylabel={\mrac Probability Guarantee},
        xmin=-2, xmax=202,
        ymin=0, ymax=1.05,
        ytick={0, 0.2, 0.4, 0.6, 0.8, 1.0},
        height=0.2\textheight,
        width=0.99\textwidth,
        ]
    \addlegendimage{line width=2.2pt, color=black!40!green}
    \addlegendentry{ Consistent Action Selection }

    \addlegendimage{line width=2.2pt, color=red}
    \addlegendentry{ Inconsistent Action Selection }

    \addlegendimage{dashed, color=blue}
    \addlegendentry{ Threshold $1-\epsilon$ }

    \addplot [ybar, bar width=0.3, color=black!40!green, fill=black!40!green]
    coordinates {
        (0, 0.0) (1, 1.0) (2, 1.0) (3, 0.7) (4, 1.0) (5, 1.0) (6, 1.0) (7, 1.0) (8, 1.0) (9, 1.0) (10, 1.0) (11, 1.0) (12, 0.44699999999999995) (13, 0.26010000000000005) (14, 1.0) (15, 1.0) (16, 1.0) (17, 1.0) (18, 1.0) (19, 1.0) (20, 0.7) (21, 1.0) (22, 0.7) (23, 1.0) (24, 1.0) (25, 0.7) (26, 0.7) (27, 1.0) (28, 1.0) (29, 1.0) (30, 1.0) (31, 1.0) (32, 1.0) (33, 1.0) (34, 1.0) (35, 1.0) (36, 1.0) (37, 1.0) (38, 1.0) (39, 1.0) (40, 1.0) (41, 0.875) (42, 1.0) (43, 1.0) (44, 1.0) (45, 1.0) (46, 1.0) (47, 1.0) (48, 1.0) (49, 1.0) (50, 1.0) (51, 1.0) (52, 1.0) (53, 1.0) (54, 1.0) (55, 1.0) (56, 1.0) (57, 1.0) (58, 1.0) (59, 1.0) (60, 1.0) (61, 1.0) (62, 1.0) (63, 1.0) (64, 1.0) (65, 0.9396306818181818) (66, 1.0) (67, 1.0) (68, 0.3000000000000001) (69, 1.0) (70, 1.0) (71, 1.0) (72, 0.49218749999999994) (73, 0.73828125) (74, 1.0) (75, 0.875) (76, 1.0) (77, 1.0) (78, 1.0) (79, 0.875) (80, 1.0) (81, 1.0) (82, 1.0) (83, 1.0) (84, 1.0) (85, 1.0) (86, 1.0) (87, 1.0) (88, 1.0) (89, 0.9999999999999999) (90, 1.0) (91, 1.0) (92, 1.0) (93, 1.0) (94, 1.0) (95, 1.0) (96, 0.984375) (97, 1.0) (98, 1.0) (99, 1.0) (100, 1.0) (101, 1.0) (102, 1.0) (103, 1.0) (105, 1.0) (106, 1.0) (107, 1.0) (108, 1.0) (109, 1.0) (110, 1.0) (111, 1.0) (112, 1.0) (113, 1.0) (114, 1.0) (115, 1.0) (116, 1.0) (117, 1.0) (118, 1.0) (119, 1.0) (120, 1.0) (121, 1.0) (122, 1.0) (123, 1.0) (124, 1.0) (125, 1.0) (126, 1.0) (127, 1.0) (128, 1.0) (129, 1.0) (130, 1.0) (131, 0.875) (132, 1.0) (133, 1.0) (134, 0.875) (135, 1.0) (136, 1.0) (137, 1.0) (138, 0.9545454545454546) (139, 1.0) (140, 0.9545454545454546) (141, 0.9111570247933886) (142, 1.0) (143, 1.0) (144, 1.0) (145, 1.0) (146, 1.0) (147, 1.0) (148, 1.0) (149, 1.0) (150, 1.0) (151, 1.0) (152, 1.0) (153, 1.0) (154, 1.0) (155, 0.984375) (156, 1.0) (157, 1.0) (158, 1.0) (159, 1.0) (160, 1.0) (161, 1.0) (162, 1.0) (163, 1.0) (164, 1.0) (165, 1.0) (166, 1.0) (167, 0.9850852272727273) (168, 1.0) (169, 1.0) (170, 0.984375) (171, 1.0) (172, 1.0) (173, 1.0) (174, 1.0) (175, 1.0) (176, 1.0) (177, 1.0) (178, 1.0) (179, 1.0) (180, 1.0) (181, 1.0) (182, 1.0) (183, 1.0) (184, 1.0) (185, 1.0) (186, 1.0) (187, 1.0) (188, 1.0) (189, 1.0) (190, 1.0) (191, 1.0) (192, 1.0) (193, 0.3000000000000001) (194, 0.4200000000000001) (195, 1.0) (196, 1.0) (197, 1.0) (198, 1.0) (199, 0.8749999999999999) (200, 1.0)
    };

    \addplot [ybar, bar width=0.3, color=red, fill=red]
    coordinates {
        (104, 0.30000000000000004)
    };

    \addplot[draw=blue, dashed] 
    coordinates {
        (-5, 0.3) (205, 0.3)
    };

    \end{axis}
\end{tikzpicture}

		}		
		\label{plot:simulation-results-epsilon-0.7-guarantees1}
	}

	\subfigure[ 
	]{
		\resizebox{0.9\linewidth}{!}
		{
\begin{tikzpicture}
    \begin{axis} [
        hide x axis, xmin=-2, xmax=202,
        axis y line*=right, ymin=0, ymax=1.05,
        ytick={0, 0.2, 0.4, 0.6, 0.8, 1.0},
        height=0.2\textheight,
        width=0.99\textwidth,
        ]
    \end{axis}
    \begin{axis} [
        legend style={at={(0.5, 1.2)}, anchor=north, legend columns=3, column sep=10pt},
        xlabel={Time step},
        ylabel={\mrac Probability Guarantee},
        xmin=-2, xmax=202,
        ymin=0, ymax=1.05,
        ytick={0, 0.2, 0.4, 0.6, 0.8, 1.0},
        height=0.2\textheight,
        width=0.99\textwidth,
        ]
    \addlegendimage{line width=2.2pt, color=black!40!green}
    \addlegendentry{ Consistent Action Selection }

    \addlegendimage{line width=2.2pt, color=red}
    \addlegendentry{ Inconsistent Action Selection }

    \addlegendimage{dashed, color=blue}
    \addlegendentry{ Threshold $1-\epsilon$ }

    \addplot [ybar, bar width=0.3, color=black!40!green, fill=black!40!green]
    coordinates {
        (0, 0.0) (1, 1.0) (2, 1.0) (3, 0.7) (4, 1.0) (5, 1.0) (6, 1.0) (7, 1.0) (8, 1.0) (9, 0.7) (10, 0.7) (11, 1.0) (12, 0.20588235294117652) (13, 1.0) (14, 1.0) (15, 1.0) (16, 1.0) (17, 1.0) (18, 1.0) (19, 1.0) (20, 0.7) (21, 1.0) (22, 0.7) (23, 1.0) (24, 1.0) (25, 0.7) (26, 0.7) (27, 1.0) (28, 1.0) (29, 1.0) (30, 1.0) (31, 0.7) (32, 0.7) (33, 1.0) (34, 1.0) (35, 0.20588235294117652) (36, 1.0) (37, 0.3693771626297579) (38, 0.51) (39, 1.0) (40, 1.0) (41, 0.875) (42, 1.0) (43, 1.0) (44, 1.0) (45, 1.0) (46, 1.0) (47, 1.0) (48, 1.0) (49, 0.875) (50, 0.875) (51, 1.0) (52, 1.0) (53, 1.0) (54, 1.0) (55, 0.9545454545454547) (56, 1.0) (57, 1.0) (58, 1.0) (59, 1.0) (60, 0.9545454545454546) (61, 0.9566115702479339) (62, 1.0) (63, 1.0) (64, 1.0) (65, 0.984375) (66, 1.0) (67, 1.0) (68, 1.0) (69, 0.23437500000000025) (70, 0.33007812500000033) (71, 1.0) (72, 1.0) (73, 1.0) (74, 1.0) (75, 1.0) (76, 1.0) (77, 1.0) (78, 1.0) (79, 1.0) (80, 1.0) (81, 1.0) (82, 1.0) (83, 0.875) (84, 1.0) (85, 1.0) (86, 0.5100000000000001) (87, 0.6300000000000002) (88, 1.0) (89, 1.0) (90, 0.875) (91, 0.12500000000000003) (92, 0.20588235294117652) (93, 0.23437500000000008) (94, 0.3232050173010381) (95, 1.0) (96, 0.875) (97, 1.0) (98, 0.7) (100, 1.0) (101, 0.3) (104, 1.0) (107, 0.7984374999999999) (108, 1.0) (109, 1.0) (110, 1.0) (111, 1.0) (115, 1.0) (116, 1.0) (117, 1.0) (118, 0.875) (119, 1.0) (120, 1.0) (121, 1.0) (122, 1.0) (123, 1.0) (124, 1.0) (125, 1.0) (126, 1.0) (127, 1.0) (128, 1.0) (129, 1.0) (130, 1.0) (131, 1.0) (132, 1.0) (133, 1.0) (134, 1.0) (135, 0.8750000000000001) (136, 1.0) (137, 1.0) (138, 0.875) (139, 0.9125) (140, 1.0) (141, 0.875) (142, 0.875) (143, 1.0) (144, 1.0) (145, 1.0) (146, 1.0) (147, 1.0) (148, 1.0) (149, 1.0) (150, 1.0) (151, 1.0) (152, 1.0) (153, 1.0) (155, 0.9545454545454545) (156, 0.8352272727272727) (157, 1.0) (158, 1.0) (159, 1.0) (160, 1.0) (161, 0.7656250000000001) (162, 0.8906250000000001) (163, 1.0) (164, 0.12500000000000006) (165, 1.0) (166, 1.0) (167, 0.9545454545454546) (168, 1.0) (169, 1.0) (170, 0.9545454545454547) (171, 1.0) (172, 1.0) (173, 1.0) (174, 1.0) (175, 1.0) (176, 1.0) (177, 1.0) (178, 1.0) (179, 1.0) (180, 1.0) (181, 0.9843750000000001) (182, 1.0) (183, 1.0) (184, 1.0) (185, 1.0) (186, 1.0) (187, 1.0) (188, 1.0) (189, 1.0) (190, 1.0) (191, 1.0) (192, 1.0) (193, 1.0) (194, 1.0) (195, 1.0) (196, 1.0) (197, 1.0) (198, 1.0) (199, 1.0) (200, 1.0)
    };

    \addplot [ybar, bar width=0.3, color=red, fill=red]
    coordinates {
        (99, 0.21) (102, 0.7) (103, 0.5799999999999998) (105, 0.12500000000000003) (106, 0.12500000000000003) (112, 0.30000000000000004) (113, 0.27375) (114, 0.6387499999999999) (154, 0.12499999999999988)
    };

    \addplot[draw=blue, dashed] 
    coordinates {
        (-5, 0.1) (205, 0.1)
    };

    \end{axis}
\end{tikzpicture}

		}		
		\label{plot:simulation-results-epsilon-0.9-guarantees1}
	}
	
	\caption{
		Action Consistency Probability Guarantees over time of robot $r$  running \mbox{\textsc{R-EnforceAC}} with \textit{Prior-Knowledge} initialization and zero communication restrictions for $E=200$ time steps, with \textbf{(a)} $\epsilon=0.7$ and \textbf{(b)} $\epsilon=0.9$. 
		The green bars indicate the action selection was consistent between the robots, and the red bars indicate the action selection was inconsistent.
	}
	\label{fig:simulation-results-guarantees}
\end{figure*}

Figure~{\ref{fig:simulation-results-guarantees}} shows the probability of Action Consistency provided by algorithm \mbox{\textsc{R-EnforceAC}} of Robot 1 for two $\epsilon$ values (0.7, 0.9) in two specific runs. 
We can see more Action Inconsistencies in Figure~\ref{plot:simulation-results-epsilon-0.9-guarantees1} than in Figure~\ref{plot:simulation-results-epsilon-0.7-guarantees1}.
For a higher value of $\epsilon$, the threshold for communication $1-\epsilon$ decreases and so the number of communications also decreases, on the account of more Action Inconsistencies occurrences.

Algorithm \mbox{\textsc{R-EnforceAC}} results in Action Consistency during most of the simulation for both $\epsilon$ values, but still results in action inconsistencies in some time steps, even when the probability of Action Consistency is above the provided threshold $1-\epsilon$ value, since the provided guarantees are only probabilistic in these time steps (since they are below 1).
As mentioned in Section \ref{subsec:verifyac}, the condition for a \emph{deterministic} Action Consistency guarantee is when in steps 2 and 3 the optimal joint action has a cumulative likelihood of 1.
Correspondingly, we can see in Figure~\ref{fig:simulation-results-guarantees} that action inconsistencies do not occur when the  Action Consistency guarantee equals 1, as expected for a deterministic guarantee.

In Figure~\ref{plot:simulation-results-epsilon-0.7-guarantees1}, we detect an interesting case at  time step 13 where the probability of Action Consistency is \emph{below} the provided threshold $1-\epsilon$, and still no action inconsistency occurred.
The reason for this phenomenon is due to the conditions of declaring \mrac (Definition \ref{def:epsilon-mrac}), meaning this case indicates that the action selected in step 1 is in fact the highest cumulative likelihood in step 2, and therefore it does not matter for the algorithm if it is above or  below the threshold.
In fact, when the Action Consistency guarantee is below the threshold, then no other action has a cumulative likelihood which is above the threshold (otherwise a \textsc{comm} would have been triggered), and therefore, according to Equations~\eqref{eq:epsilon-mrac} and \eqref{eq:r-verifyac-not-ac-prob}, in this case the probability of Action Inconsistency is always zero.
Altogether, we conclude that Action Inconsistency can only occur when the cumulative likelihood in step 2 of the selected optimal action is below 1 and above the $1-\epsilon$ threshold, as can be seen in Figure~\ref{fig:simulation-results-guarantees}.







\subsubsection*{Dynamic communication restrictions}
In some scenarios, \textsc{comm} restrictions may arise dynamically, barring \textsc{comm}s between the robots, even though \textsc{comm} is suggested by the algorithm.
Table~\ref{table:simulation-performance} shows that action inconsistencies (Not-AC) have occurred due to \textit{comm-restr} $>0$ for both \mbox{\textsc{EnforceAC}} and Baseline I. However, the number of Not-ACs is less than the number of restricted time steps (\textit{comm-restr}) in $E$.
For \mbox{\textsc{EnforceAC}}, the reason being:
(a) \textsc{comm} was not required in some of the \textit{comm-restr} time steps as \mbox{\textsc{EnforceAC}} ensures \mrac even without \textsc{comm}, and (b) the action selections may be same coincidentally even though it required more communication to reach the threshold condition.
In fact, the number of Not-ACs is less for \mbox{\textsc{EnforceAC}} compared to Baseline I (Table~\ref{table:simulation-performance}). This is because, in contrast to \mbox{\textsc{EnforceAC}}, Baseline I reduces Not-ACs by means of only reason (b).




\providecommand{\SLAMreviewdel}[1]{{\color[rgb]{0.42,0.42,0.42}\textbf{[Delete:} #1\textbf{]}}}
\providecommand{\SLAMreviewadd}[1]{{\color[rgb]{0,0.38,0.38}#1}}
\providecommand{\AGdel}[1]{{\color{orange}\textbf{[AG deleted:} #1\textbf{]}}}
\subsection{Active Multi-Robot Visual SLAM }\label{subsec:slam-results}

In this section we evaluate our approach in an active multi-robot visual SLAM scenario using real-world data captured by DJI Robomaster equipped with cameras and a  LiDAR sensor. We consider two robots, $r$ and $r'$, autonomously and collaboratively localizing themselves and mapping an indoor environment based on visual observations that are only partially shared between the robots. In this setting, as opposed to the previous scenario (Section \ref{subsec:simulation-results}), the state and observation spaces are continuous. Therefore, we evaluate a sample-based implementation of \mbox{\textsc{R-EnforceAC}}, using \mbox{\textsc{R-VerifyAC}} for local action verification, by calculating appropriate estimators as detailed in Section \ref{sec:continuous-and-high-dim-cases}.

The visual front-end is handled by SuperGlue \citep{Sarlin20cvpr}, which extracts and matches features between images. \rev{The matched features are used within a standard monocular visual SLAM framework, where the belief over the robots' poses and mapped 3D landmarks is represented by a factor graph, and modeled as  a high-dimensional Gaussian distribution. Specifically, in our implementation we are using the  GTSAM library  \citep{Dellaert12tr} for the back-end optimization. }
The planning session 
starts after a partial mapping of the environment, during which each robot has access to its own data and to the visual observations shared by 
the other robot. In such a setting, the beliefs maintained by the two robots are inconsistent at planning time.

\begin{figure}[h]
	\centering
	\includegraphics[width=0.2\textwidth]{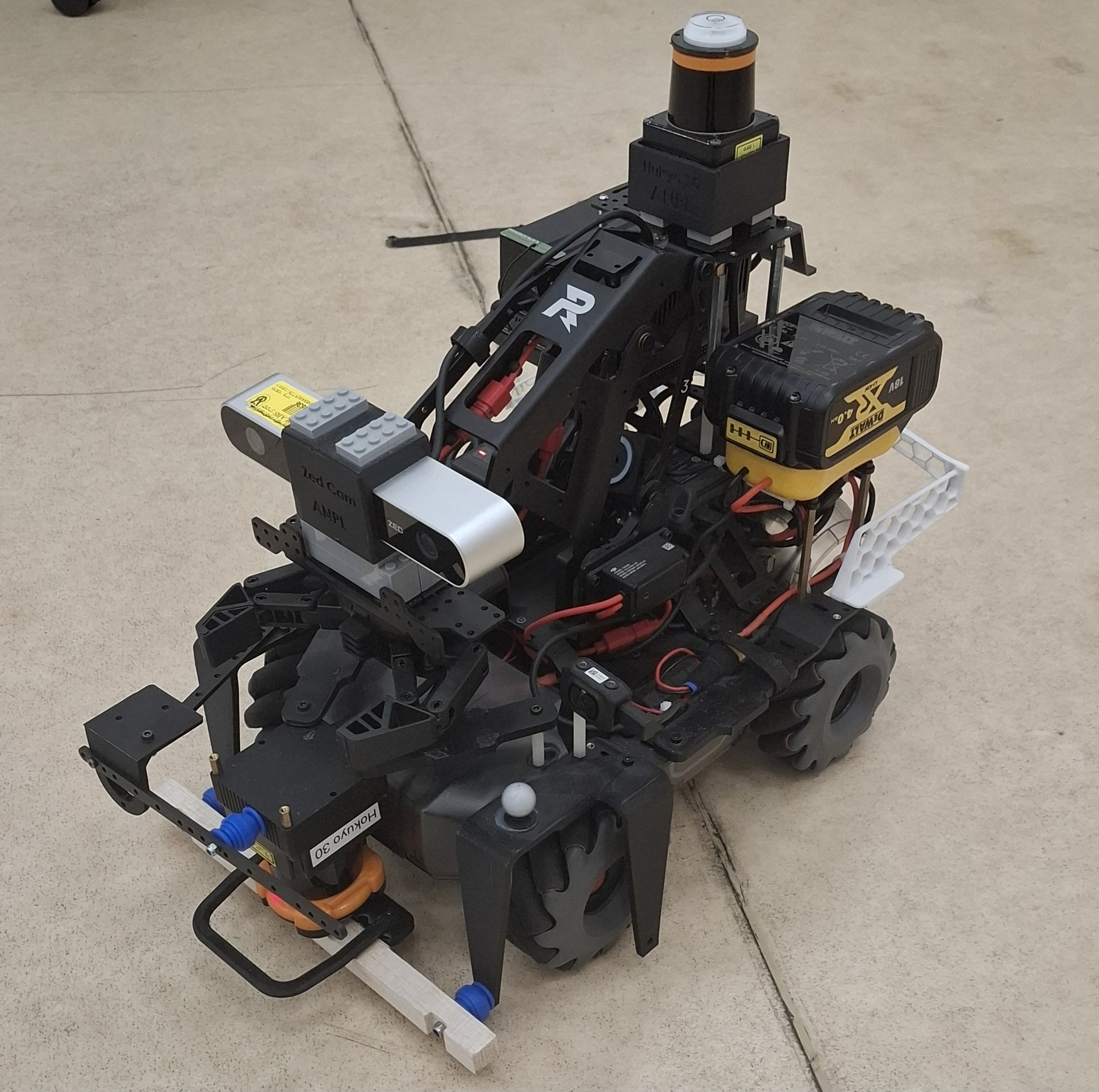}
	\caption{The DJI Robomaster, equipped with cameras and LiDAR sensors.}
	\label{fig:robot}
\end{figure}

\begin{figure}[h]
	\centering
	\includegraphics[width=0.45\textwidth]{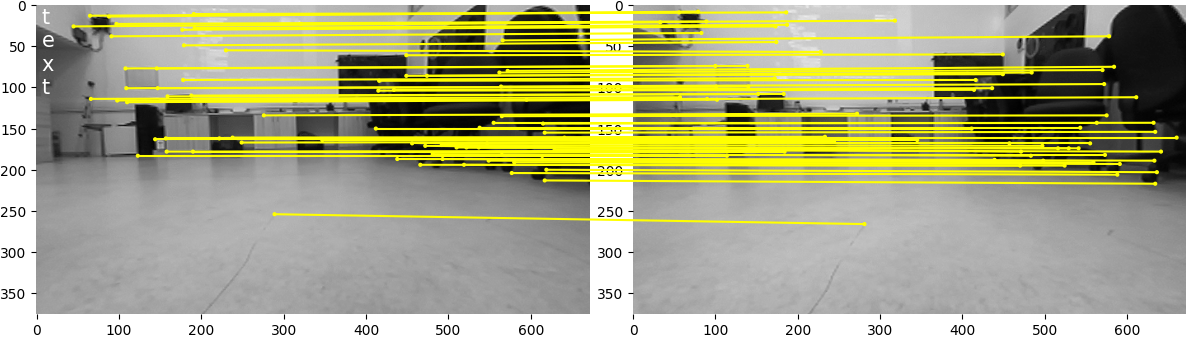}
	\caption{SuperGlue feature matching for a pair of consecutive images.}
	\label{fig:feature-matching}
\end{figure}

\subsubsection{Multi-Robot Visual SLAM Framework}


Within our multi-robot visual SLAM framework, we consider a factor graph \citep{Kschischang01it} based belief representation. Specifically, the belief maintained by robot $r$ at planning time $k$ given the history $H^r_k$ that includes the performed actions of both robots, its own observations and those shared by robot $r'$, is given by

    \begin{equation}\label{eq:beliefvslam}
        \begin{aligned}
            & b^r_k = \prob{x^r_k \mid H^r_k} = \prob{\xi^r_{1:k}, \xi^{r'}_{1:k}, L^r \mid H^r_k}\propto \\
            & \prod_{i\in \{r,r'\}} \!\! \mathbb{P}_0(\xi^i_0) \prod_{t=1}^k T^i(\xi^i_t\mid \xi^i_{t-1}, a^i_{t-1}) \!\!
            \prod_{\substack{j\in \mathcal{M}^i_t \\ z^i_{t,j} \in H^r_k}} \!\! O^i(z^i_{t,j} \mid \xi^i_t, l_j). 
        \end{aligned}
    \end{equation}


	Here, $\xi^i_t$ denotes the pose of robot $i\in \{r,r'\}$ at time instant $t$ (as in the previous section), and $L^r=\{l_j\}$ is a set of observed landmarks that is
    maintained by robot $r$. Further, in \eqref{eq:beliefvslam}, $z^i_{t,j}$ is part of the history $H^r_k$ available to robot $r$, denoting the visual observation of 
    robot $i\in\{r,r'\}$ at time $t$ and the $j$th landmark, and the corresponding associations are represented by $\mathcal{M}^i_t$. In our setup, we consider a Gaussian 
    transition model $T(\xi_t\mid \xi_{t-1}, a_{t-1})$  that is based on visual-odometry, and a Gaussian observation model, 	$O(z \mid \xi, l)$ that is based on standard 
    camera projection equations. We assume the cameras are calibrated.

    \begin{figure}[h]
        \centering
        \subfigure[Belief maintained by $r$.]{%
            \includegraphics[width=0.4\textwidth,trim=0 0 0 26,clip]{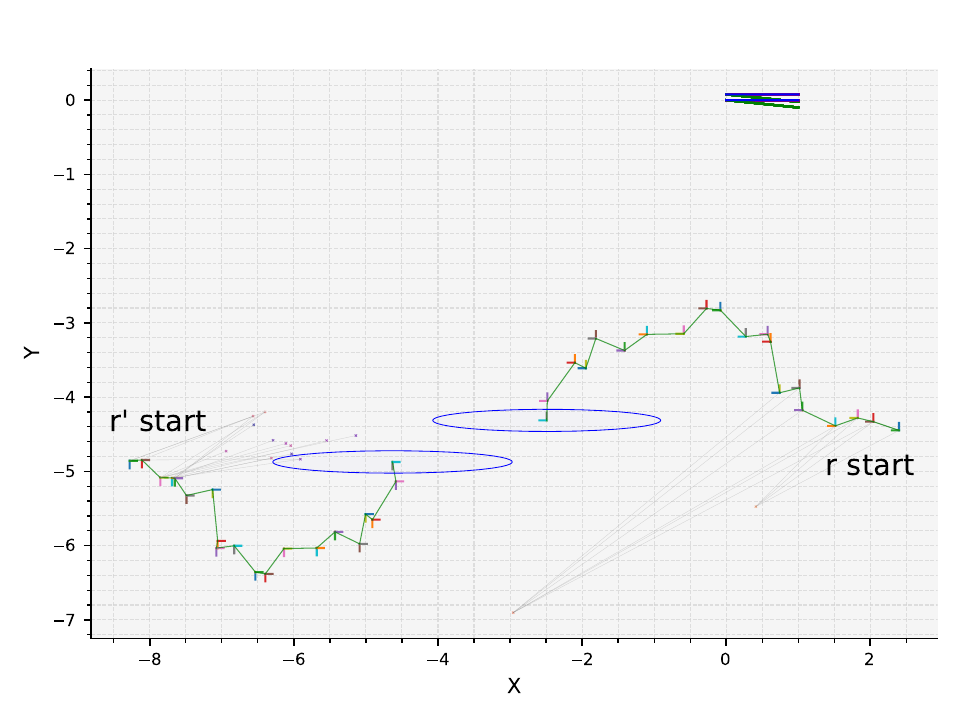}%
            \label{fig:belief-representation-r}%
        }\\[2mm]
        \subfigure[Belief maintained by $r'$.]{%
            \includegraphics[width=0.4\textwidth,trim=0 0 0 26,clip]{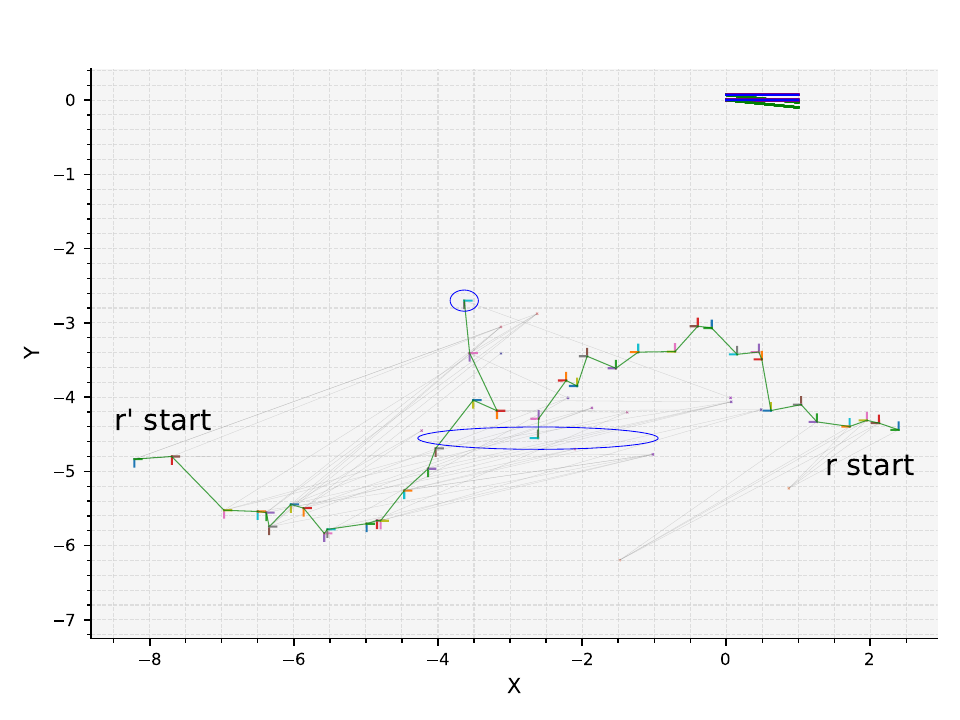}%
            \label{fig:belief-representation-r-prime}%
        }
        \caption{Estimated trajectories and landmarks in the local belief maintained by robot $r$~\subref{fig:belief-representation-r} and by robot  $r'$~\subref{fig:belief-representation-r-prime}, immediately after the final data-sharing iteration of the planning process based on \textsc{R-VerifyAC} and before motion execution. Gray edges show visual observations. Blue ellipses denote uncertainty around the current robot poses based on the locally available belief of each robot.}
        \label{fig:belief-representation}
    \end{figure}

Since the robots share only part of the data, each robot $r$ has its own history $H^r_k$ and the corresponding belief $b^r_k$. Figure~\ref{fig:belief-representation} 
visualizes the estimated trajectories and landmarks based on the beliefs $b^r_k$ and $b^{r'}_k$ maintained by robots $r$ and $r'$, respectively, immediately after the final data-sharing iteration of the planning process based on \textsc{R-VerifyAC}.
 The larger ellipse around the current pose of robot $r'$ in the belief maintained by robot $r$ (Figure~\ref{fig:belief-representation-r}), compared with the corresponding ellipse in robot $r'$'s own belief (Figure~\ref{fig:belief-representation-r-prime}), indicates greater pose uncertainty due to observations of robot $r'$ available in its own belief and absent from the belief maintained by robot $r$ (and likewise for the inference process of the other robot).

Further, recall that in our formulation we currently consider the same state space for all agents. To satisfy this setting, each robot $r$ only considers within our algorithms the observations corresponding to the landmarks present in the shared data $\leftidx{^c}{\mathcal{H}}{_k^{r,r'}}$, such that $L^r=L^{r'}$. Observations corresponding to landmarks available solely to $r$ are not used in the planning process until they are shared and become part of the shared data $\leftidx{^c}{\mathcal{H}}{_k^{r,r'}}$.

\begin{figure}[h]
	\centering
	\includegraphics[width=0.4\textwidth]{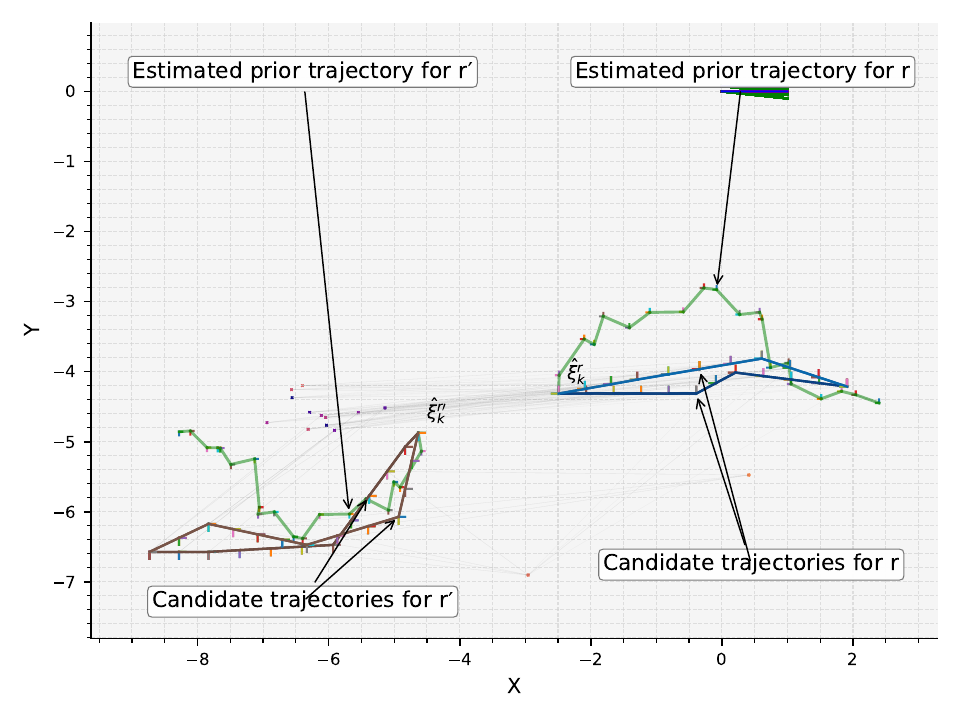}
    \caption{Candidate relative-pose sequences applied to prior belief $b^r_k$. Labels $\hat\xi^r_k$ and $\hat\xi^{r'}_k$ mark estimated current poses. Each robot has two candidate action sequences, with corresponding predicted trajectories. The trajectories are shown in blue shades for $r$ and brown for $r'$. Pale green shows prior trajectories and motion factors. Gray shows predicted visual observations; prior-belief visual observations remain hidden.}
        \label{fig:posterior-iter0}
\end{figure}

During the planning stage, the system evaluates various candidate action sequences (i.e., lists of relative motions) based on an objective function \rev{described in Section \ref{expr:obj-func}}. Each candidate 
action is applied to the current (prior) belief, and in doing so, new vision factors are generated by projecting landmarks onto the anticipated frames. This produces a 
distinct posterior belief for every candidate action sequence, representing the belief that would result if that action were executed. The objective function is then 
evaluated for each posterior belief to determine the optimal action sequence.
 Figure~\ref{fig:posterior-iter0} illustrates candidate action sequences for each robot.


	During \textsc{COMM},  visual observations 
are shared, i.e. each robot shares vision factors to the landmarks that are in a field of view from the oldest position, for which the data has not been
shared. Specifically, in our current implementation, each COMM corresponds to sharing \textit{all} the vision factors \AG{from up to two additional keyframes}, while further deciding on the most valuable subset of these is outside the scope of this work. If the shared data is insufficient to reach MR-AC, additional data is shared from the following \AG{keyframes}, and so on, until MR-AC is declared to be achieved by \textsc{R-VerifyAC}. 



Each subsequent non-hold movement creates one new keyframe per moving robot; later communications may also include earlier unshared keyframes.

\subsubsection{Objective Function}\label{expr:obj-func}

To determine the optimal action sequence, we first recall the  theoretical objective function \eqref{eq-objfun}.
Our objective is entropy-focused: it reduces the expected uncertainty in the two robots' predicted terminal poses. We only consider the reward at the end of the candidate action sequence, i.e.~$\rho_l(b_{k+l}^r, a_{k+l})=0$ for $l\in[0,L-1]$, and set
the terminal reward $\rho_L(b_{k+L}^r)$ to 
\revnew{
\begin{align}\label{eq:terminalreward}
	\rho_L(b_{k+L}^r)&=-\beta \sum_{i\in\{r,r'\}} H(\xi_{k+L}^{i}\mid b_{k+L}^r) 
	- (1-\beta) \cdot\\
	& \!\!\!\!\!\!\!\!\!\!\!\! \expt{\xi^{r}_{k+L-1}, \xi^{r'}_{k+L-1} \mid b^r_{k+L-1}}{d(g^r, \xi^{r}_{k+L-1}) + d(g^{r'}, \xi^{r'}_{k+L-1})}. \nonumber
\end{align}
Here, $H(\xi_{k+L}^{i}\mid b_{k+L}^r)$ is the entropy over the terminal pose of robot $i$, calculated based on the belief $b_{k+L}^r$ of robot $r$ by marginalizing the rest of the random variables, i.e.~given $b_{k+L}^r = \prob{\xi_{k+L}^{r}, \xi_{k+L}^{r'}, L\mid H^{r}_{k+L}}$ with a corresponding history $H^{r}_{k+L}$, we have $H(\xi_{k+L}^{i}\mid b_{k+L}^r) = H(\prob{\xi_{k+L}^{i} \mid H^r_{k+L}})$. Further,} $g^r$ and $g^{r'}$ are robots' $r$ and $r'$ goal locations, respectively, and $d(\cdot,\cdot)$ is the Euclidean distance function.


Further, in our specific implementation, we approximate the expectation in \eqref{eq-objfun} by maximum likelihood observations, and the expectation in \eqref{eq:terminalreward} by maximum likelihood estimates $\hat{\xi}^{r}_{k+L-1}, \hat{\xi}^{r'}_{k+L-1}$ from $b_{k+L}^r$, which is modeled as Gaussian.  Therefore, the corresponding objective function estimator given the belief $b^r_k$ and a candidate action sequence $a_{k+}$ is 
\revnew{
\begin{align}\label{eq:hybrid_obj_fun_max}
	\hat{J}(b_k^r, a_{k+}) = & -\beta \, \sum_{i\in\{r,r'\}} H(\xi_{k+L}^{i}\mid b_{k+L}^r) \\
	& - (1-\beta) \, (d(g^r, \hat{\xi}^{r}_{k+L-1}) + d(g^{r'}, \hat{\xi}^{r'}_{k+L-1})). \nonumber
\end{align}
}
\noindent The optimal joint action sequence is then chosen by maximizing the objective function estimator:
\begin{equation}\label{eq:hybrid_argmax}
a_{k+}^* = \argmax_{a_{k+} \in \mathcal{A}_{k+}} \hat{J}(b_k^r, a_{k+}).
\end{equation}

\subsubsection{Replanning Protocol}\label{subsec:sequential-replanning}

We use a sequential replanning protocol in which every executed motion is followed by a new planning and potential data-sharing session. 
At each planning session, both robots invoke \mbox{\textsc{R-VerifyAC}}, or a baseline planning approach, using their respective local beliefs. 
	Only the first motion of each selected path is executed: for example, for a planning session at time instant $k$, from its selected joint action sequence $a^*_{k+}$, robot $r$ executes only its own component of the first joint action $a^*_k$. At planning session $k+1$, both local histories are conditioned on the joint action actually executed at planning session $k$, and the set of joint action sequences is regenerated. A robot whose generated path has reached its terminal component executes a hold action while continuing inference and (potential) communication until the other robot also terminates. 
	
	If, at some planning session,  the robots execute components selected from different joint action sequences, such an outcome is recorded as an inconsistent decision rather than treated as an execution failure. 
	The MR-AC guarantee concerns the joint-action choice at each planning session; subsequent replanning does not change whether an earlier joint-action choice was consistent.

\subsubsection{Results}

%

\begingroup
\makeatletter
\setlength{\@dblfptop}{0pt}
\setlength{\@dblfpsep}{12pt}
\setlength{\@dblfpbot}{0pt plus 1fil}
\setlength{\dbltextfloatsep}{8pt plus 1pt minus 2pt}
\makeatother
\begin{figure*}[!t]
	\centering
	\subfigure[\scriptsize\upshape Data-sharing iteration 1.\label{fig:samples-action-rewards-step1}]{
	\includegraphics[width=0.4\textwidth]{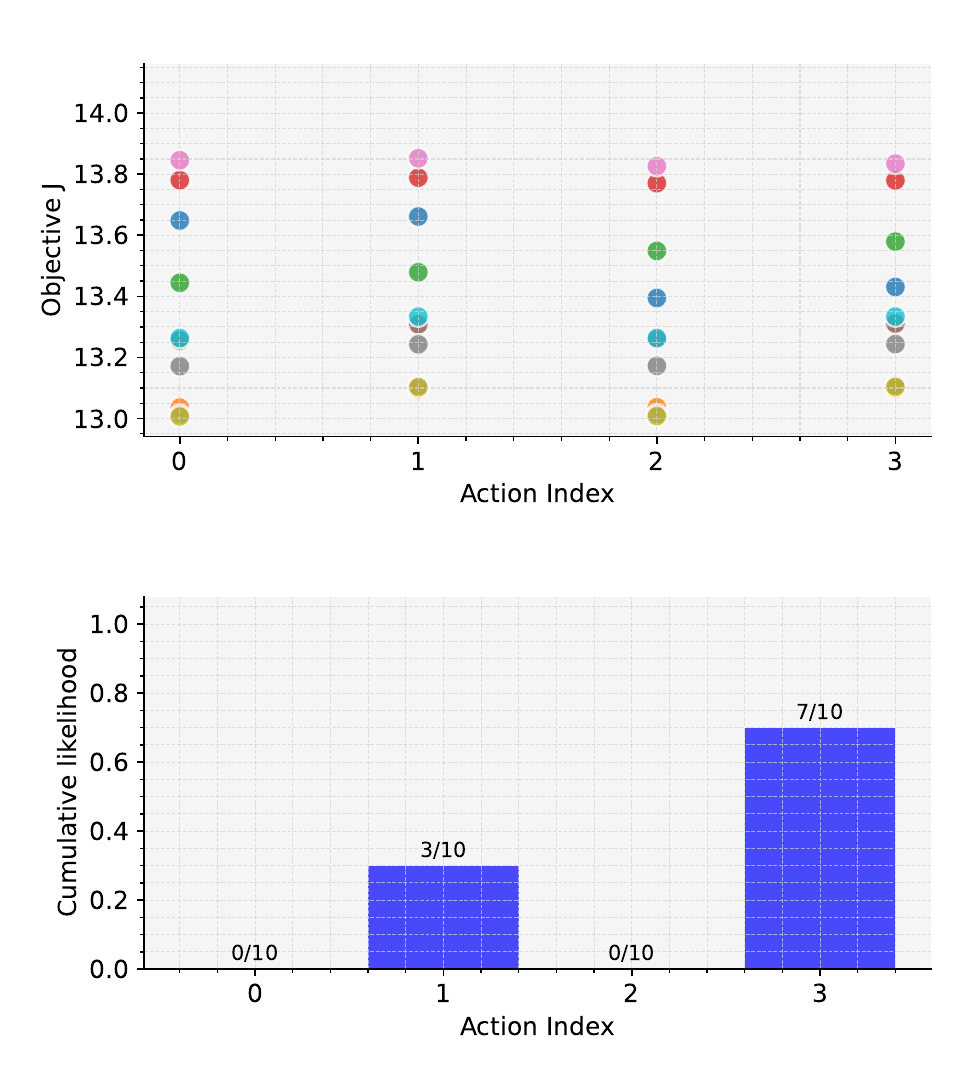}
	}
	\subfigure[\scriptsize\upshape Data-sharing iteration 2.\label{fig:samples-action-rewards-step2}]{
	\includegraphics[width=0.4\textwidth]{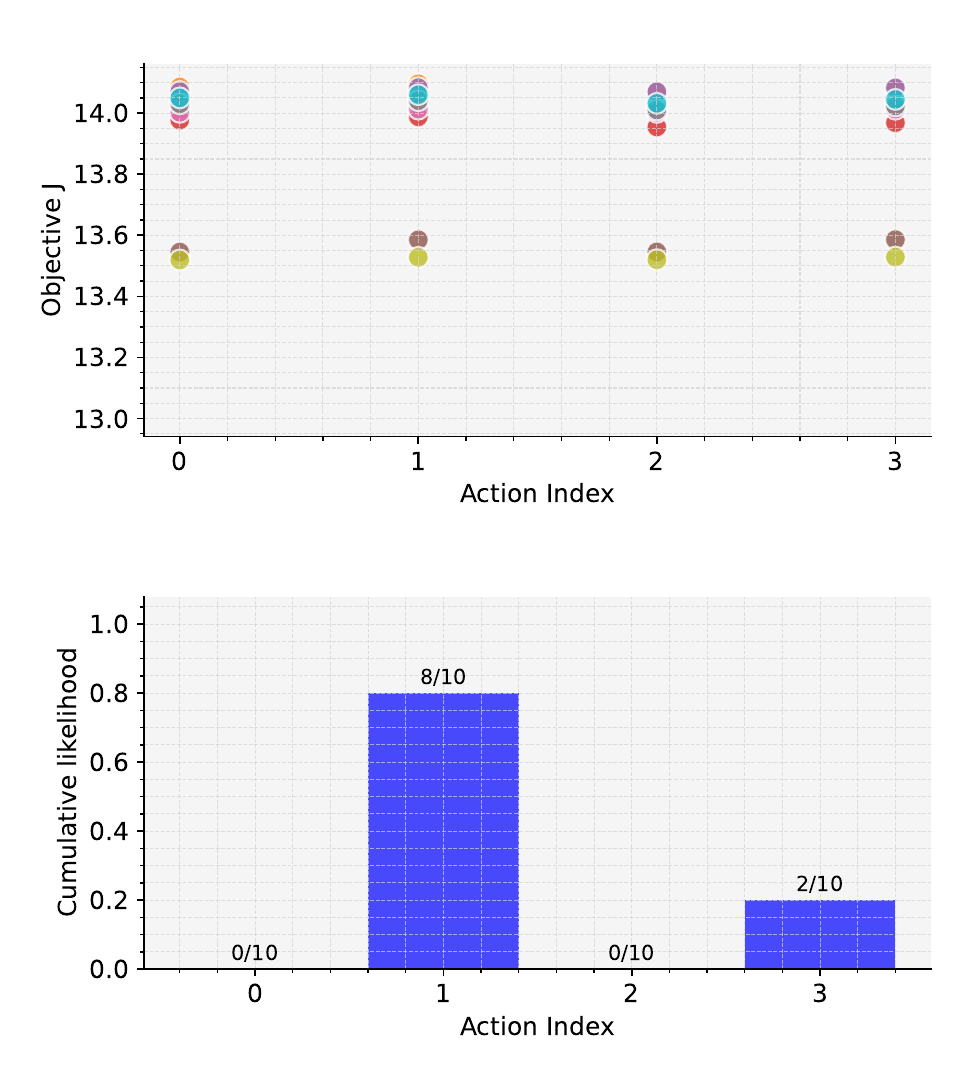}
   }
   \caption{
Plots from the perspective of robot $r'$: \textbf{(a)} data-sharing iteration 1, \textbf{(b)} data-sharing iteration 2.
   		(top) Objective values ($J$) for different joint robot actions at step 2 of R-VerifyAC. Each color corresponds to a different sampled realization of missing measurements. In this scenario, ten samples of missing observations are generated. 
   		(bottom) Estimated cumulative likelihood \eqref{eq:cumulative-likelihood-estimator} 
		of each action. Due to additional data shared between robots, the distribution of
   		objective values in \textbf{(b)} differs from
		\textbf{(a)}.
	}
   \label{fig:step2-illustration}
\end{figure*}

\revnew{We now provide a closer look at the planning process before presenting the results across simulation runs. During Step~1 of \mbox{\textsc{R-VerifyAC}}, each robot evaluates its potential joint actions based on its own observations and the shared observations from the other robot. In Steps~2 and~3, missing visual landmark measurements are sampled to generate hypothetical factor graphs and corresponding estimated values, each sample representing a possible prior belief for planning. The objective function \eqref{eq:hybrid_obj_fun_max} is then evaluated for each such prior belief, for each of the candidate action sequences.

Figures~\ref{fig:samples-action-rewards-step1} and~\ref{fig:samples-action-rewards-step2} visualize these calculations from the perspective of robot $r'$ at two data-sharing iterations considering Step~2 of \mbox{\textsc{R-VerifyAC}}. The top plots show the calculated objective function for different candidate actions, where each color corresponds to a different prior belief. The histograms in the bottom plots show the estimated cumulative likelihood of each action \eqref{eq:cumulative-likelihood-estimator}, calculated as the number of samples that favor that action divided by the total number of samples.

As robots share more data, the prior beliefs used for sampling change, and the samples and their corresponding objective values evolve as well. Figure~\ref{fig:iter5} shows three sampled priors used for the calculations in Figure~\ref{fig:samples-action-rewards-step1}.

Table~\ref{table-results} summarizes the results based on 30 simulation runs of our \mbox{\textsc{R-EnforceAC}} algorithm, compared to Baseline-I and Baseline-II algorithms. COMMs is the sum of communication rounds across all planning sessions. Factors communicated is the total number of directed visual-observation factors transferred during those communication rounds.  For each planning method, inconsistent decisions are reported as the mean and standard deviation across runs of the percentage of inconsistent planning sessions within each run. The number of planning sessions may vary for different planning methods and the specific configuration of the parameter $\epsilon$ of the \mbox{\textsc{R-EnforceAC}} algorithm.
	
	In Table~\ref{table-results}, we also report performance, in terms of the realized objective \eqref{eq:hybrid_obj_fun_max} upon execution of the actions determined by each planning method. This is presented  from the perspective of each robot, utilizing the belief it has access to by the end of the scenario. For each robot, this belief is conditioned on all of its locally observed landmarks and only those observations of the other robot that were actually communicated. In particular, the  realized objective, over the poses of robot $r$ and robot $r'$ at the end of the scenario (see \eqref{eq:terminalreward}),  is evaluated separately under post-execution estimated beliefs for each of the robots.
	We report these performance results, i.e.~calculations of  \eqref{eq:hybrid_obj_fun_max} separately from the perspective of each of the two robots, in the last two columns of Table~\ref{table-results}.

Overall, as $1-\epsilon$ increases, i.e.~the epsilon value decreases, the \mbox{\textsc{R-EnforceAC}} algorithm becomes stricter about ensuring action consistency, resulting in more
data-sharing iterations (i.e.~more COMMs and transmitted visual factors), but with fewer inconsistent decisions. 


The reported mean realized objective tends to increase as $\epsilon$ decreases. As shown in the last two columns of Table~\ref{table-results}, from both robots' perspectives and all examined values of $\epsilon$, the mean realized objective obtained by \textsc{R-EnforceAC} lies between those obtained by Baseline-I and Baseline-II.

}

\VI{
As discussed in Sections~\ref{sec:ProblemFormulation} and~\ref{sec:PlanningImplications} and illustrated in Fig.~\ref{fig:objective-surface}, beliefs conditioned on different data may induce different objective profiles, objective values, and optimal joint action sequences. Whenever MR-AC is certified at a planning session, the robots select a common joint action sequence that maximizes the objective with respect to each robot's current local belief, conditioned on all data available to that robot at termination of that planning session, including data received from the other robot. 
	For the relaxed approach considered here, \mbox{\textsc{R-EnforceAC}}, this local-optimality statement holds on the MR-AC event, whose probability is quantified by the corresponding certificate in Theorem~\ref{thm:probabilistic-guarantees-rverifyac}.  
	This guarantee  is both belief-relative and local to the corresponding planning session: it does not imply optimality with respect to the belief induced by full data sharing, as used by Baseline-I.


	
	In an exact linear-Gaussian model,
	conditioning on additional measurements cannot increase the differential entropy and decreases it whenever the added measurements provide nonredundant information. Under the Gaussian approximation employed here with linearized models, sharing additional informative measurements is therefore expected to reduce terminal pose uncertainty and improve the negative-entropy component of the objective. 
	This component is assigned a substantially larger weight than the distance-to-goal component; in our implementation, we use $\beta=0.95$. Moreover, the distance-to-goal term is Lipschitz continuous in the terminal pose estimates, so its variation is bounded by the corresponding changes in these estimates. Its contribution is further scaled by the relatively small coefficient $1-\beta=0.05$. 
	Together, these observations provide a plausible explanation  
	for the empirical performance trend in Table~\ref{table-results}. 
	

	Because the experiment involves sequential decision making over multiple planning sessions, different values of \(\epsilon\) may lead to different decisions at individual sessions and thereby affect subsequent trajectories, observations, and beliefs. Consequently, 
	in this sequential replanning setting, the observed trend in the final realized objective as $\epsilon$ varies should be viewed as an empirical observation rather than a theoretical monotonicity guarantee.
	

%
%
%
%
%
%
}

\begin{figure*}[!b]
    \centering
    \subfigure[]{%
        \includegraphics[width=0.32\textwidth]{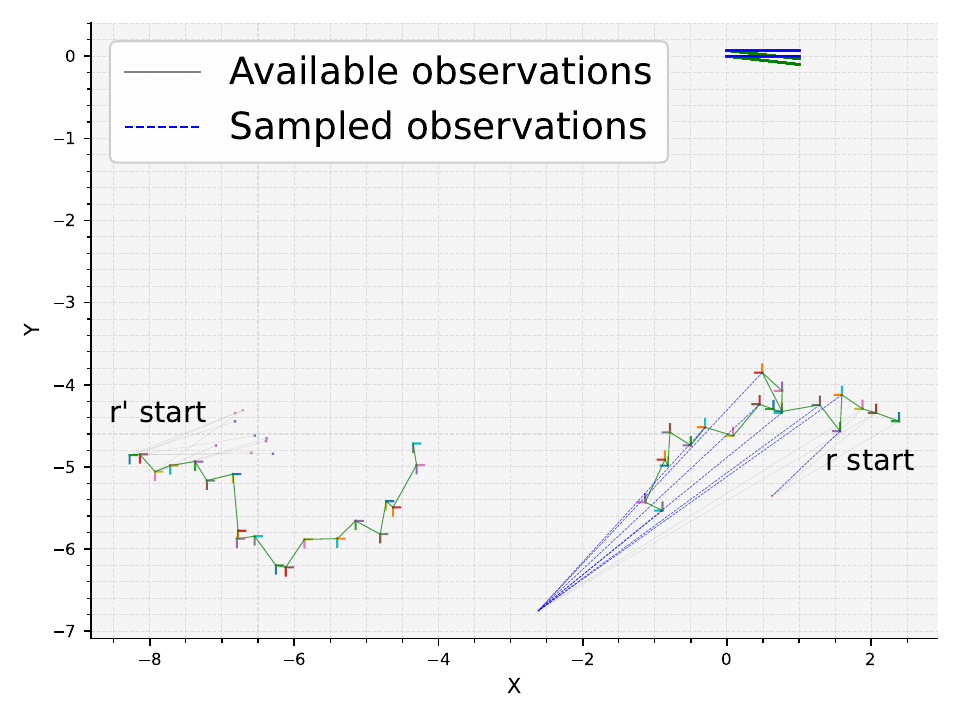}%
        \label{fig:iter5-sample0}%
    }
    \hfill
    \subfigure[]{%
        \includegraphics[width=0.32\textwidth]{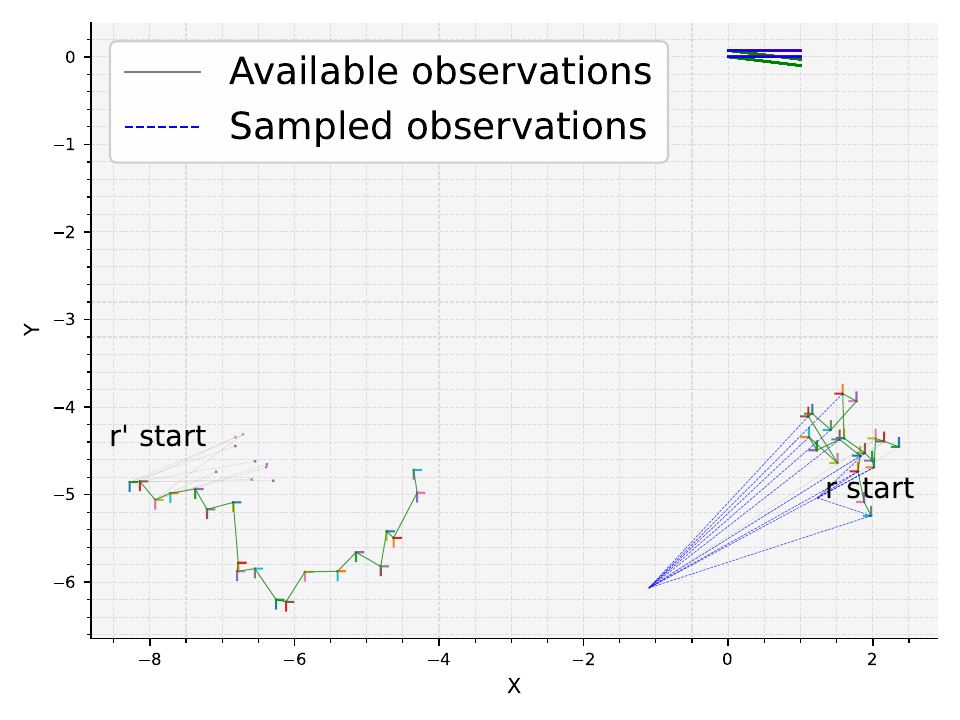}%
        \label{fig:iter5-sample1}%
    }
    \hfill
    \subfigure[]{%
        \includegraphics[width=0.32\textwidth]{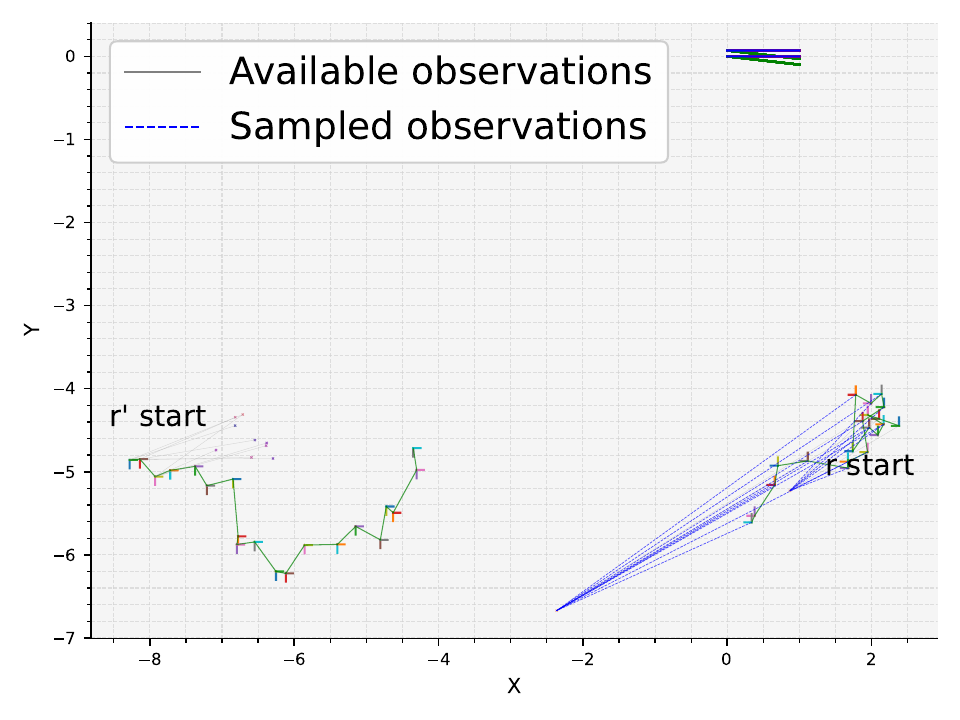}%
        \label{fig:iter5-sample2}%
    }
    \caption{Sampled visual landmark observations of robot $r$  at data-sharing iteration 1, as obtained from the belief maintained by robot $r'$ during Step 2 of R-VerifyAC. These sampled observations are denoted with dashed blue edges. Solid gray edges denote visual observations available to robot $r'$, including local observations and those received from robot $r$. Each plot corresponds to another prior belief. Panels (a), (b), and (c) show samples 0, 1, and 2 from Figure~\ref{fig:samples-action-rewards-step1}, respectively. 
}
        \label{fig:iter5}
\end{figure*}

\begin{table*}[!t]
\centering
\captionof{table}{\scriptsize \AG{Multi-robot active visual SLAM results.}}
\label{table-results}
\resizebox{\textwidth}{!}{%
\begin{tabular}{lccccc}
\hline
\AG{Policy} & \AG{\shortstack{COMMs\\(per run, mean $\pm$ SD)}} & \AG{\shortstack{Factors communicated\\(per run, mean $\pm$ SD)}} & \AG{\shortstack{Inconsistent planning sessions\\(\% per run, mean $\pm$ SD)}} & \AG{\shortstack{Realized objective\\$r$ perspective}} & \AG{\shortstack{Realized objective\\$r'$ perspective}} \\
\hline
\AG{Baseline-I (full sharing)} & \AG{$24.03 \pm 3.00$} & \AG{$448.24 \pm 45.99$} & \AG{$0.00 \pm 0.00$\%} & \AG{$12.57 \pm 1.80$} & \AG{$12.67 \pm 1.94$} \\
\AG{\textsc{R-EnforceAC}  $\epsilon = 0.2$} & \AG{$15.50 \pm 3.87$} & \AG{$366.79 \pm 68.99$} & \AG{$23.63 \pm 15.07$\%} & \AG{$9.37 \pm 2.24$} & \AG{$8.35 \pm 1.72$} \\
\AG{\textsc{R-EnforceAC}  $\epsilon = 0.75$} & \AG{$14.32 \pm 2.18$} & \AG{$349.34 \pm 45.52$} & \AG{$25.10 \pm 14.96$\%} & \AG{$8.87 \pm 1.98$} & \AG{$8.32 \pm 1.82$} \\
\AG{\textsc{R-EnforceAC}  $\epsilon = 0.9$} & \AG{$13.76 \pm 1.99$} & \AG{$341.55 \pm 45.20$} & \AG{$30.95 \pm 11.17$\%} & \AG{$8.85 \pm 2.01$} & \AG{$7.65 \pm 0.76$} \\
\AG{Baseline-II (no sharing)} & \AG{$0.00 \pm 0.00$} & \AG{$0.00 \pm 0.00$} & \AG{$38.19 \pm 11.84$\%} & \AG{$5.58 \pm 0.33$} & \AG{$6.14 \pm 0.32$} \\
\hline
\end{tabular}}
\end{table*}

In conclusion, in this section we demonstrated the application of our algorithm \textsc{R-EnforceAC} in continuous observation and state spaces, considering the challenging setting of multi-robot active  visual SLAM. 
In these experiments, \textsc{R-EnforceAC} reduced the amount of communication between the robots while controlling the percentage of inconsistent planning sessions.



%

\endgroup




\section{Conclusions}\label{sec:conclusion}

\VI{

In this work, we introduce and formulate a new problem: MR-BSP with inconsistent beliefs across a general team of cooperative robots. For the two-robot setting, we develop and analyze algorithms that seek to achieve multi-robot action consistency while reducing the communication required between the robots. Whenever MR-AC is certified, the robots select a common joint action sequence that is optimal with respect to each robot's local belief, conditioned on all data available to that robot at termination of the corresponding planning session, including data received from the other robot. Importantly, this guarantee does not require the resulting local beliefs to be identical.

We develop three variants of the proposed approach. The base algorithm \textsc{EnforceAC} self-triggers communication until MR-AC is established with a deterministic guarantee. The extended algorithm \textsc{R-EnforceAC} relaxes the MR-AC satisfaction criteria, while providing deterministic or probabilistic MR-AC guarantees under the corresponding conditions. Finally, \textsc{R-EnforceAC-simp} reduces computation by reasoning over a smaller subset of observations while retaining the corresponding guarantees.

We evaluate the proposed algorithms in a simulated multi-robot search-and-rescue problem with discrete state and observation spaces. 
In the evaluated scenarios, \textsc{EnforceAC} reduced the number of communications by up to 60\% compared with the full-communication approach, Baseline-I; \textsc{R-EnforceAC} further reduced communication by up to 20\%; and \textsc{R-EnforceAC-simp} yielded a runtime improvement of up to 63\% over \textsc{R-EnforceAC}. 

 We further formulate the approach for continuous and high-dimensional state and observation spaces and demonstrate it experimentally in an active multi-robot visual SLAM application with real robots. The results illustrate how $\epsilon$ controls the trade-off between communication and action consistency: decreasing $\epsilon$ results in more data sharing but fewer inconsistent planning sessions, while the mean realized objective tends to improve. Across all examined values of $\epsilon$, \textsc{R-EnforceAC} required less communication than the full-sharing baseline, produced fewer inconsistent decisions than the no-sharing baseline, and achieved mean realized objective values between those of the two baselines.

 Extending the proposed algorithms and their formal guarantees beyond the two-robot setting, while maintaining computational and communication scalability, constitutes an important direction for future work.
}

\begin{acks}
	This research was supported by the Israel Science Foundation (ISF), the Israeli Smart Transportation Research Center (ISTRC), and the Bernard M. Gordon Center for Systems Engineering at the Technion.
\end{acks}


\end{document}